\documentclass[a4paper,11pt]{article}
\usepackage[margin=1in]{geometry}
\usepackage{amsmath}
\usepackage{amsfonts}
\usepackage{color}
\usepackage{natbib}
\usepackage{authblk}
\usepackage{amsmath,amsfonts,amssymb}
\usepackage{amsthm}
\usepackage{url}
\usepackage{bm}
\usepackage{comment}
\usepackage{bbm}
\usepackage{algorithmic}
\usepackage{soul}
\usepackage[dvipsnames]{xcolor}
\usepackage{dsfont}
\usepackage{csquotes}

\usepackage{xspace}
\usepackage{fixfoot}

\usepackage[ruled,vlined,linesnumbered]{algorithm2e}

\usepackage{hyperref}
\usepackage{cleveref}
\usepackage{nicefrac}

\usepackage{booktabs}

\usepackage{subfig}
\usepackage{comment}

\usepackage{mathtools}

\newtheorem{lemma}{Lemma}
\newtheorem{theorem}{Theorem}
\newtheorem{proposition}{Proposition}
\newtheorem{assumption}{Assumption}
\newtheorem{definition}{Definition}
\newtheorem{corollary}{Corollary}
\newtheorem{remark}{Remark}

\newtheorem{fact}{Fact}

\title{Towards Quantifying the Preconditioning Effect of Adam}
\author[1]{Rudrajit Das\thanks{Part of this work was done as a student researcher at Google DeepMind.}}
\author[2]{Naman Agarwal}
\author[1]{Sujay Sanghavi}
\author[1,3]{Inderjit S. Dhillon}
\affil[1]{UT Austin}
\affil[2]{Google DeepMind}
\affil[3]{Google}

\date{}

\begin{document}

\maketitle

\begin{abstract}
    \noindent There is a notable dearth of results characterizing the preconditioning effect of Adam and showing how it may alleviate the curse of ill-conditioning -- an issue plaguing gradient descent (GD). In this work, we perform a detailed analysis of Adam's preconditioning effect for quadratic functions and quantify to what extent Adam can mitigate the dependence on the condition number of the Hessian. Our key finding is that Adam can suffer less from the condition number but at the expense of suffering a dimension-dependent quantity. Specifically, for a $d$-dimensional quadratic with a diagonal Hessian having condition number $\kappa$, we show that the effective condition number-like quantity controlling the iteration complexity of Adam without momentum is $\mathcal{O}(\min(d, \kappa))$.  For a diagonally dominant Hessian, we obtain a bound of $\mathcal{O}(\min(d \sqrt{d \kappa}, \kappa))$ for the corresponding quantity. Thus, when $d < \mathcal{O}(\kappa^p)$ where $p = 1$ for a diagonal Hessian and $p = 1/3$ for a diagonally dominant Hessian, Adam can outperform GD (which has an $\mathcal{O}(\kappa)$ dependence). On the negative side, our results suggest that Adam can be worse than GD for a \textit{sufficiently non-diagonal} Hessian even if $d \ll \mathcal{O}(\kappa^{1/3})$; we corroborate this with empirical evidence. Finally, we extend our analysis to functions satisfying per-coordinate Lipschitz smoothness and a modified version of the Polyak-\L ojasiewicz condition.
\end{abstract}

\section{Introduction}
Adaptive methods such as \textit{Adam} \citep{kingma2014adam} are often favored over (stochastic) gradient descent in training deep neural networks. There are some papers theoretically showing certain kinds of meaningful benefits of adaptive methods from the optimization perspective; for instance, with sparse gradients \citep{duchi2011adaptive,zhou2018convergence}, exhibiting robustness to saddle points \citep{staib2019escaping,antonakopoulos2022adagrad}, easier tuning of hyper-parameters \citep{faw2022power,faw2023beyond,wang2023convergence}, exhibiting robustness to unbounded smoothness \citep{crawshaw2022robustness,faw2023beyond,wang2022provable}, etc. We survey related work in more detail in \Cref{sec:rel-wrk-app}. However, given that the most widely used adaptive method, namely, Adam (update rule stated in \Cref{prelim}) is frequently motivated as an adaptive diagonal preconditioner, \textit{there is a conspicuous absence of results characterizing its preconditioning effect, let alone demonstrating any benefits with respect to the condition number in any setting}. In this work, we seek to quantify the preconditioning effect of Adam with exact gradients (i.e., in the deterministic setting) and answer the following question:
\begin{center}
    \enquote{\textit{Compared to gradient descent, how much better can Adam perform in terms of dependence on the condition number?}}
\end{center}
Quadratic functions have a fixed Hessian, and so they are a natural first step to understand the preconditioning effect of Adam -- given that there seem to be no results on this \textit{in any setting}. We focus on $d$-dimensional quadratics of the form:
\begin{equation}
    \label{eq:f}
    f(\bm{x}) = \frac{1}{2}(\bm{x} - \bm{x}_{*})^{\top} \bm{Q} (\bm{x} - \bm{x}_{*}),
\end{equation}
where $\bm{x}, \bm{x}_{*} \in \mathbb{R}^d$ and the Hessian $\bm{Q} \in \mathbb{R}^{d \times d}$ is symmetric positive-definite\footnote{The addition of a constant term in $f(\bm{x})$ does not change any insight.}. 
Note that $\nabla f(\bm{x}) = \bm{Q} (\bm{x} - \bm{x}_{*})$ and $\text{arg min}_{\bm{x} \in \mathbb{R}^d} f(\bm{x}) = \bm{x}_{*}$. 
{Let $\bm{D}$ be the diagonal matrix containing the diagonal elements of $\bm{Q}$.}
Also, let the eigen-decomposition of $\bm{Q}$ be $\bm{U} \bm{\Lambda} \bm{U}^\top$; here $\bm{U}$ is a unitary matrix containing the eigenvectors of $\bm{Q}$ and $\bm{\Lambda}$ is a diagonal matrix containing the corresponding eigenvalues of $\bm{Q}$. 
Let $\kappa$ be the condition number of $\bm{Q}$ (i.e., the ratio of the maximum and minimum eigenvalues of $\bm{Q}$). 
It is well known in the literature that the iteration complexity of gradient descent (without momentum) scales as $\widetilde{\mathcal{O}}(\kappa)$\footnote{Here and subsequently in this section, $\widetilde{\mathcal{O}}$ hides (poly)-logarithmic factors depending on the error level to which we wish to converge.}; this can be prohibitive when $\kappa$ is extremely large or the Hessian is \enquote{ill-conditioned}. A key message of our paper is that Adam can suffer less from $\kappa$ but at the expense of suffering a quantity depending on the dimension $d$. 
With this preface, we will now state our \textbf{main results and contributions}.
\vspace{0.2 cm}
\\
\noindent \textbf{1.} In \Cref{thm3-diag}, we show that for \textit{diagonal} $\bm{Q}$, the iteration complexity of Adam is $\widetilde{\mathcal{O}}\big(\min(d, \kappa)\big)$ with constant probability (e.g., ${3}/{4}$) over random initialization. \textit{So when $d < \mathcal{O}(\kappa)$, Adam can outperform gradient descent (GD)} with constant probability. 
\vspace{0.2 cm}
\\
\noindent \textbf{2.} 
Let $\overline{\kappa}$ be the condition number of $\bm{D}^{-1/2} \bm{Q} \bm{D}^{-1/2}$; this a diagonally preconditioned version of $\bm{Q}$ known as \textit{Jacobi preconditioning} \citep{demmel1997applied,wathen2015preconditioning}. Also, let $\kappa_\text{diag}$ be the condition number of $\bm{D}$; this is just the ratio of the maximum and minimum diagonal elements of $\bm{Q}$ and it is $\leq \kappa$. In \Cref{thm3} and the discussion after it, we show that for \textit{general} (non-diagonal) $\bm{Q}$, the complexity of Adam is $\widetilde{\mathcal{O}}\big(\min\big(d \overline{\kappa} \sqrt{d \overline{\kappa} \kappa_\text{diag}}, \overline{\kappa} \kappa_\text{diag}\big)\big)$ with constant probability over random initialization.
\textit{So when $d < \mathcal{O}\big(\frac{1}{\overline{\kappa}} \big(\frac{\kappa^2}{\kappa_\textup{diag}}\big)^{{1}/{3}}\big)$, Adam can outperform GD} with constant probability. 
\vspace{0.2 cm}
\\
\noindent \textbf{3.} For \textit{diagonally dominant} $\bm{Q}$ (see \Cref{diag-dom}), we show that $\overline{\kappa} = \mathcal{O}(1)$ and $\kappa = \mathcal{O}(\kappa_\textup{diag})$ in \Cref{prop:bar-kappa}. So for such $\bm{Q}$, \textit{Adam can outperform GD when $d < \mathcal{O}\big({\kappa^{1/3}}\big)$}. 
\vspace{0.2 cm}
\\
\noindent \textbf{4.} On the negative side, our result suggests that when $\overline{\kappa}$ is large enough, which is possible for sufficiently \textit{non-diagonal $\bm{Q}$}, \textit{Adam can be worse than GD}; see the discussion right after \Cref{rmk-diag-dom}. We corroborate this with empirical evidence in \Cref{fig:adam-expt}. Also somewhat surprisingly, we show that Adam may \textit{not} converge to zero function value asymptotically. We characterize the asymptotic point(s) to which it may converge in \Cref{sec-fixed-pt}.
\vspace{0.2 cm}
\\
\noindent \textbf{5.} In \Cref{sec-smooth-plc}, we turn our attention to functions that satisfy per-coordinate Lipschitz smoothness (\Cref{def-1-jan20}) and a modified version of the Polyak-\L ojasiewicz (PL) condition (\Cref{new-PLC}). As per \Cref{thm-adanorm-plc} and the discussion in \Cref{rmk11-jan29}, Adam can be better than GD when an initialization-dependent quantity which implicitly grows with the dimension is smaller than the ratio of the maximum and minimum per-coordinate smoothness constants.
\vspace{0.2 cm}
\\
\noindent Our main results for quadratic functions are summarized in \Cref{tab-sum}; they are directly relevant to linear regression and wide deep neural networks that evolve as linear models \citep{lee2019wide}, where the objective function is quadratic. 
\vspace{0.2 cm}
\\
\textbf{High-level overview of proof techniques:} Key ingredients of our analysis include envisioning the evolution of Adam's iterates as a \enquote{pseudo-linear} system (\Cref{rmk-plin}), applying judicious manipulations to \enquote{cancel out} the effect of $\kappa$ as much as possible (\Cref{rmk-sym-2}) and deriving a decay rate with a two-stage analysis procedure (\Cref{rmk-2-stage}). We discuss the proof outline of \Cref{thm3} (result for general $\bm{Q}$) in \Cref{pf-sketch}.

\begin{table}[!htb]
\begin{minipage}{\textwidth}
  \centering
\begin{tabular}{|c|c|c|}
\hline
\textbf{Nature of Hessian $\bm{Q}$} & \textbf{Adam} (w/o momentum) & \textbf{GD}
\\
\hline
Diagonal & $\mathcal{O}\Big(\min(d, \kappa) \log (\frac{1}{\varepsilon}) \log \log (\frac{1}{\varepsilon})\Big)$ & $\mathcal{O}\Big(\kappa \log (\frac{1}{\varepsilon})\Big)$ 
\\
\hline
Diagonally Dominant (Def. \ref{diag-dom}) & $\mathcal{O}\Big(\min\big(d \sqrt{d \kappa}, \kappa\big) \log (\frac{1}{\varepsilon}) \log \log (\frac{1}{\varepsilon})\Big)$ & $\mathcal{O}\Big(\kappa \log (\frac{1}{\varepsilon})\Big)$
\\
\hline
General & $\mathcal{O}\Big(\min\big(d \overline{\kappa} \sqrt{d \overline{\kappa} \kappa_\text{diag}}, \overline{\kappa} \kappa_\text{diag}\big) \log (\frac{1}{\varepsilon}) \log \log (\frac{1}{\varepsilon})\Big)$ & $\mathcal{O}\Big(\kappa \log (\frac{1}{\varepsilon})\Big)$ 
\\
\hline
\end{tabular}
\end{minipage}
\caption{
Iteration complexities of Adam w/o momentum and GD to converge to $\varepsilon$ error for minimizing $f(\bm{x}) = \frac{1}{2}(\bm{x} - \bm{x}_{*})^{\top} \bm{Q} (\bm{x} - \bm{x}_{*})$ (\cref{eq:f}). Here $d$ is the dimension, $\kappa$ is the condition number of $\bm{Q}$, $\overline{\kappa}$ is the condition number of $\bm{D}^{-1/2} \bm{Q} \bm{D}^{-1/2}$ and $\kappa_\text{diag} \leq \kappa$ is the condition number of $\bm{D}$, where $\bm{D}$ is the diagonal matrix containing the diagonal elements of $\bm{Q}$. The results for Adam are with constant probability (e.g., $3/4$) over random initialization. 
}
\label{tab-sum}
\end{table}

\section{Related Work}
\label{sec:rel-wrk-app}
There is copious amount of work on the convergence of adaptive gradient methods \citep{ward2020adagrad,reddi2018convergence,chen2018convergence,zaheer2018adaptive,zhou2018convergence,li2019convergence,huang2021super,da2020general,shi2021rmsprop,defossez2022simple,zhang2022adam}. Some papers show certain kinds of meaningful benefits of adaptive methods from the optimization perspective, and we focus on these papers. For e.g., adaptive methods can outperform non-adaptive methods when the gradients are sparse \citep{duchi2011adaptive,zhou2018convergence}.  \cite{staib2019escaping,antonakopoulos2022adagrad} show that adaptive methods are robust to the presence of saddle points. \cite{faw2022power,faw2023beyond,wang2023convergence} show that setting hyper-parameters is easier for adaptive methods in the sense that we do not need to be privy to problem-dependent parameters such as the smoothness constant, noise variance, etc. Moreover, adaptive methods enjoy convergence even with unbounded smoothness \citep{crawshaw2022robustness,faw2023beyond,wang2022provable}. \cite{pan2023toward} attempt to justify the success of adaptive methods over SGD in deep networks by \textit{empirically} showing that the update direction of adaptive methods is associated with a much smaller directional sharpness value than SGD.
Another rather weak motivation for the effectiveness of Adam is that it can approximately mimic natural gradient descent \citep{amari1998natural}; as pointed out by \cite{martens2020new,balles2018dissecting} there are issues with this view.

\section{Preliminaries}
\label{prelim}
Suppose we wish to minimize $f: \mathbb{R}^d \xrightarrow{} \mathbb{R}$ with access to exact gradients. Let $\bm{x}_k$ be the iterate of Adam in the $k^{\text{th}}$ iteration and let $\bm{g}_k = \nabla f(\bm{x}_k)$. For each $i \in \{1,\ldots,d\}$, let 
\begin{equation}
    \label{eq:2-dec28}
    {m}_{k}^{(i)} = \beta_1 {m}_{k-1}^{(i)} + {g}_k^{(i)} \text{ with } {m}_{0}^{(i)} = 0 \text{ and } {v}_{k}^{(i)} = \beta_2 {v}_{k-1}^{(i)} + \big({g}_k^{(i)}\big)^2 \text{ with } {v}_{0}^{(i)} = 0,
\end{equation}
where $\beta_1, \beta_2 \in (0,1)$ are known as the first moment/momentum and second moment hyper-parameters\footnote{For a vector $\bm{y}$, we denote its $i^\text{th}$ coordinate by $y^{(i)}$.}. Before stating the exact update rule that we analyze in this work, we will state its \enquote{nearest neighbor} to the vanilla version of Adam proposed by \cite{kingma2014adam}. For this nearest neighbor, the update rule of the $i^{\text{th}}$ coordinate in the $k^{\text{th}}$ iteration is:
\begin{equation}
    \label{eq:1-feb8}
    x^{(i)}_{k+1} = x^{(i)}_{k} - \frac{\alpha {m}_k^{(i)}}{\sqrt{{v}_k^{(i)} + \delta}},
\end{equation} 
where $\alpha$ is the step-size and $\delta$ is a small correction term in the update's denominator (to avoid division by zero). This is very closely related to the vanilla version of Adam proposed by \cite{kingma2014adam}; we discuss this close connection in \Cref{connection-app}. For the variant that we propose and analyze, $\delta$ in \cref{eq:1-feb8} is replaced by a coordinate-dependent and iteration-dependent term; it is as follows:
\begin{equation}
    \label{eq:3-dec28}
    x^{(i)}_{k+1} = x^{(i)}_{k} - \frac{\alpha {m}_k^{(i)}}{\sqrt{v_k^{(i)} + \max\Big(\big({g}_0^{(i)}\big)^2, \phi^2\Big)\delta_{k}^{2}}}.
\end{equation}
Note that $\delta$ in \cref{eq:1-dec28} is replaced by $\max\big(\big|{g}_0^{(i)}\big|, \phi\big)\delta_{k}$. It is worth pointing out that the \textit{coordinate-dependent $\delta$} that we set above plays a crucial role in improving the dependence with respect to the condition number (see Remarks \ref{per-coordinate-eps-diag} and \ref{per-coordinate-eps-gen}).

Our results for Adam in this paper are \textit{without first moment/momentum, i.e., with $\beta_1 = 0$}\footnote{Adam with $\beta_1 = 0$ is the same as \enquote{RMSProp} \citep{hinton2012neural}.}. The key difference of Adam from GD-style algorithms is the square root diagonal preconditioner involving the exponential moving average of squared gradients -- \textit{this is the aspect that we focus on}. 

Also, recall the update rule of gradient descent (GD) with step-size $\alpha$ is $x^{(i)}_{k+1} = x^{(i)}_{k} - {\alpha {g}_k^{(i)}}$.

\subsection{Notation}
For any $n \in \mathbb{N}$, we define $[n] := \{1,\ldots,n\}$. Vectors and matrices are in bold font. Let $\bm{e}_i$ denote the $i^\text{th}$ canonical vector, i.e., the vector of all zeros except a one in the $i^\text{th}$ coordinate. 
We denote the $i^{\text{th}}$ coordinate of a vector $\bm{v} \in \mathbb{R}^p$ by $v^{(i)}$, where $i \in [p]$. For a matrix $\bm{M} \in \mathbb{R}^{m \times n}$, we denote its $(i,j)^{\text{th}}$ element by $M[i,j]$, where $i \in [m]$ and $j \in [n]$. We denote the maximum and minimum singular values/eigenvalues of $\bm{M}$ by $\sigma_\text{max}(\bm{M})$ and $\sigma_\text{min}(\bm{M})$/$\lambda_\text{max}(\bm{M})$ and $\lambda_\text{min}(\bm{M})$, respectively. 
$\text{cond}(\bm{M})$ denotes the condition number of $\bm{M}$, i.e., $\frac{\sigma_\text{max}(\bm{M})}{\sigma_\text{min}(\bm{M})}$. 
We denote the Hadamard (i.e., element-wise) product of two matrices $\bm{A}, \bm{B} \in \mathbb{R}^{m \times n}$ by $\bm{A} \circ \bm{B}$. $\textup{Unif}[a, b]$ denotes the uniform distribution over $[a,b]$.

\subsection{Important Definitions}
\label{definitions}
Refer to the problem setting in \cref{eq:f}. Recall that $\bm{U} \bm{\Lambda} \bm{U}^\top$ is the eigen-decomposition of $\bm{Q}$ and $\bm{D}$ is the diagonal matrix containing the diagonal elements of $\bm{Q}$ (i.e., $\bm{D} = \bm{Q} \circ \bm{\text{{I}}}$). 
We will define some important quantities here.
\begin{itemize}
    \item Let $\lambda_i = {\Lambda}[i,i]$ be the $i^\text{th}$ eigenvalue of $\bm{Q}$. Further, let $\lambda_{\max}(\bm{Q}) = \kappa$ and $\lambda_{\min}(\bm{Q}) = 1$ for simplicity; so, $\text{cond}(\bm{Q}) = \kappa$.
    \item Let $q_i = Q[i,i] = D[i,i]$ be the $i^\text{th}$ diagonal element of $\bm{Q}$ (and $\bm{D}$). Further, let $q_\text{max} = \max_{i \in [d]} q_i$ and $q_\text{min} = \min_{i \in [d]} q_i$. Also, define $\kappa_\text{diag} := \big(\frac{q_\text{max}}{q_\text{min}}\big)$. Then, $\text{cond}(\bm{D}) = \kappa_\text{diag}$. It holds that $\kappa_\text{diag} \leq \kappa$ \footnote{For any $i \in [d]$, $Q[i,i] = \bm{e}_i^\top \bm{Q} \bm{e}_i \in [1, \kappa]$ (as $\lambda_{\min}(\bm{Q}) = 1$ and $\lambda_{\max}(\bm{Q}) = \kappa$). So, $q_\text{max} \leq \kappa$ and $q_\text{min} \geq 1$.} with equality holding when $\bm{Q}$ is diagonal and equal to $\bm{D}$.
    \item Define $\bar{\bm{Q}} := \bm{D}^{-1/2} \bm{U} \bm{\Lambda}^{1/2}$ and let $\text{cond}(\bar{\bm{Q}}) = \sqrt{\overline{\kappa}}$. Note that  $\bar{\bm{Q}}\bar{\bm{Q}}^\top = \bm{D}^{-1/2} \bm{Q} \bm{D}^{-1/2}$ and $\text{cond}(\bm{D}^{-1/2} \bm{Q} \bm{D}^{-1/2}) = \text{cond}(\bar{\bm{Q}}\bar{\bm{Q}}^\top) = \overline{\kappa}$.  
    \item Define $\widehat{\bm{Q}} := \bm{D}^{-1} \bm{U} \bm{\Lambda}^{1/2} = \bm{D}^{-1/2} \bar{\bm{Q}}$. Let $\sigma_\text{min}(\widehat{\bm{Q}}) = \rho_1$ and $\sigma_\text{max}(\widehat{\bm{Q}}) = \rho_2$. Also, define $\widehat{\kappa} := \big(\frac{\rho_2}{\rho_1}\big)$. Then, $\text{cond}(\widehat{\bm{Q}}) = \widehat{\kappa}$ and it holds that $\widehat{\kappa} \leq \sqrt{\kappa_\text{diag} \overline{\kappa}}$.\footnote{This is because $\widehat{\kappa} = \text{cond}(\widehat{\bm{Q}}) = \text{cond}(\bm{D}^{-1/2} \bar{\bm{Q}}) \leq \text{cond}(\bm{D}^{-1/2}) \text{cond}(\bar{\bm{Q}}) = \sqrt{\kappa_\text{diag} \overline{\kappa}}$.}
\end{itemize}
Observe that if $\bm{Q}$ is diagonal (in which case, $\bm{U} = \bm{\text{{I}}}$ and $\bm{Q} = \bm{D} = \bm{\Lambda}$), $\overline{\kappa} = 1$, $\rho_1 = \frac{1}{\sqrt{\kappa}}$, $\rho_2 = 1$ and hence, $\widehat{\kappa} = \sqrt{\kappa} = \sqrt{\kappa_\text{diag}}$. Also, $\bm{D}^{-1/2} \bm{Q} \bm{D}^{-1/2}$ is a \textit{diagonally preconditioned} version of $\bm{Q}$ and this type of preconditioning is known as \textit{Jacobi preconditioning} \citep{wathen2015preconditioning}.

\section{Main Results for Quadratic Functions}
\label{adam-v-gd}
Here we will compare the convergence of Adam and GD (both without momentum) for quadratic functions. We first state the folklore iteration complexity of GD with step-size $\alpha ={1}/{\kappa}$; recall that the step-size of GD must not exceed two times the inverse of the maximum eigenvalue of the Hessian to ensure convergence.
\begin{theorem}[\textbf{GD: Folklore}]
    \label{thm5}
    Suppose our initialization is $\bm{x}_0$. Set $\alpha ={1}/{\kappa}$. Then, $f(\bm{x}_{\widehat{K}}) < \frac{\varepsilon^2}{2}$ in $$\widehat{K} = \mathcal{O}\Big({\kappa} \log \Big(\frac{\sqrt{\kappa}\|\bm{x}_{0} - \bm{x}_{*}\|_2}{\varepsilon}\Big) \Big)$$ iterations of GD.
\end{theorem}
Note that the iteration complexity of GD depends on the condition number of $\bm{Q}$, viz., $\kappa$. We shall state the corresponding iteration complexity of Adam (with $\beta_1 = 0$) for diagonal and general $\bm{Q}$ separately. Let us look at the diagonal case first.

\begin{theorem}[\textbf{Adam: Diagonal $\bm{Q}$}]
    \label{thm3-diag}
    Let $\bm{Q} = \bm{\Lambda}$ be diagonal. Suppose $\|\bm{x}_{*}\|_\infty \leq B$ and our initialization $\bm{x}_0$ is sampled randomly such that ${x}_0^{(i)} \underset{\textup{iid}}{\sim} \textup{Unif}[-2B, 2B]$ $\forall$ $i \in [d]$. There exist $\alpha$, $\beta_2$, $\delta_k$ and $\phi$ such that with a probability of at least $p$ over the randomness of $\bm{x}_0$, $f(\bm{x}_{{K}^{*}}) < \frac{\varepsilon^2}{2}$ in
    \begin{equation*}
        K^{*} = \mathcal{O}\Bigg(\kappa_\textup{Adam} {\log\Big(\frac{\sqrt{d \kappa} B}{\varepsilon}\Big) {\log\Big(\log\Big(\frac{\sqrt{d \kappa} \|\bm{x}_{0} - \bm{x}_{*}\|_\infty}{\varepsilon}\Big)\Big)}}\Bigg)
    \end{equation*}
    iterations of {Adam}, where $\kappa_\textup{Adam} := \min\Big(\frac{d}{\log (1/p)} + 1, \kappa\Big)$.
\end{theorem}
The detailed version and proof of \Cref{thm3-diag} can be found in \Cref{pf-2}. Observe that the \textit{condition number-like quantity} $\kappa_\textup{Adam}$ appearing in the iteration complexity of Adam is $\mathcal{O}(\min(d, \kappa))$ with constant probability (e.g., ${3}/{4}$). 

It is true that the dependence on $\varepsilon$ for Adam is slightly worse than GD (\Cref{thm5}). However, logarithmic dependence on $\nicefrac{1}{\varepsilon}$ is pretty benign and the \textit{problem-dependent constant factors} multiplying the $\log (\nicefrac{1}{\varepsilon})$ and $\log (\nicefrac{1}{\varepsilon}) \log \log (\nicefrac{1}{\varepsilon})$ terms for GD and Adam -- viz., $\kappa$ and $\kappa_\textup{Adam}$ -- dictate the overall convergence speed. Hence, \textit{we are interested in these problem-dependent constant factors} and we ignore the dependence on $\varepsilon$ in all subsequent discussions.

\begin{remark}
    \label{cor-2}
    When $d < \mathcal{O}(\kappa)$, $\kappa_\textup{Adam} < \kappa$ with constant probability over the randomness of $\bm{x}_0$.
\end{remark}
\begin{remark}
    \label{per-coordinate-eps-diag}
    The {coordinate-dependent $\delta$} that we chose in \cref{eq:3-dec28} is critical to $\kappa_\textup{Adam}$ being $\mathcal{O}(\min(d, \kappa))$. Specifically, without the $\big({g}_0^{(i)}\big)^2$ terms, our bound for $\kappa_\textup{Adam}$ is $\mathcal{O}(\kappa)$.
\end{remark}
Let us now move on to the case of general (non-diagonal) $\bm{Q}$. We begin by introducing an assumption on the nature of random initialization.
\begin{assumption}[\textbf{Random Initialization}]
    \label{dist}
    Let $\bm{T} = \bm{D}^{-1} \bm{Q}$ and $\|\bm{T} \bm{x}_{*}\|_\infty \leq B$. Our initialization $\bm{x}_0$ is sampled randomly such that each coordinate of $\bm{T} \bm{x}_0$ $\in [-2B, 2B]$. Also, let $\widehat{\bm{g}}_0 = \bm{T} (\bm{x}_0 - \bm{x}_{*})$. For any $t \in [0,1]$, $\mathbb{P}\Big(\min_{i \in [d]} |\widehat{g}_0^{(i)}| > \frac{2 t B}{d}\Big) \geq \exp\Big(-\frac{\zeta t^\theta}{1 - \frac{\zeta t^\theta}{d}}\Big)$ for some $\theta \in (0,1]$ and $\xi \geq 1$.
\end{assumption}
The above assumption on the distribution of $\min_{i \in [d]} |\widehat{g}_0^{(i)}|$ is \textit{weaker} than assuming each coordinate of $\bm{T} \bm{x}_0$ is drawn i.i.d. from $\textup{Unif}[-2B, 2B]$. If we assumed the latter, then we would have $\mathbb{P}\Big(\min_{i \in [d]} |\widehat{g}_0^{(i)}| > \frac{2 t B}{d}\Big) \geq \exp\Big(-\frac{t}{1 - \frac{t}{d}}\Big)$ (see \cref{eq:86-jan8} in the proof of \Cref{lem5-jan3}) which is $\geq \exp\Big(-\frac{\zeta t^\theta}{1 - \frac{\zeta t^\theta}{d}}\Big)$ for $t \in [0,1]$ as $\theta \in (0,1]$ and $\xi \geq 1$. 

{Ideally, we would have liked to derive results with $\bm{x}_0$ being sampled uniformly at random like the diagonal case (\Cref{thm3-diag}). Unfortunately, this seems difficult in general but when $\bm{Q}$ is close to diagonal, this is possible; we discuss these things in \Cref{unif-hard}.}

Let us now look at the result for a general $\bm{Q}$ under \Cref{dist}.
\begin{theorem}[\textbf{Adam: General $\bm{Q}$}]
    \label{thm3}
    Suppose our initialization $\bm{x}_0$ satisfies \Cref{dist}. Fix some $p \in [\frac{1}{e}, 1)$. There exist $\alpha$, $\beta_2$, $\delta_k$ and $\phi$ such that with a probability of at least $p$ over the randomness of $\bm{x}_0$, $f(\bm{x}_{{K}^{*}}) < \frac{\varepsilon^2}{2}$ in
    \begin{equation*}
        K^{*} = \mathcal{O}\Bigg(\kappa_\textup{Adam} \log\Big(\frac{\sqrt{d} B}{\rho_1 \varepsilon}\Big){\log\Big(\log\Big(\frac{\|\bm{\Lambda}^{1/2} {\bm{U}}^\top (\bm{x}_0 - \bm{x}_{*})\|_2}{\varepsilon}\Big)\Big)}\Bigg)
    \end{equation*}
    iterations of {Adam}, where $\kappa_\textup{Adam} := \overline{\kappa} \min \Big(\Big(\frac{\widehat{\kappa} d^{3/2}}{\log^{1/\theta} ({1}/{p})}\Big) \zeta^{1/\theta} \big(1 + \frac{\log (1/p)}{d}\big)^{1/\theta}, \kappa_\textup{diag}\Big)$. 
\end{theorem}
The detailed version and proof of \Cref{thm3} can be found in \Cref{pf-1}; also see the discussion on hyper-parameters after the detailed version of \Cref{thm3}. \textit{We discuss the proof outline of \Cref{thm3} in \Cref{pf-sketch}}. Observe that the \textit{condition number-like quantity} $\kappa_\textup{Adam}$ appearing in the iteration complexity of Adam above is $\mathcal{O}\big(\min\big(\overline{\kappa} \widehat{\kappa} d^{3/2}, \overline{\kappa} \kappa_\textup{diag}\big)\big)$ with constant probability. Using the fact that $\widehat{\kappa} \leq \sqrt{\kappa_\text{diag} \overline{\kappa}}$, we have the following simpler bound for $\kappa_\textup{Adam}$: 
\begin{equation}
    \label{eq:6-jan8}
    \kappa_\textup{Adam} \leq \mathcal{O}\Big(\min\big(d \overline{\kappa} \sqrt{d \overline{\kappa} \kappa_\text{diag}}, \overline{\kappa} \kappa_\text{diag}\big)\Big). 
\end{equation}
So $\kappa_\textup{Adam} < \kappa$ when $\mathcal{O}\big(d \overline{\kappa} \sqrt{d \overline{\kappa} \kappa_\text{diag}}\big) < \kappa$.\footnote{We disregard the $\overline{\kappa} \kappa_\text{diag}$ term here because it is probably of the same scale as $\kappa$ if not more.} Simplifying this, we can make the following remark.
\begin{remark}
    \label{remark-2-jan8}
    When $d < \mathcal{O}\big(\frac{1}{\overline{\kappa}} \big(\frac{\kappa^2}{\kappa_\textup{diag}}\big)^{{1}/{3}}\big)$, $\kappa_\textup{Adam} < \kappa$ with constant probability over the randomness of $\bm{x}_0$. Further, since $\kappa_\text{diag} \leq \kappa$, a simpler but looser version of the previous condition is $d < \mathcal{O}\big(\frac{\kappa^{1/3}}{\overline{\kappa}}\big)$.
\end{remark}
\begin{remark}
    \label{per-coordinate-eps-gen}
    Once again, the $\big({g}_0^{(i)}\big)^2$ terms in the {coordinate-dependent $\delta$} that we set in \cref{eq:3-dec28} are crucial in obtaining the above improvement.
\end{remark}
\noindent \textbf{Value of $\overline{\kappa}$.} Recall that $\overline{\kappa} = \text{cond}(\bm{D}^{-1/2} \bm{Q} \bm{D}^{-1/2})$ is the condition number of Jacobi-preconditioned $\bm{Q}$. In general, $\overline{\kappa}$ may not be \enquote{small}. To our knowledge, the best known bound for $\overline{\kappa}$ is $\overline{\kappa} \leq \kappa^{*} \min(\kappa^{*}, d)$, where $\kappa^{*} = \min_{\text{diag. } \bm{\Sigma} \succeq  0_{d \times d}} \text{cond}(\bm{\Sigma}^{1/2} \bm{Q} \bm{\Sigma}^{1/2})$ \citep{jambulapati2020fast}. When $\bm{Q}$ is diagonal (in which case it is equal to $\bm{D}$), $\overline{\kappa} = 1$. {Hence, we expect $\overline{\kappa}$ to be $\mathcal{O}(1)$ when $\bm{Q}$ is \enquote{sufficiently diagonal}. We formally quantify this in the \textit{diagonally dominant} case (defined below) which is a relaxation of the perfectly diagonal case.}
\begin{definition}[\textbf{$\nu$-diagonally-dominant}]
    \label{diag-dom}
    A square matrix $\bm{P} \in \mathbb{R}^{n \times n}$ is said to be $\nu$-diagonally-dominant if for all $i \in [n]$, it holds that $\sum_{j \neq i} |P[i,j]| \leq \nu |P[i,i]|$ for some $\nu \in [0,1)$.
\end{definition}
It is worth mentioning that the usual definition of diagonally dominant matrices in the literature is with $\nu = 1$ above. Further, note that $\nu = 0$ corresponds to the pure diagonal case.

\begin{theorem}
\label{prop:bar-kappa}
Suppose $\bm{Q}$ is $\nu$-diagonally-dominant. Then $\bar{\kappa} \leq \big(\frac{1+\nu}{1-\nu}\big)$ and $\kappa \leq \big(\frac{1+\nu}{1-\nu}\big) \kappa_\textup{diag}$. So when $\nu$ is uniformly bounded below $1$, $\overline{\kappa} = \mathcal{O}(1)$ and $\kappa = \mathcal{O}(\kappa_\textup{diag})$.
\end{theorem}
We prove \Cref{prop:bar-kappa} in \Cref{prop:bar-kappa-pf}.

\begin{remark}
    \label{rmk-diag-dom}
    When $\bm{Q}$ satisfies the conditions in \Cref{prop:bar-kappa}, $\overline{\kappa} = \mathcal{O}(1)$ and $\kappa = \mathcal{O}(\kappa_\textup{diag})$. Then in the setting of \Cref{thm3} and using the bound in \Cref{remark-2-jan8}, $\kappa_\textup{Adam} < \kappa$ when $d < \mathcal{O}(\kappa^{1/3})$ with constant probability over random initialization.
\end{remark}
In the pure diagonal case, $\kappa_\textup{Adam} < \kappa$ when $d < \mathcal{O}(\kappa)$ as per \Cref{cor-2}; however, the condition for the same in the diagonally dominant case is $d < \mathcal{O}(\kappa^{1/3})$ as per \Cref{rmk-diag-dom}. The likely reason for this gap is that we have a unified analysis for \textit{any} non-diagonal $\bm{Q}$ (which is different from our analysis for diagonal $\bm{Q}$) and this could be loose for diagonally dominant $\bm{Q}$. It is possible that one could develop an analysis tailored specifically to the diagonally dominant case to bridge this gap; although this does not seem easy. We leave this for future work.
\vspace{0.15cm}
\\
\noindent {\textbf{Adam can be worse than GD for non-diagonal $\bm{Q}$ even if $d < \mathcal{O}(\kappa^{1/3})$.}} As mentioned earlier, $\overline{\kappa}$ may not be \enquote{small} for an arbitrary $\bm{Q}$. In fact, for $d=2$, we exhibit a $\bm{Q}$ for which $\overline{\kappa} = \kappa \gg d^3$. 
\begin{proposition}
    \label{bad-kappa-bar}
    Consider the $2 \times 2$ matrix ${\bm{Q}} = \begin{bmatrix}
                b & b-1 \\
                b-1 & b 
                \end{bmatrix}$,
    where $b \to \infty$ (thus, $b \gg d^3 = 8$). For such $\bm{Q}$, we have $\kappa = \overline{\kappa} = \mathcal{O}(b) \gg d^3$.
\end{proposition}
\noindent We prove \Cref{bad-kappa-bar} in \Cref{bad-kappa-bar-pf}. For higher $d$ also, we can construct several $\bm{Q}$ (with the help of Python) for which $\overline{\kappa} \approx \kappa > d^3$.  For instance, let $d = 50$ and $\bm{Q} = \frac{1}{d} \bm{A} \bm{A}^\top$, where each element of $\bm{A} \in \mathbb{R}^{d \times d}$ is drawn i.i.d. from $\mathcal{N}(5, 1)$. We obtain 5 different realizations of $\bm{Q}$ and list the corresponding values of $\kappa$, $\overline{\kappa}$ and $\nicefrac{\overline{\kappa}}{\kappa}$ in \Cref{tab1}; as we can see, $\overline{\kappa} \approx \kappa > d^3 (=125000)$.

\begin{table}[!htb]
\begin{minipage}{\textwidth}
  \centering
\begin{tabular}{|c|c|c|}
\hline
$\kappa$ & $\overline{\kappa}$ & $\nicefrac{\overline{\kappa}}{\kappa}$
\\
\hline
$2073169$ & $2103577$ & $1.015$  
\\
\hline
$5054195$ & $5208092$ & $1.030$
\\
\hline
$109946128$ & $109546215$ & $0.996$
\\
\hline
$4518264$ & $4578310$ & $1.013$
\\
\hline
$4282546$ & $4277742$ & $0.999$
\\
\hline
\end{tabular}
\end{minipage}
\caption{Values of $\kappa$ and $\overline{\kappa}$ for 5 different realizations of $\bm{Q}$. Note that $\overline{\kappa} \approx \kappa > d^3 (=125000)$. \label{tab1}}
\end{table}
{So for such $\bm{Q}$, the bound for $\kappa_\textup{Adam}$ in \cref{eq:6-jan8} is worse than $\kappa$, which would suggest \textbf{Adam is worse than GD} if the Hessian is equal to such a $\bm{Q}$. This is indeed the case in reality for the standard version of Adam used in practice (\cref{eq:1-dec28} in \Cref{connection-app}) -- we show this empirically for the first realization of \Cref{tab1} in \Cref{fig:non-diag}. Further, we compare Adam and GD if the Hessian is equal to $\bm{\Lambda}$, i.e., the diagonal matrix containing the eigenvalues of $\bm{Q}$; note that $\text{cond}(\bm{Q}) = \text{cond}(\bm{\Lambda})$. In this case, \Cref{thm3-diag} suggests that Adam should outperform GD because $d \ll \kappa$. As we show in \Cref{fig:diag}, this prediction is also correct.}

{It is also worth pointing out that in Figures \ref{fig:non-diag} and \ref{fig:diag}, Adam's function value stops decaying beyond a point and converges to a non-zero function value. We theoretically characterize this phenomenon in \Cref{sec-fixed-pt}\footnote{Also refer to the last couple of sentences of \Cref{pf-sketch}.}.}
\vspace{0.2 cm}
\\
\textbf{Dependence on dimension:} We will now show that a dependence on the dimension $d$ is necessary for Adam's iteration complexity, at least for $\varepsilon$ up to a certain range. 
\begin{theorem}
    \label{lb-diag}
    Suppose $\delta_k = 0$. For any diagonal $\bm{Q}$, achieving $f(\bm{x}_{K'}) \leq \frac{1}{2} \big(\min_{j \in [d]} |x^{(j)}_{0} - x^{(j)}_{*}|\big)^2$ requires at least $K' = \Omega(d)$ iterations of Adam with constant probability over the randomness of $\bm{x}_0$.
\end{theorem}
We prove \Cref{lb-diag} in \Cref{lb-diag-pf}.

\begin{figure}
    \centering
    \subfloat[\centering Non-diagonal Hessian]{{\includegraphics[width=0.5\linewidth]{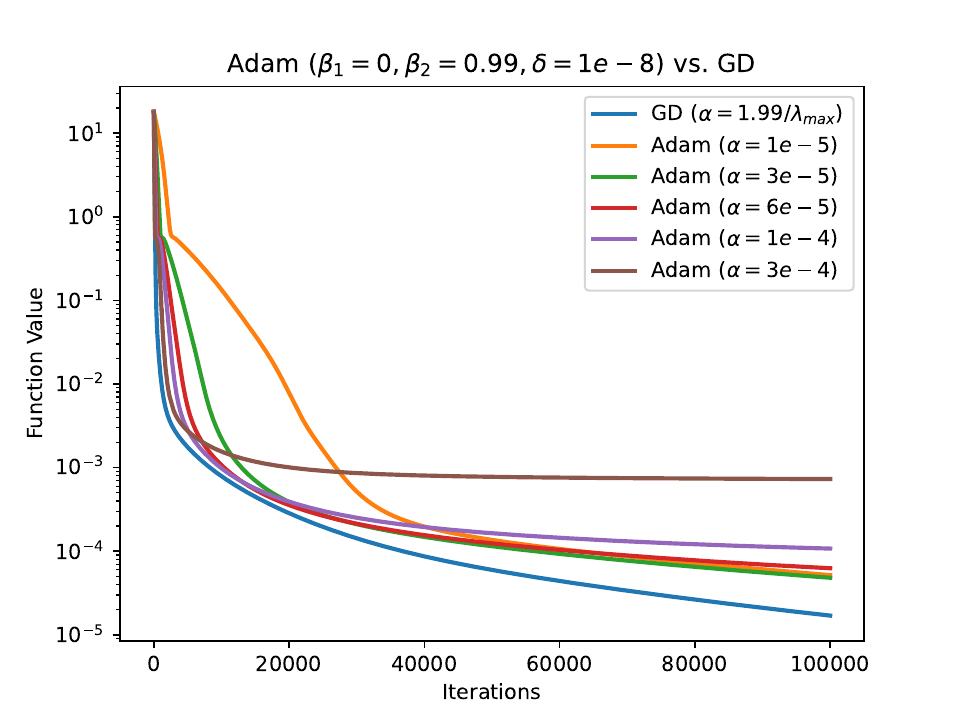}}\label{fig:non-diag}}
    \hspace{-0.3cm}
    \subfloat[\centering Diagonal Hessian w/ same condition no. as Fig. \ref{fig:non-diag}]{{\includegraphics[width=0.5\linewidth]{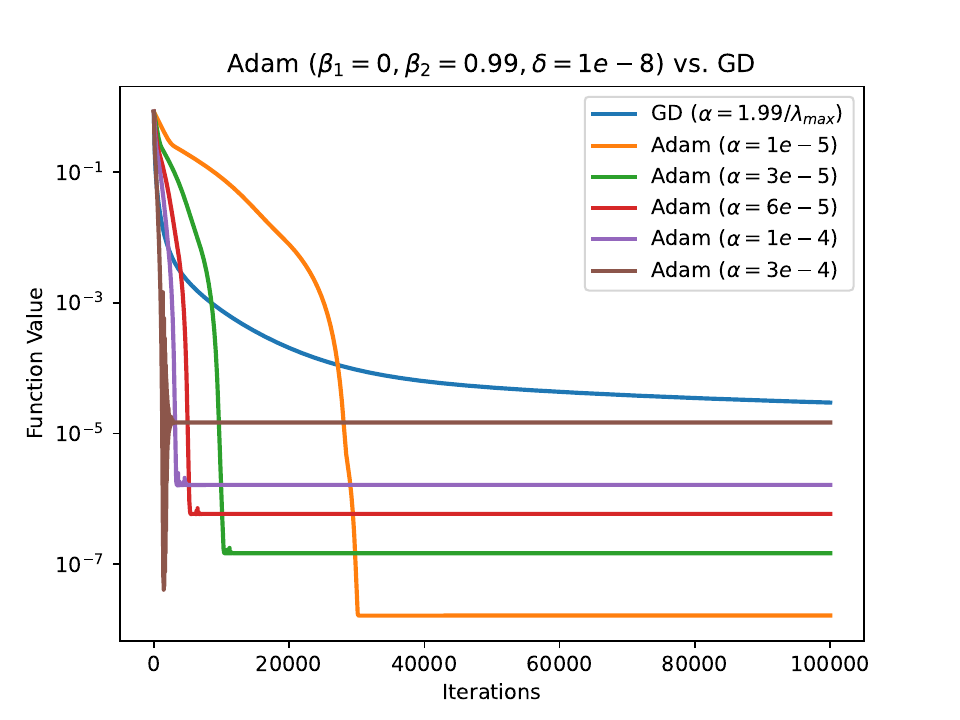}}\label{fig:diag}}
    \caption{In \ref{fig:non-diag}, the Hessian is the first realization of $\bm{Q}$ in \Cref{tab1}. In \ref{fig:diag}, the Hessian is the \textit{diagonal} matrix containing the eigenvalues of the first realization of $\bm{Q}$ in \Cref{tab1}. We compare the standard version of Adam in \cref{eq:1-dec28} with different step-sizes $\alpha$ against GD with step-size $\alpha = {1.99}/{\lambda_\text{max}(\bm{Q})}$. \textit{Consistent with our theoretical prediction, \textbf{GD is better than Adam} in \ref{fig:non-diag}, whereas \textbf{Adam is better than GD} in \ref{fig:diag}.}}
    \label{fig:adam-expt}
\end{figure}

\subsection{Proof Outline of Theorem~\ref{thm3} and Asymptotic Behavior of Adam}
\label{pf-sketch}
We shall now sketch our proof technique for \Cref{thm3} (detailed proof in \Cref{pf-1}), i.e., in the case of a general $\bm{Q}$. The proof techniques of other convergence bounds for Adam (Theorems \ref{thm3-diag} and \ref{thm-adanorm-plc}) are similar. This will also shed some light on the behavior of Adam if we let it run forever.

First, we shall introduce some definitions. To that end, recall that $\bm{D} = \bm{Q} \circ \bm{\text{{I}}}$ and the eigen-decomposition of $\bm{Q}$ is $\bm{U} \bm{\Lambda} \bm{U}^\top$; thus, $\bm{g}_k = \nabla f(\bm{x}_k) = \bm{U} \bm{\Lambda} \bm{U}^\top (\bm{x}_k - \bm{x}_{*})$. Now let:
\begin{equation}
    \label{eq:6-dec29-def}
    \bar{\bm{g}}_k := \bm{\Lambda}^{1/2} {\bm{U}}^\top (\bm{x}_k - \bm{x}_{*}).
\end{equation}
Note that $f(\bm{x}_{k}) = \frac{1}{2}\|\bar{\bm{g}}_{k}\|_2^2$ and $\bm{g}_k = \bm{U} \bm{\Lambda}^{1/2} \bar{\bm{g}}_k$. Next let:
\begin{equation}
    \label{eq:7-dec29-def}
    \widehat{\bm{g}}_k := \bm{D}^{-1} \bm{g}_k = (\bm{D}^{-1} \bm{U} \bm{\Lambda}^{1/2}) \bar{\bm{g}}_k = \widehat{\bm{Q}} \bar{\bm{g}}_k.
\end{equation}
Also, recall that $\bar{\bm{Q}} = \bm{D}^{-1/2} \bm{U} \bm{\Lambda}^{1/2}$. Let:
\begin{equation}
    \sigma_\text{min}(\bar{\bm{Q}}) = \sqrt{\mu_1} \text{ and } \sigma_\text{max}(\bar{\bm{Q}}) = \sqrt{\mu_2}.
\end{equation}
Since $\text{cond}(\bar{\bm{Q}}) = \sqrt{\overline{\kappa}}$, we have $\overline{\kappa} = \big(\frac{\mu_2}{\mu_1}\big)$.
\vspace{0.2 cm}
\\
Adam's update rule in \cref{eq:3-dec28} with $\beta_1 = 0$ is $x^{(i)}_{k+1} = x^{(i)}_{k} - \frac{\alpha {g}_k^{(i)}}{\sqrt{{v}_{k}^{(i)} + \max(({g}_0^{(i)})^2, \phi^2)\delta_{k}^{2}}}$. Unfolding the recursion of ${v}_{k}^{(i)}$ in \cref{eq:2-dec28}, we get ${v}_{k}^{(i)} = {\sum_{l=0}^{k} \beta_2^{k-l} \big({g}_l^{(i)}\big)^2}$. Let ${\bm{V}}_k$ be a diagonal matrix whose $i^{\text{th}}$ diagonal element is $\frac{1}{\sqrt{{v}_{k}^{(i)} + \max(({g}_0^{(i)})^2, \phi^2) \delta_k^{2}}}$. Then, we can rewrite Adam's update rule in matrix-vector notation as:
\begin{equation}
    \label{eq:7-dec29}
    \bm{x}_{k+1} = \bm{x}_k - \alpha {\bm{V}}_k {\bm{U}} \bm{\Lambda} {\bm{U}}^\top (\bm{x}_k - \bm{x}_{*}).
\end{equation}
After manipulating \cref{eq:7-dec29} a bit, we obtain the following equivalent equation:
\begin{equation}
    \label{eq:8-dec29}
    \bar{\bm{g}}_{k+1} = \Big(\bm{\text{I}} - \alpha \bm{\Lambda}^{1/2} \bm{U}^\top {\bm{V}}_k {\bm{U}} \bm{\Lambda}^{1/2} \Big)\bar{\bm{g}}_k,
\end{equation}
where $\bar{\bm{g}}_k = \bm{\Lambda}^{1/2} {\bm{U}}^\top (\bm{x}_k - \bm{x}_{*})$ is as defined in \cref{eq:6-dec29-def}. 
From \cref{eq:7-dec29-def}, recall that $\widehat{\bm{g}}_k = \bm{D}^{-1} \bm{g}_k$. Let:
\begin{equation}
    \widehat{v}_{k}^{(i)} := {\sum_{l=0}^{k} \beta_2^{k-l} \big(\widehat{g}_l^{(i)}\big)^2} = {\sum_{l=0}^{k} \beta_2^{k-l} \big({{g}_l^{(i)}}/{q_i}\big)^2} = {{v}_{k}^{(i)}}/{q_i^2} \text{ and } \phi_i := {\phi}/{q_i}.
\end{equation}
Let $\widehat{\bm{V}}_k$ be a diagonal matrix whose $i^{\text{th}}$ diagonal element is $\frac{1}{\sqrt{\widehat{v}_{k}^{(i)} + \max((\widehat{g}_0^{(i)})^2, {\phi_i^2})\delta_k^{2}}}$; note that ${\bm{V}}_k = \bm{D}^{-1/2} \widehat{\bm{V}}_k \bm{D}^{-1/2}$. Plugging this into \cref{eq:8-dec29} and recalling that $\bar{\bm{Q}} = \bm{D}^{-1/2} \bm{U} \bm{\Lambda}^{1/2}$, we get:
\begin{equation}
    \label{eq:12-dec29}
    \bar{\bm{g}}_{k+1} = {\Big(\bm{\text{I}} - \alpha \bar{\bm{Q}}^\top \widehat{\bm{V}}_k \bar{\bm{Q}} \Big)} \bar{\bm{g}}_k.
\end{equation}
Since $\widehat{\bm{V}}_k$ is a diagonal matrix with positive diagonal elements, $\bar{\bm{Q}}^\top \widehat{\bm{V}}_k \bar{\bm{Q}}$ is symmetric positive-definite (PD). Let $\bm{A}_k = \bm{\text{I}} - \alpha \bm{B}_k$, where $\bm{B}_k = \bar{\bm{Q}}^\top \widehat{\bm{V}}_k \bar{\bm{Q}}$. 

\begin{remark}[\textbf{Basis change}]
    \label{rmk-sym}
    The basis change in going from \cref{eq:7-dec29} to \cref{eq:8-dec29} has been chosen to ensure that $\bm{A}_k$ is of the form $\bm{\text{I}} - \alpha \bm{B}_k$, where $\bm{B}_k$ is symmetric PD; in \cref{eq:12-dec29}, $\bm{B}_k = \bar{\bm{Q}}^\top \widehat{\bm{V}}_k \bar{\bm{Q}}$. If $\bm{B}_k$ is not symmetric PD, then the subsequent analysis does not hold. In particular, we cannot directly start the subsequent analysis from \cref{eq:7-dec29} (after subtracting $\bm{x}_{*}$ from both sides) because in this case $\bm{B}_k$ would be ${\bm{V}}_k {\bm{U}} \bm{\Lambda} {\bm{U}}^\top$ which is not symmetric unless $\bm{U} = \bm{\text{{I}}}$, i.e., $\bm{Q} = \bm{\Lambda}$ is diagonal.
\end{remark}

\begin{remark}[\textbf{Pseudo-linear system}]
    \label{rmk-plin}
    For a moment, if we ignore the fact that $\bm{A}_k$ depends on the iteration number $k$, then \cref{eq:12-dec29} is a linear system that we can analyze using standard ideas from linear algebra. Our proof strategy is to treat \cref{eq:12-dec29} like a \enquote{pseudo-linear} system and involves tracking the evolution of $\bm{A}_k$ so that we can utilize standard linear algebra ideas, even though \cref{eq:12-dec29} is not truly linear.
\end{remark}

\begin{remark}[\textbf{Rationale of going from \cref{eq:8-dec29} to \cref{eq:12-dec29}}]
    \label{rmk-sym-2}
    The subsequent steps in the analysis require us to bound $\sigma_\textup{min}(\bm{B}_k)$ and $\sigma_\textup{max}(\bm{B}_k)$ for each $k$. Eq. (\ref{eq:8-dec29}) is equivalent to \cref{eq:12-dec29}, but the latter can enable us to obtain tighter bounds for $\sigma_\textup{min}(\bm{B}_k)$ and $\sigma_\textup{max}(\bm{B}_k)$ by allowing us to theoretically \enquote{cancel out} the dependence on $\kappa$ to an extent (albeit introducing other quantities) -- which improves the iteration complexity. Specifically, with \cref{eq:8-dec29}, we could not improve the $\mathcal{O}(\kappa)$ dependence. But with \cref{eq:12-dec29}, we get an $\mathcal{O}(\sqrt{\kappa})$ factor in the end for diagonally dominant $\bm{Q}$.
\end{remark}
\noindent Let $\widehat{w}_k^{(j)} = \widehat{v}_{k}^{(j)} + \max((\widehat{g}_0^{(j)})^2, \phi_j^2) \delta_k^{2}$. 
Recalling that $\sigma_\text{min}(\bar{\bm{Q}}) = \sqrt{\mu_1}$ and $\sigma_\text{max}(\bar{\bm{Q}}) = \sqrt{\mu_2}$, we have $\sigma_{\text{max}}(\bar{\bm{Q}}^\top \widehat{\bm{V}}_k \bar{\bm{Q}}) \leq \frac{\mu_2}{\sqrt{\min_{j \in [d]} \widehat{w}_k^{(j)}}}$ and $\sigma_{\text{min}}(\bar{\bm{Q}}^\top \widehat{\bm{V}}_k \bar{\bm{Q}}) \geq \frac{\mu_1}{\sqrt{\max_{j \in [d]} \widehat{w}_k^{(j)}}}$. Now at least until ${\min_{j \in [d]} \widehat{w}_k^{(j)}} \geq (\mu_2 \alpha)^2$ starting from the first iteration, $\bm{A}_k$ \textit{is positive semi-definite} and we have: 
\begin{equation}
    \sigma_{\text{max}}(\bm{A}_k) \leq 1 - \frac{\mu_1 \alpha}{\sqrt{\max_{j \in [d]} \widehat{w}_k^{(j)}}}.
\end{equation}
So at least until ${\min_{j \in [d]} \widehat{w}_k^{(j)}} \geq (\mu_2 \alpha)^2$ holds, \textit{descent is ensured in each iteration or there is \enquote{continuous descent}} and we have:
\begin{flalign}
    \label{eq:14-dec29-2}
    \|\bar{\bm{g}}_{k+1}\|_2 & \leq \sigma_{\text{max}}(\bm{A}_k) \|\bar{\bm{g}}_k\|_2 
    \leq \Bigg(1 - \frac{\mu_1 \alpha}{\sqrt{\max_{j \in [d]} \widehat{w}_k^{(j)}}}\Bigg) \|\bar{\bm{g}}_k\|_2.
\end{flalign}
In \Cref{lem1-sep30}, we analyze the decay rule of \cref{eq:14-dec29-2} until there is \textit{continuous descent}; suppose this happens for $\tilde{K}$ iterations. Below we present an abridged version of \Cref{lem1-sep30}.
\begin{lemma}[\textbf{Abridged version of \Cref{lem1-sep30}}]
\label{lem1-sep30-mini}
Recall $\widehat{\kappa} = \textup{cond}(\widehat{\bm{Q}})$. Choose $\alpha \leq \frac{\max(3 \widehat{\kappa} \sqrt{d} B, \phi/q_\textup{min})}{\sqrt{2} \mu_1}$, $\beta_2 = 1 - \frac{\alpha^2 \mu_1^2}{8 \max(9 \widehat{\kappa}^2 d B^2, \phi^2/q_\textup{min}^2)}$ and $\delta_k = \frac{1}{\sqrt{2(1-\beta_2)}} \beta_2^{\frac{k}{2}}$ for $k \geq 0$. Then for any $k \leq \tilde{K}$, we have:
\begin{equation}
    \label{eq:16-feb4}
    \|\bar{\bm{g}}_{k+1}\|_2 \leq \exp\Big(-2\Big(\beta_2^{-(k+1)/2} - 1\Big)\Big)\|\bar{\bm{g}}_0\|_2.
\end{equation}
\end{lemma}
\begin{remark}[\textbf{Proof of \Cref{lem1-sep30}}]
    \label{rmk-2-stage}
    In \Cref{lem1-sep30}, we employ a two-stage analysis procedure wherein we first obtain a loose bound for $\|\bar{\bm{g}}_{k+1}\|_2$ in terms of $\|\bar{\bm{g}}_0\|_2$ (specifically, $\|\bar{\bm{g}}_{k}\|_2^2 \leq \gamma^{k}\|\bar{\bm{g}}_{0}\|_2^2$ with $\gamma \approx \beta_2$), and then obtain a tighter bound by using this loose bound.
\end{remark}
We will now elaborate on the two stages in \Cref{rmk-2-stage}. 
\vspace{0.1 cm}
\\
\textbf{First Stage.} Let us define $\widehat{v}_{k} := \sum_{j \in [d]} \widehat{v}_{k}^{(j)} = {\sum_{l=0}^{k} \beta_2^{k-l} \|\widehat{\bm{g}}_l\|_2^2}$. With some simple algebra, we obtain the following relaxation of \cref{eq:14-dec29-2} that is more conducive to analysis:
\begin{flalign}
    \label{eq:17-feb4}
    \|\bar{\bm{g}}_{k+1}\|_2 \leq \Bigg(1 - \frac{\mu_1 \alpha}{\sqrt{\widehat{v}_{k} + \max\big(\|\widehat{\bm{g}}_0\|_\infty^2, \phi^2/q_\text{min}^2) {\delta}_k^2}}\Bigg)\|\bar{\bm{g}}_k\|_2.
\end{flalign}
From \cref{eq:17-feb4}, we have $\|\bar{\bm{g}}_{0}\|_2 \geq \|\bar{\bm{g}}_{1}\|_2 \geq \ldots \geq \|\bar{\bm{g}}_{\tilde{K}+1}\|_2$. Using this with some simple linear algebra, we obtain the following crude bound: $\widehat{v}_k \leq \frac{\widehat{\kappa}^2 \|\widehat{\bm{g}}_0\|_2^2}{1-\beta_2}$ $\forall$ $k \leq \tilde{K}+1$. Combining this with our choice of $\delta_k$ and using \Cref{dist}, we get:
\begin{equation}
    \widehat{v}_k + \max\big(\|\widehat{\bm{g}}_0\|_\infty^2, \phi^2/q_\text{min}^2) {\delta}_k^2 \leq \frac{2 \max \big(9 \widehat{\kappa}^2 d B^2, \phi^2/q_\text{min}^2)}{1-\beta_2}.
\end{equation}
Note that the above bound is independent of $k$. Plugging this into \cref{eq:17-feb4} yields $\|\bar{\bm{g}}_{k+1}\|_2^2 \leq \gamma \|\bar{\bm{g}}_{k}\|_2^2$, where $\gamma := \Big(1 - \frac{{\alpha} \mu_1 \sqrt{1-\beta_2}}{\sqrt{2} \max(3 \widehat{\kappa} \sqrt{d} B, \phi/q_\text{min})}\Big)$. Unfolding this recursion yields:
\begin{equation}
    \label{eq:19-feb4}
        \|\bar{\bm{g}}_{k+1}\|_2^2 \leq \gamma^{k+1}\|\bar{\bm{g}}_{0}\|_2^2.
\end{equation}
The above bound is loose but we will obtain a tighter bound using this bound in the next stage.
\vspace{0.1 cm}
\\
\textbf{Second Stage.} With our choice of $\beta_2$, we get $\gamma = 1 - \frac{\alpha^2 \mu_1^2}{4 \max(9 \widehat{\kappa}^2 d B^2, \phi^2/q_\text{min}^2)}$ and $\beta_2 - \gamma = {1 - \beta_2}$. Using \cref{eq:19-feb4} and with some linear algebra, we get $\|\widehat{\bm{g}}_{l}\|_2^2 \leq \gamma^l \widehat{\kappa}^2 \|\widehat{\bm{g}}_{0}\|_2^2$. Thus,
\begin{flalign*}
    \widehat{v}_k = \sum_{l=0}^{k} \beta_2^{k-l} \|\widehat{\bm{g}}_{l}\|_2^2 \leq \sum_{l=0}^{k} \beta_2^{k-l} \gamma^l \widehat{\kappa}^2 \|\widehat{\bm{g}}_{0}\|_2^2 = \Bigg(\frac{\beta_2^{k+1} - \gamma^{k+1}}{\beta_2 - \gamma}\Bigg) \widehat{\kappa}^2 \|\widehat{\bm{g}}_{0}\|_2^2 \leq \Big(\frac{\beta_2^{k+1}}{1-\beta_2}\Big) \widehat{\kappa}^2 \|\widehat{\bm{g}}_{0}\|_2^2.
\end{flalign*}
Combining this with our choice of $\delta_k$ (which was specifically chosen to enable this) and using \Cref{dist}, we get:
\begin{equation} 
    \label{eq:20-feb4}
    \widehat{v}_k + \max\big(\|\widehat{\bm{g}}_0\|_\infty^2, \phi^2/q_\text{min}^2) {\delta}_k^2 \leq \Big(\frac{2 \beta_2^{k+1}}{1-\beta_2}\Big) \max\big(9 \widehat{\kappa}^2 d B^2, \phi^2/q_\text{min}^2\big).
\end{equation}
Plugging this into \cref{eq:17-feb4} (after using the fact that $\exp(-z) \geq 1-z$ $\forall$ $z \in \mathbb{R}$), we get:
\begin{equation}
    \label{eq:21-feb4}
    \|\bar{\bm{g}}_{k+1}\|_2 \leq \exp\Bigg(- \frac{\alpha \mu_1 \sqrt{1-\beta_2} \beta_2^{-{(k+1)/2}}}{\sqrt{2} \max\big(3 \widehat{\kappa} \sqrt{d} B, \phi/q_\text{min}\big)} \Bigg)\|\bar{\bm{g}}_k\|_2.
\end{equation}
\Cref{eq:21-feb4} is our final \textit{per-iteration} decay rule of interest. Finally, with some more algebra and simplification, we get the result of \Cref{lem1-sep30-mini} (i.e., \cref{eq:16-feb4}). 
\vspace{0.1 cm}
\\
\noindent Let us come back to the proof of \Cref{thm3}. From \cref{eq:16-feb4}, in order to have $\|\bar{\bm{g}}_{K^{*}}\|_2 \leq \varepsilon$ (and thus, $f(\bm{x}_{{K}^{*}}) \leq \frac{\varepsilon^2}{2}$), we can choose: 
\begin{equation}
    \label{eq:16-dec29}
    K^{*} = \mathcal{O}\Big({\log\Big(\log^2\Big(\frac{\|\bar{\bm{g}}_0\|_2}{\varepsilon}\Big)\Big)}\Big/{\log \big({1}/{\beta_2}\big)}\Big).
\end{equation}
For the above result to be valid, we must have $K^{*} \leq \tilde{K}$. In \Cref{lem1-oct1}, we obtain the following bound for $\tilde{K}$:
$$\tilde{K} = {\log\Bigg(\frac{\min_{j \in [d]} (\widehat{g}_0^{(j)})^2}{(\mu_2 \alpha)^2} + \frac{\max(\min_{j \in [d]} (\widehat{g}_0^{(j)})^2, {\phi^2}/{q_\textup{max}^2})}{{2(1-\beta_2)(\mu_2 \alpha)^2}}\Bigg)}\Bigg/{\log \big({1}/{\beta_2}\big)}.$$
We will skip the proof details of \Cref{lem1-oct1} as it is straightforward. Imposing $K^{*} \leq \tilde{K}$ gives us an upper bound for the step-size $\alpha$, viz.,
\begin{equation}
    \label{eq:23-feb4}
    \alpha \leq \mathcal{O}\Bigg(\Bigg(\frac{\mu_1^2 \max(\min_{j \in [d]} (\widehat{g}_0^{(j)})^2, {\phi^2}/{q_\textup{max}^2})}{\mu_2^2 \max( \widehat{\kappa}^2 d B^2, \phi^2/q_\textup{min}^2)} \Bigg)^{1/4}  \frac{\max\big(\widehat{\kappa} \sqrt{d} B, \phi/q_\textup{min}\big)}{\mu_1}\log^{-1/2}\Big(\frac{\|\bar{\bm{g}}_0\|_2}{\varepsilon}\Big)\Bigg).
\end{equation}
From \Cref{lem3-jan7}, we have that $\min_{j \in [d]} |\widehat{g}_0^{(j)}|  > \frac{2B}{d} \Big(\frac{\log ({1}/{p})}{\zeta\big(1 + \frac{\log (1/p)}{d}\big)}\Big)^{1/\theta}$ with a probability of at least $p$. Let $\widetilde{p} = \Big(\frac{\log ({1}/{p})}{\zeta\big(1 + \frac{\log (1/p)}{d}\big)}\Big)^{1/\theta}$. Additionally, we can show that $\|\bar{\bm{g}}_0\|_2 \leq \mathcal{O}(\frac{\sqrt{d} B}{\rho_1})$. Thus, with a probability of at least $p$, we can set:
\begin{equation}
    \label{eq:24-feb4}
    \alpha = \mathcal{O}\Bigg(\frac{1}{\sqrt{\mu_1 \mu_2}} \Bigg({\max\Big(\frac{\widetilde{p}^2 B^2}{d^2}, \frac{\phi^2}{q_\textup{max}^2}\Big) \max\Big(\widehat{\kappa}^2 d B^2, \frac{\phi^2}{q_\textup{min}^2}\Big)} \Bigg)^{1/4} \log^{-1/2}\Big(\frac{\sqrt{d} B}{\rho_1 \varepsilon}\Big)\Bigg).
\end{equation}
Recalling the value of $\beta_2$ and using \Cref{log-fact}, we get $\log \frac{1}{\beta_2} \geq \frac{\alpha^2 \mu_1^2}{8 \max(9 \widehat{\kappa}^2 d B^2, \phi^2/q_\textup{min}^2)}$. Finally, plugging this into \cref{eq:16-dec29} with the value of $\alpha$ from \cref{eq:24-feb4} and setting $\phi = \widehat{\kappa} \sqrt{d} B q_\textup{min}$ gives us the result of \Cref{thm3}.
\vspace{0.15 cm}
\\
An important part of the above analysis is ensuring that we attain $\mathcal{O}({\varepsilon^2})$ function value\footnote{We implicitly assume that $\min_{\bm{x} \in \mathbb{R}^d} f(\bm{x}) = 0$ here. A practical way to stop the algorithm can be based on the gradient norm instead of function value.} before we can no longer guarantee that $\bm{A}_k = \bm{\text{I}} - \alpha \bar{\bm{Q}}^\top \widehat{\bm{V}}_k \bar{\bm{Q}}$ is positive semi-definite; that is why we impose $K^{*} \leq \tilde{K}$ above and get a constraint on the step-size $\alpha$ (\cref{eq:23-feb4}). 
{Moreover, we should stop the algorithm at this point because we cannot guarantee {continuous descent} forever.} If there is continuous descent forever, then ${\min_{j \in [d]} \widehat{w}_k^{(j)}}$ will also keep decreasing with a decaying $\delta_k$ (and a decaying $\delta_k$ is required for a \enquote{fast decay} rule like \cref{eq:16-feb4} to hold). But when ${\min_{j \in [d]} \widehat{w}_k^{(j)}} < (\frac{\mu_2 \alpha}{2})^2$, then it may happen that $\lambda_\text{min}(\bm{A}_k) < -1$ and in that case, we can no longer guarantee descent. Therefore, Adam may not enjoy continuous descent forever and it \textit{may not converge to zero function value asymptotically}. This is consistent with what we observed in Figures \ref{fig:non-diag} and \ref{fig:diag}. We formalize this next in \Cref{sec-fixed-pt}.

\subsection{Fixed Points of Adam}
\label{sec-fixed-pt}
Here we shall characterize the asymptotic point(s) to which Adam may converge.

\begin{definition}[\textbf{Fixed point}]
Consider an iterative and deterministic optimization algorithm $\mathcal{A}$ operating on a function $f(\bm{x})$. Suppose its output at iteration number $k$ is $\bm{x}_k$. We define a fixed point of $\mathcal{A}$ to be a point $\overline{\bm{x}}$ such that $\lim_{k \to \infty} f(\bm{x}_k) = f(\overline{\bm{x}})$.
\end{definition}
We have the following result regarding the fixed point(s) $\overline{\bm{x}}$ of Adam for $f(\bm{x}) = \frac{1}{2}(\bm{x} - \bm{x}_{*})^{\top} \bm{Q} (\bm{x} - \bm{x}_{*})$ assuming $\lim_{k \to \infty} \delta_k$ exists.
\begin{theorem}[\textbf{Fixed point(s) of Adam}]
\label{thm-fixed-pt}
$\overline{\bm{x}} = \bm{x}_{*}$ is a trivial fixed point. There may exist non-trivial fixed points depending on the value of $\lim_{k \to \infty} \delta_k = \delta \geq 0$ (note that this also covers the case of the constant sequence $\delta_k = \delta$ $\forall$ $k \geq 0$, which is what is used in practice). 
\begin{itemize}
    \item If $\delta = 0$, there exists a fixed point $\overline{\bm{x}} \neq \bm{x}_{*}$ with non-zero gradient $\overline{\bm{g}} = \bm{Q} (\overline{\bm{x}} - \bm{x}_{*})$ such that $\|\overline{\bm{g}}\|_1 \geq \frac{\alpha d \sqrt{1 - \beta_2}}{2}$. 
    \item If $\delta < \frac{\alpha}{2 \overline{R}}$ where $\overline{R} = \max_{j \in [d]} \max(|{g}_0^{(j)}|, \phi)$, there exists a fixed point $\overline{\bm{x}} \neq \bm{x}_{*}$ with non-zero gradient $\overline{\bm{g}} = \bm{Q} (\overline{\bm{x}} - \bm{x}_{*})$ such that $\|\overline{\bm{g}}\|_\infty \geq \sqrt{(1 - \beta_2)\big(\frac{\alpha^2}{4} - \overline{R}^2 \delta^2\big)}$.
    \item If $\delta > \frac{\kappa \alpha}{2 \overline{r}}$ where $\overline{r} := \min_{j \in [d]} \max(|{g}_0^{(j)}|, \phi)$, $\overline{\bm{x}} = \bm{x}_{*}$ is the only fixed point.
\end{itemize}
\end{theorem}
We prove \Cref{thm-fixed-pt} in \Cref{pf-thm-fixed-pt}. For the standard variant of Adam wherein $\max\big(({g}_0^{(i)})^2, \phi^2\big)\delta_{k}^{2}$ is replaced by $\phi^2 \delta^{2}$ in the denominator of \cref{eq:3-dec28}, $\overline{R} = \overline{r} = \phi$ in \Cref{thm-fixed-pt}.

\Cref{thm-fixed-pt} tells us that when $\delta < \frac{\alpha}{2 \overline{R}}$, Adam may \textit{not} converge to the minimum function value $f(\bm{x}_{*})$. When $\delta > \frac{\kappa \alpha}{2 \overline{r}}$, Adam will converge to $f(\bm{x}_{*})$. In this case, the effective step-size of Adam in each coordinate is always less than ${2}/{\kappa}$ and so we expect behavior similar to GD, i.e., eventual convergence to the optimum. {Unfortunately, we cannot formally assert if there exists a non-trivial fixed point (i.e., $\overline{\bm{x}} \neq \bm{x}_{*}$) when $\delta \in \big[\frac{\alpha}{2 \overline{R}}, \frac{\kappa \alpha}{2 \overline{r}}\big]$, but we strongly suspect that there are indeed non-trivial fixed points in many cases (as is the case in Figures \ref{fig:non-diag} and \ref{fig:diag}).}
\\
\\
We will now move on from quadratics and extend our analysis to more general functions.

\section{Beyond Quadratics: Per-Coordinate Smoothness and a Modified Polyak-\L ojasiewicz (PL) Condition}
\label{sec-smooth-plc}
Here we focus on functions that satisfy per-coordinate Lipschitz smoothness and a modified version of the Polyak-\L ojasiewicz (PL) condition. Let us begin by defining the standard {global} version (with respect to the coordinates) of smoothness and the standard PL condition.
\begin{definition}[\textbf{Smoothness}]
\label{def-smooth}
$f: \mathbb{R}^d \xrightarrow{} \mathbb{R}$ is said to be $L$-smooth if $\forall$ $\bm{x}, \bm{y} \in \mathbb{R}^d$, it holds that $\|\nabla f(\bm{y}) - \nabla f(\bm{x})\|_2 \leq L \|\bm{y} - \bm{x}\|_2$.
\end{definition}

\begin{definition}[\textbf{Polyak- \L ojasiewicz (PL) Condition \citep{karimi2016linear}}]
    \label{def-PL}
    $f: \mathbb{R}^d \xrightarrow{} \mathbb{R}$ with $f^{*} := \min_{\bm{x} \in \mathbb{R}^d} f(\bm{x})$ is said to be $\mu$-PL if $\exists$ $\mu > 0$ such that: $$\|\nabla f(\bm{x})\|_2^2 \geq 2 \mu (f(\bm{x}) - f^{*}) \text{ } \forall \text{ } \bm{x} \in \mathbb{R}^d.$$
\end{definition}

\begin{theorem}[\textbf{GD: Smooth and PL \citep{karimi2016linear}}]
    \label{thm-gd-plc}
    Suppose $f: \mathbb{R}^d \xrightarrow{} \mathbb{R}$ is $L$-smooth and $\mu$-PL, and let $\kappa = \frac{L}{\mu}$ be the condition number. 
    Suppose our initialization is $\bm{x}_0$. Set the step-size $\alpha = \frac{1}{L}$. Then, $f(\bm{x}_{\widehat{K}}) - f^{*} \leq \varepsilon$ in
    \begin{equation*}
        \widehat{K} = \mathcal{O}\Bigg({\kappa} \log \Bigg(\frac{f(\bm{x}_{0}) - f^{*}}{\varepsilon}\Bigg) \Bigg)
    \end{equation*}
    iterations of GD.
\end{theorem}
\begin{remark}
    \label{rmk-7-jan20}
    Similar to the quadratic case (\Cref{thm5}), the iteration complexity of GD depends on the condition number, i.e., $\kappa = \frac{L}{\mu}$. 
\end{remark}
{We will now introduce \textit{per-coordinate} smoothness and a \textit{modified} version of the PL condition for functions satisfying per-coordinate smoothness under which we show that Adam can do better than GD.}

\begin{definition}[\textbf{Per-Coordinate Smoothness}]
\label{def-1-jan20}
$f: \mathbb{R}^d \xrightarrow{} \mathbb{R}$ is said to be $\{L_i\}_{i=1}^d$-smooth if for all $\bm{x}, \bm{y} \in \mathbb{R}^d$ and for each coordinate $i \in [d]$, it holds that:
\begin{equation*}
    |\nabla f(\bm{y})^{(i)} - \nabla f(\bm{x})^{(i)}| \leq L_i |y^{(i)} - x^{(i)}|.
\end{equation*}  
\end{definition}
Similar assumptions have been made in the literature \citep{bernstein2018signsgd,wright2015coordinate,richtarik2014iteration}. 

\begin{definition}[\textbf{Smoothness-Dependent Global PL Condition}]
\label{new-PLC}
Suppose $f: \mathbb{R}^d \xrightarrow{} \mathbb{R}$ is $\{L_i\}_{i=1}^d$-smooth. Then $f$ is said to obey the smoothness-dependent $\tilde{\mu}$-global-PL condition if $\exists$ $\tilde{\mu} \in (0, \min_{i \in [d]} L_i]$ such that:
\begin{equation*}
    \sum_{i=1}^d \frac{(\nabla f(\bm{x})^{(i)})^2}{\sqrt{2 L_i}} \geq \sqrt{2 \tilde{\mu}} \big(f(\bm{x}) - f^{*}\big) \text{ } \forall \text{ } \bm{x} \in \mathbb{R}^d, \text{ where $f^{*} = \min_{\bm{z} \in \mathbb{R}^d} f(\bm{z})$.}
\end{equation*}
\end{definition}
We are now ready to state our main assumption. 
\begin{assumption}
    \label{asmp-3-jan19}
    $f: \mathbb{R}^d \xrightarrow{} \mathbb{R}$ is $\{L_i\}_{i=1}^d$-smooth and it also satisfies the smoothness-dependent $\tilde{\mu}$-global-PL condition. We are privy to the values of $L_\textup{min} := \min_{i \in [d]} L_i$, $L_\textup{max} := \min_{i \in [d]} L_i$, $\tilde{\mu}$ and $f^{*} = \min_{\bm{z} \in \mathbb{R}^d} f(\bm{z})$.
\end{assumption}
In many modern machine learning problems, $f^{*} \approx 0$; for e.g., if $f$ is the objective function of an over-parameterized deep network. Also, it can be shown that under \Cref{asmp-3-jan19}, $f$ satisfies standard global smoothness (\Cref{def-smooth}) and the standard global PL condition (\Cref{def-PL}).
\begin{corollary}
    \label{cor1-jan20}
    Suppose \Cref{asmp-3-jan19} holds. Then $f$ is $L_\textup{max}$-smooth and $\sqrt{\tilde{\mu} L_\textup{min}}$-PL.
\end{corollary}
For the sake of completeness, we prove \Cref{cor1-jan20} in \Cref{cor1-jan20-pf}. Based on this and recalling \Cref{thm-gd-plc} (and \Cref{rmk-7-jan20}), we can make the following remark.
\begin{remark}
    \label{rmk8-jan20}
    Under \Cref{asmp-3-jan19}, the iteration complexity of GD scales as $\frac{L_\textup{max}}{\sqrt{\tilde{\mu} L_\textup{min}}}$.  
\end{remark}
Coming back to Adam, we need one more assumption at the initialization before we can state our result.
\begin{assumption}[\textbf{Per-Coordinate PL-like Condition at Initialization}]
    \label{lower-bound}
    At our initialization $\bm{x}_0$, the following holds:
    \begin{equation*}
        {(\nabla f(\bm{x}_0)^{(i)})^2} \geq 2 \mu_{i,0} \big(f(\bm{x}_0) - f^{*}\big)
    \end{equation*}
    for some $\mu_{i,0} \leq L_i$ for each $i \in [d]$. Further, the condition number-like quantity $\frac{L_i}{\mu_{i,0}}$ along each coordinate $i \in [d]$ at $\bm{x}_0$ is bounded, i.e., $\exists$ ${\kappa}_{\textup{max},0}$ such that $\frac{L_i}{\mu_{i,0}} \leq {\kappa}_{\textup{max},0} \text{ } \forall \text{ } i \in [d]$. We are privy to the value of ${\kappa}_{\textup{max},0}$.
\end{assumption}
$\max_{i \in [d]} \frac{L_i}{\mu_{i,0}}$ has an implicit dependence on $d$ and so ${\kappa}_{\textup{max},0}$ (which is an upper bound for $\max_{i \in [d]} \frac{L_i}{\mu_{i,0}}$) will also have an implicit dependence on $d$. In particular, ${\kappa}_{\textup{max},0}$ should grow with $d$.

\begin{theorem}[\textbf{Adam: Per-Coordinate Smooth and Smoothness-Dependent PL}]
\label{thm-adanorm-plc}
Suppose our initialization is $\bm{x}_0$ and Assumptions \ref{asmp-3-jan19} and \ref{lower-bound} hold. There exist $\alpha$, $\beta_2$, $\delta_k$ and $\phi$ such that $f(\bm{x}_{{K}^{*}}) - f^{*} \leq \varepsilon$ in 
\begin{equation*}
    K^{*} = \mathcal{O}\Bigg(\sqrt{\frac{L_\textup{max}}{\tilde{\mu}}} \min \Bigg(\sqrt{\frac{L_\textup{max}}{L_\textup{min}}},\sqrt{{\kappa_{\textup{max},0}}}\Bigg) \log\Big(\frac{f(\bm{x}_{0}) - f^{*}}{\varepsilon}\Big) {\log\Big(\log\Big(\frac{f(\bm{x}_{0}) - f^{*}}{\varepsilon}\Big)\Big)}\Bigg)
\end{equation*}
iterations of {Adam}.
\end{theorem}
The detailed version and proof of \Cref{thm-adanorm-plc} are in \Cref{sec-smooth-plc-pf}. 

\begin{remark}
    \label{rmk11-jan29}
    The iteration complexity of Adam scales as $\sqrt{\frac{L_\textup{max}}{\tilde{\mu}}} \min \Big(\sqrt{\frac{L_\textup{max}}{L_\textup{min}}}, \sqrt{{\kappa_{\textup{max},0}}}\Big)$. In contrast, the iteration complexity of GD scales as $\frac{L_\textup{max}}{\sqrt{\tilde{\mu} L_\textup{min}}} = \sqrt{\frac{L_\textup{max}}{\tilde{\mu}}} \times \sqrt{\frac{L_\textup{max}}{L_\textup{min}}}$ per \Cref{rmk8-jan20}. So when $\kappa_{\textup{max},0} < {\frac{L_\textup{max}}{L_\textup{min}}}$, Adam can outperform GD.
\end{remark}
As we discussed after \Cref{lower-bound}, $\kappa_{\textup{max},0}$ should implicitly grow with $d$. Based on this, we can provide the following characterization of when Adam is better than GD in simple words.
\begin{remark}
    \label{rmk10-jan24}
    In plain English, Adam can do better than GD when an initialization-dependent quantity which implicitly grows with the dimension is smaller than the ratio of the maximum and minimum per-coordinate smoothness constants.
\end{remark}

\bibliography{refs}

\newpage
\appendix
\onecolumn

\begin{center}
    \textbf{\Large Appendix}\vspace{5mm}
\end{center}

\section*{Contents}
\vspace{2mm}
\begin{itemize}
    \item \textbf{\Cref{connection-app}:} {Connection of the Update Rule in Eq.~\ref{eq:1-feb8} to the Vanilla Version of Adam}
    \item \textbf{\Cref{unif-hard}:} {Difficulty with Uniform Random Sampling for Non-Diagonal Quadratics}
    \item \textbf{\Cref{pf-1}:} {Detailed Version and Proof of Theorem~\ref{thm3}}
    \item \textbf{\Cref{pf-2}:} {Detailed Version and Proof of Theorem~\ref{thm3-diag}}
    \item \textbf{\Cref{lb-diag-pf}:} {Proof of Theorem~\ref{lb-diag}}
    \item \textbf{\Cref{prop:bar-kappa-pf}:} {Proof of Theorem~\ref{prop:bar-kappa}}
    \item \textbf{\Cref{bad-kappa-bar-pf}:} {Proof of Proposition~\ref{bad-kappa-bar}}
    \item \textbf{\Cref{pf-thm-fixed-pt}:} {Proof of Theorem~\ref{thm-fixed-pt}}
    \item \textbf{\Cref{sec-smooth-plc-pf}:} {Detailed Version and Proof of Theorem~\ref{thm-adanorm-plc}}
    \item \textbf{\Cref{cor1-jan20-pf}:} Proof of Corollary~\ref{cor1-jan20}
\end{itemize}

\clearpage

\section{Connection of the Update Rule in Eq.~\ref{eq:1-feb8} to the Vanilla Version of Adam}
\label{connection-app}
We will first state the vanilla version of Adam proposed by \cite{kingma2014adam}. For each $i \in \{1,\ldots,d\}$, let: 
\begin{equation}
    \overline{m}_{k}^{(i)} = \beta_1 \overline{m}_{k-1}^{(i)} + (1-\beta_1) {g}_k^{(i)} \text{ with } \overline{m}_{0}^{(i)} = 0 \text{ and } \overline{v}_{k}^{(i)} = \beta_2 \overline{v}_{k-1}^{(i)} + (1-\beta_2) \big({g}_k^{(i)}\big)^2 \text{ with } \overline{v}_{k}^{(i)} = 0,
\end{equation}
where recall that $\bm{g}_k = \nabla f(\bm{x}_k)$. The update rule of the $i^{\text{th}}$ coordinate in the $k^{\text{th}}$ iteration of the vanilla version of {Adam} proposed by \citep{kingma2014adam} is:
\begin{equation}
    \label{eq:1-dec28}
    x^{(i)}_{k+1} = x^{(i)}_{k} - \frac{\alpha \widetilde{m}_k^{(i)}}{\sqrt{\widetilde{v}_k^{(i)}} + {\delta}},
\end{equation}
where $\widetilde{m}_{k}^{(i)} = \frac{\overline{m}_{k}^{(i)}}{1 - \beta_1^{k+1}}$, $\widetilde{v}_{k}^{(i)} = \frac{\overline{v}_{k}^{(i)}}{1 - \beta_2^{k+1}}$, $\alpha$ is the step-size, $\beta_1, \beta_2 \in (0,1)$ are the first moment/momentum and second moment hyper-parameters, and ${\delta}$\footnote{In the literature, ${\delta}$ is more commonly denoted by $\epsilon$; we use ${\delta}$ to avoid confusion with $\varepsilon$ which denotes the error level that we wish to converge to.} is a correction term in the update's denominator. 

We will now show the relation between \cref{eq:1-dec28} and \cref{eq:1-feb8}. To that end, recall that in \cref{eq:1-feb8}, we had: 
\begin{equation}
    \label{eq:2-dec28-new}
    {m}_{k}^{(i)} = \beta_1 {m}_{k-1}^{(i)} + {g}_k^{(i)} \text{ with } {m}_{0}^{(i)} = 0 \text{ and } {v}_{k}^{(i)} = \beta_2 {v}_{k-1}^{(i)} + \big({g}_k^{(i)}\big)^2 \text{ with } {v}_{0}^{(i)} = 0.
\end{equation}
It can be verified that $\overline{m}_{k}^{(i)} = (1-\beta_1) {m}_{k}^{(i)}$ and $\overline{v}_{k}^{(i)} = (1-\beta_2) {v}_{k}^{(i)}$. Thus, $$\widetilde{m}_{k}^{(i)} = {{m}_{k}^{(i)}}\Big/{\Big(\frac{1 - \beta_1^{k+1}}{1 - \beta_1}\Big)} \text{ and } \widetilde{v}_{k}^{(i)} = {{v}_{k}^{(i)}}\Big/{\Big(\frac{1 - \beta_2^{k+1}}{1 - \beta_2}\Big)}.$$
Plugging this into \cref{eq:1-dec28}, we get:
\begin{equation}
    \label{eq:1-dec28-new}
    x^{(i)}_{k+1} = x^{(i)}_{k} - \Bigg(\alpha \Bigg(\frac{1-\beta_1}{\sqrt{1-\beta_2}}\Bigg) \Bigg(\frac{\sqrt{1 - \beta_2^{k+1}}}{1 - \beta_1^{k+1}}\Bigg)\Bigg) \frac{{m}_k^{(i)}}{\sqrt{{v}_k^{(i)}} + {\delta} \sqrt{\frac{{1 - \beta_2^{k+1}}}{1-\beta_2}}},
\end{equation}
We disregard the \enquote{bias-correction} terms $(1 - \beta_1^{k+1})$ and $(1 - \beta_2^{k+1})$ because they do not yield any theoretical improvement and converge to $1$. With that, \cref{eq:1-dec28-new} becomes:
\begin{equation}
    \label{eq:1-dec28-new2}
    x^{(i)}_{k+1} = x^{(i)}_{k} - \alpha \Bigg(\frac{1-\beta_1}{\sqrt{1-\beta_2}}\Bigg) \frac{{m}_k^{(i)}}{\sqrt{{v}_k^{(i)}} +  {\frac{{\delta}}{\sqrt{1-\beta_2}}}}.
\end{equation}
Pushing the $\frac{{\delta}}{\sqrt{1-\beta_2}}$ term inside the $\sqrt{{v}_k^{(i)}}$ term does not really change anything meaningfully; that gives us:
\begin{equation}
    \label{eq:1-dec28-new3}
    x^{(i)}_{k+1} = x^{(i)}_{k} - \alpha \Bigg(\frac{1-\beta_1}{\sqrt{1-\beta_2}}\Bigg) \frac{{m}_k^{(i)}}{\sqrt{{v}_k^{(i)} + \frac{{\delta}^2}{1-\beta_2}}}.
\end{equation}
Note that \cref{eq:1-feb8} is equivalent to \cref{eq:1-dec28-new3}; specifically, $\alpha$ and $\delta$ of \cref{eq:1-feb8} correspond to $\alpha \Big(\frac{1-\beta_1}{\sqrt{1-\beta_2}}\Big)$ and $\frac{{\delta}^2}{1-\beta_2}$ of \cref{eq:1-dec28-new3}, respectively.

\section{Difficulty with Uniform Random Sampling for Non-Diagonal Quadratics}
\label{unif-hard}
Here we will discuss the difficulty of deriving a result if $\bm{x}_0$ is sampled uniformly at random like the diagonal case (\Cref{thm3-diag}). We need a high-probability lower bound for $\min_{i \in [d]} |\widehat{g}_0^{(i)}|$, where recall that $\widehat{\bm{g}}_0 = \bm{T} (\bm{x}_0 - \bm{x}_{*})$ and $\bm{T} = \bm{D}^{-1} \bm{Q}$. Suppose $\|\bm{x}_{*}\|_\infty \leq B$ and ${x}_0^{(i)} \underset{\textup{iid}}{\sim} \textup{Unif}[-2B, 2B]$ $\forall$ $i \in [d]$. Let $\mathcal{P}(\bm{T}) := \{\bm{y} \text{ } | \text{ } \bm{y} = \bm{T} (\bm{x} - \bm{x}_{*}) \text{ s.t. } \|\bm{x}\|_\infty \leq 2B\}$. The density function for any $\bm{y} \in \mathcal{P}(\bm{T})$ is given by $p(\bm{y}) = \frac{1}{(4 B)^d\text{det}(\bm{T})}$. So:
\begin{equation}
    \mathbb{P}\Big(\min_{i \in [d]} |\widehat{g}_0^{(i)}| > w\Big) = \frac{1}{(4 B)^d\text{det}(\bm{T})} \int_{\bm{y} \in \mathcal{P}(\bm{T})} \Pi_{i \in [d]}\mathds{1}\Big(|y^{(i)}| \geq w\Big) \bm{d y}.
\end{equation}
The above volumetric integral seems hard to evaluate in general. This is the main difficulty in deriving a result if $\bm{x}_0$ is sampled uniformly at random. However, we can indeed get a result if $\bm{Q}$ is close to diagonal and we discuss this next.

\subsection{Approximately Diagonal Case}
Note that $\widehat{\bm{g}}_0 = (\bm{x}_0 - \bm{x}_{*}) + \big(\bm{T} - \bm{\text{{I}}}\big) (\bm{x}_0 - \bm{x}_{*})$. Let $\bm{z}_0 = \bm{x}_0 - \bm{x}_{*}$ and  $\bm{T}' = \bm{T} - \bm{\text{{I}}}$. Since $\bm{T} = \bm{D}^{-1} \bm{Q}$, we have that $T'[i,j] = \frac{Q[i,j]}{q_i}$ for $j \neq i$ and $T'[i,i] = 0$. Let $\nu = \max_{i \in [d]} \sum_{j \neq i} T'[i,j]$. When $\nu$ is small enough\footnote{Also see \Cref{diag-dom}; this is the same as saying $\bm{Q}$ is $\nu$-diagonally-dominant when $\nu \leq 1$.}:
\begin{equation}
    |\widehat{{g}}_0^{(i)}| \geq |z_0^{(i)}| - \nu \|\bm{z}_0\|_\infty \text{ } \forall \text{ } i \in [d].
\end{equation}
Since $\|\bm{x}_0\|_\infty \leq 2B$ and $\|\bm{x}_{*}\|_\infty \leq B$, we have $\|\bm{z}_0\|_\infty \leq 3 B$. Using this above, we get:
\begin{equation}
    |\widehat{{g}}_0^{(i)}| \geq |z_0^{(i)}| - \nu (3 B) \text{ } \forall \text{ } i \in [d].
\end{equation}
From \Cref{lem5-jan3}, we have that $\min_{i \in [d]} |{z}_0^{(i)}| > \Big(\frac{2 \log ({1}/{p})}{1 + \frac{\log ({1}/{p})}{d}}\Big)\frac{B}{d}$ with a probability of at least $p$. So when $\nu \leq \frac{1}{3d}$:
$$\min_{i \in [d]}|\widehat{{g}}_0^{(i)}| > \Bigg(\frac{2 \log ({1}/{p})}{1 + \frac{\log ({1}/{p})}{d}} - 1\Bigg)\frac{B}{d},$$
with a probability of at least $p$. We can use the bound above to obtain a result similar to \Cref{thm3} (with $\theta = 1$). 

\section{Detailed Version and Proof of Theorem~\ref{thm3}}
\label{pf-1}
We need to introduce a couple of quantities first. Recall that $\bar{\bm{Q}} = \bm{D}^{-1/2} \bm{U} \bm{\Lambda}^{1/2}$. Let:
\begin{equation}
    \sigma_\text{min}(\bar{\bm{Q}}) = \sqrt{\mu_1} \text{ and } \sigma_\text{max}(\bar{\bm{Q}}) = \sqrt{\mu_2}.
\end{equation}
Since $\text{cond}(\bar{\bm{Q}}) = \sqrt{\overline{\kappa}}$, we have $\overline{\kappa} = \big(\frac{\mu_2}{\mu_1}\big)$. We are now ready to state the detailed version of \Cref{thm3}.

\begin{theorem}[\textbf{Detailed Version of \Cref{thm3}}]
    \label{thm3-det}
    Suppose our initialization $\bm{x}_0$ satisfies \Cref{dist}. Fix some $p \in [\frac{1}{e}, 1)$. Choose 
    \begin{equation*}
    \alpha = \mathcal{O}\Bigg(\frac{\widehat{\kappa} \sqrt{d} B}{\sqrt{\mu_1 \mu_2}} \Big({\max\Big(\frac{\widetilde{p}^2}{\widehat{\kappa}^2 d^3}, \frac{1}{\kappa_\textup{diag}^2}\Big)} \Big)^{1/4} \log^{-1/2}\Big(\frac{\sqrt{d} B}{\rho_1 \varepsilon}\Big)\Bigg) \text{ where } \widetilde{p} = \Bigg(\frac{\log ({1}/{p})}{\zeta\big(1 + \frac{\log (1/p)}{d}\big)}\Bigg)^{1/\theta}, 
    \end{equation*}
    \begin{equation*}
    \beta_2 = 1 - \frac{\alpha^2 \mu_1^2}{72 \widehat{\kappa}^2 d B^2}, \text{ } \delta_k = \frac{\beta_2^{{k}/{2}}}{\sqrt{2(1-\beta_2)}} \text{ and }  \phi = \widehat{\kappa} \sqrt{d} B q_\textup{min}.
    \end{equation*}
    Then with a probability of at least $p$ over the randomness of $\bm{x}_0$,  $f(\bm{x}_{{K}^{*}}) < \frac{\varepsilon^2}{2}$ in 
    \begin{equation*}
        K^{*} = \mathcal{O}\Bigg(\kappa_\textup{Adam} \log\Big(\frac{\sqrt{d} B}{\rho_1 \varepsilon}\Big){\log\Big(\log\Big(\frac{\|\bm{\Lambda}^{1/2} {\bm{U}}^\top (\bm{x}_0 - \bm{x}_{*})\|_2}{\varepsilon}\Big)\Big)}\Bigg)
    \end{equation*}
    iterations of {Adam}, where $\kappa_\textup{Adam} := \overline{\kappa} \min \Big(\Big(\frac{\widehat{\kappa} d^{3/2}}{\log^{1/\theta} ({1}/{p})}\Big) \zeta^{1/\theta} \big(1 + \frac{\log (1/p)}{d}\big)^{1/\theta}, \kappa_\textup{diag}\Big)$. 
\end{theorem}
For setting the hyper-parameters listed above, we need to be privy to $\mu_1$, $\mu_2$, $\rho_1$, $\rho_2$ (recall $\widehat{\kappa} = \frac{\rho_2}{\rho_1}$), $q_\textup{max}$ and $q_\textup{min}$ (recall $\kappa_\textup{diag} = \frac{q_\textup{max}}{q_\textup{min}}$). It is true that we require more knowledge about the problem compared to GD, but we would like to emphasize that \textit{the main point of the results in this paper is to quantify how well Adam can do with a bit more information about the problem}. In this sense, our results are a significant departure from most prior results that do not quantify the gains of Adam over GD in terms of dependence on the condition number of the Hessian. Moreover, the dependence of GD cannot be improved even if we are privy to these quantities.
\vspace{0.2 cm}

\begin{proof}
Adam's update rule in \cref{eq:3-dec28} with $\beta_1 = 0$ becomes:
\begin{equation}
    \label{eq:58-oct15}
    x^{(i)}_{k+1} = x^{(i)}_{k} - \frac{\alpha {g}_k^{(i)}}{\sqrt{{v}_{k}^{(i)} + \max\Big(\big({g}_0^{(i)}\big)^2, \phi^2\Big)\delta_{k}^{2}}},
\end{equation} 
where ${v}_{k}^{(i)} := {\sum_{l=0}^{k} \beta_2^{k-l} \big({g}_l^{(i)}\big)^2}$ and ${g}_k^{(i)}$ is the $i^\text{th}$ coordinate of $\bm{g}_k = \nabla f(\bm{x}_k) = \bm{Q} (\bm{x}_k - \bm{x}_{*})$. Let ${\bm{V}}_k$ be a diagonal matrix whose $i^{\text{th}}$ diagonal element is $\frac{1}{\sqrt{{v}_{k}^{(i)} + \max(({g}_0^{(i)})^2, \phi^2) \delta_k^{2}}}$. Then, we can rewrite \cref{eq:58-oct15} in matrix-vector notation as:
\begin{equation}
    \label{eq:10-nov13}
    \bm{x}_{k+1} = \bm{x}_k - \alpha {\bm{V}}_k {\bm{g}}_k = \bm{x}_k - \alpha {\bm{V}}_k {\bm{Q}} (\bm{x}_k - \bm{x}_{*}) = \bm{x}_k - \alpha {\bm{V}}_k {\bm{U}} \bm{\Lambda} {\bm{U}}^\top (\bm{x}_k - \bm{x}_{*}).
\end{equation}
Subtracting $\bm{x}_{*}$ from both sides above and then pre-multiplying by $\bm{\Lambda}^{1/2} \bm{U}^\top$ on both sides, we get:
\begin{equation}
    \bm{\Lambda}^{1/2} \bm{U}^\top (\bm{x}_{k+1} - \bm{x}_{*}) = \bm{\Lambda}^{1/2} \bm{U}^\top (\bm{x}_k - \bm{x}_{*}) - \alpha \big(\bm{\Lambda}^{1/2} \bm{U}^\top {\bm{V}}_k {\bm{U}} \bm{\Lambda}^{1/2}\big) \bm{\Lambda}^{1/2} {\bm{U}}^\top (\bm{x}_k - \bm{x}_{*}).
\end{equation}
Let us define $\bar{\bm{g}}_k := \bm{\Lambda}^{1/2} {\bm{U}}^\top (\bm{x}_k - \bm{x}_{*})$; note that $\bm{g}_k = \bm{U} \bm{\Lambda}^{1/2} \bar{\bm{g}}_k$. Then we can re-write the above equation as:
\begin{equation}
    \label{eq:11-nov14}
    \bar{\bm{g}}_{k+1} = \Big(\bm{\text{I}} - \alpha \bm{\Lambda}^{1/2} \bm{U}^\top {\bm{V}}_k {\bm{U}} \bm{\Lambda}^{1/2} \Big)\bar{\bm{g}}_k.
\end{equation}
Let us also define $\widehat{\bm{g}}_k := \bm{D}^{-1} \bm{g}_k = (\bm{D}^{-1} \bm{U} \bm{\Lambda}^{1/2}) \bar{\bm{g}}_k = \widehat{\bm{Q}} \bar{\bm{g}}_k$. Note that the $i^\text{th}$ coordinate of $\widehat{\bm{g}}_k$ denoted by $\widehat{{g}}_k^{(i)}$ is simply $\frac{{g}_k^{(i)}}{{q_i}}$, where recall that $q_i = Q[i,i]$. Let 
\begin{equation}
    \label{eq:7-nov26}
    \widehat{v}_{k}^{(i)} := {\sum_{l=0}^{k} \beta_2^{k-l} \big(\widehat{g}_l^{(i)}\big)^2} = \frac{{v}_{k}^{(i)}}{q_i^2}.
\end{equation}
Also define $\widehat{\bm{V}}_k$ to be a diagonal matrix whose $i^{\text{th}}$ diagonal element is $\frac{1}{\sqrt{\widehat{v}_{k}^{(i)} + \max((\widehat{g}_0^{(i)})^2, {\phi_i^2})\delta_k^{2}}}$, where $\phi_i := \frac{\phi}{q_i}$. Note that ${\bm{V}}_k = \bm{D}^{-1/2} \widehat{\bm{V}}_k \bm{D}^{-1/2}$. Plugging this into \cref{eq:11-nov14}, we get:
\begin{equation}
    \label{eq:12-nov14}
    \bar{\bm{g}}_{k+1} = \Big(\bm{\text{I}} - \alpha \bm{\Lambda}^{1/2} \bm{U}^\top \bm{D}^{-1/2} \widehat{\bm{V}}_k \bm{D}^{-1/2} {\bm{U}} \bm{\Lambda}^{1/2} \Big)\bar{\bm{g}}_k.
\end{equation}
{Recalling that $\bar{\bm{Q}} = \bm{D}^{-1/2} \bm{U} \bm{\Lambda}^{1/2}$}, we can rewrite the above equation as:
\begin{equation}
    \label{eq:12-nov15}
    \bar{\bm{g}}_{k+1} = \underbrace{\Big(\bm{\text{I}} - \alpha \bar{\bm{Q}}^\top \widehat{\bm{V}}_k \bar{\bm{Q}} \Big)}_{:= \bm{A}_k}\bar{\bm{g}}_k.
\end{equation}
Since $\widehat{\bm{V}}_k$ is a diagonal matrix with positive diagonal elements, $\bar{\bm{Q}}^\top \widehat{\bm{V}}_k \bar{\bm{Q}}$ is symmetric positive-definite (PD). Recall that $\sigma_\text{min}(\bar{\bm{Q}}) = \sigma_\text{min}(\bar{\bm{Q}}^\top) = \sqrt{\mu_1}$ and $\sigma_\text{max}(\bar{\bm{Q}}) = \sigma_\text{max}(\bar{\bm{Q}}^\top) = \sqrt{\mu_2}$.
Thus, we have that: 
\small
\begin{equation*}
    \sigma_{\text{max}}(\bar{\bm{Q}}^\top \widehat{\bm{V}}_k \bar{\bm{Q}}) \leq \sigma_{\text{max}}(\bar{\bm{Q}}^\top) \sigma_{\text{max}}(\widehat{\bm{V}}_k) \sigma_{\text{max}}(\bar{\bm{Q}}) \leq \frac{\mu_2}{\sqrt{\min_{j \in [d]} \big(\widehat{v}_{k}^{(j)} + \max((\widehat{g}_0^{(j)})^2, \phi_j^2) \delta_k^{2}\big)}}
\end{equation*}
\normalsize
and
\small
\begin{equation*}
    \sigma_{\text{min}}(\bar{\bm{Q}}^\top \widehat{\bm{V}}_k \bar{\bm{Q}}) \geq \sigma_{\text{min}}(\bar{\bm{Q}}^\top) \sigma_{\text{min}}(\widehat{\bm{V}}_k) \sigma_{\text{min}}(\bar{\bm{Q}})\geq \frac{\mu_1}{\sqrt{\max_{j \in [d]} \big(\widehat{v}_{k}^{(j)} + \max((\widehat{g}_0^{(j)})^2, \phi_j^2) \delta_k^{2}\big)}}.
\end{equation*}
\normalsize
Now at least until ${\min_{j \in [d]} \big(\widehat{v}_{k}^{(j)} + \max\big((\widehat{g}_0^{(j)})^2, \phi_j^2\big) \delta_k^{2}\big)} \geq (\mu_2 \alpha)^2$ starting from the first iteration, $\bm{A}_k = \bm{\text{I}} - \alpha \bar{\bm{Q}}^\top \widehat{\bm{V}}_k \bar{\bm{Q}}$ is positive semi-definite and we have:
\begin{equation}
    \label{eq:68-nov8}
    \sigma_{\text{max}}(\bm{A}_k) \leq 1 - \frac{\mu_1 \alpha}{\sqrt{\max_{j \in [d]} \big(\widehat{v}_{k}^{(j)} + \max\big((\widehat{g}_0^{(j)})^2, \phi_j^2\big) \delta_k^{2}\big)}}.
\end{equation}
So at least until ${\min_{j \in [d]} \big(\widehat{v}_{k}^{(j)} + \max\big((\widehat{g}_0^{(j)})^2, \phi_j^2\big) \delta_k^{2}\big)} \geq (\mu_2 \alpha)^2$ holds, \textit{descent is ensured in each iteration or there is \enquote{continuous descent}} and we have:
\begin{flalign}
    \nonumber
    \|\bar{\bm{g}}_{k+1}\|_2 & \leq \sigma_{\text{max}}(\bm{A}_k) \|\bar{\bm{g}}_k\|_2 
    \\ 
    \label{eq:78-sep30}
    & \leq 
    \Bigg(1 - \frac{\mu_1 \alpha}{\sqrt{\max_{j \in [d]} \big(\widehat{v}_{k}^{(j)} + \max\big((\widehat{g}_0^{(j)})^2, \phi_j^2\big) \delta_k^{2}\big)}}\Bigg) \|\bar{\bm{g}}_k\|_2.
\end{flalign}
Suppose we have continuous descent until $\tilde{K}$ iterations, i.e., $${\min_{j \in [d]} \big(\widehat{v}_{k}^{(j)} + \max\big((\widehat{g}_0^{(j)})^2, \phi_j^2\big) \delta_k^{2}\big)} \geq (\mu_2 \alpha)^2 \text{ } \forall \text{ } k \in \{0,\ldots,\tilde{K}\}.$$ 
Choose $\alpha \leq \frac{\max(3 \widehat{\kappa} \sqrt{d} B, \phi/q_\textup{min})}{\sqrt{2} \mu_1}$, $\beta_2 = 1 - \frac{\alpha^2 \mu_1^2}{8 \max(9 \widehat{\kappa}^2 d B^2, \phi^2/q_\textup{min}^2)}$ and $\delta_k = \frac{1}{\sqrt{2(1-\beta_2)}} \beta_2^{\frac{k}{2}}$. From \Cref{lem1-oct1}, we have continuous descent until at least
\begin{flalign}
    \nonumber
    \tilde{K} & = \frac{\log\Big(\frac{\min_{j \in [d]} (\widehat{g}_0^{(j)})^2}{(\mu_2 \alpha)^2} + \frac{\max(\min_{j \in [d]} (\widehat{g}_0^{(j)})^2, {\phi^2}/{q_\textup{max}^2})}{{2(1-\beta_2)(\mu_2 \alpha)^2}}\Big)}{\log \frac{1}{\beta_2}} 
    \\
    \label{eq:31-jan16}
    & = \frac{\log\Big(\frac{\min_{j \in [d]} (\widehat{g}_0^{(j)})^2}{\mu_2^2 \alpha^2} + \frac{4 \max(\min_{j \in [d]} (\widehat{g}_0^{(j)})^2, {\phi^2}/{q_\textup{max}^2}) \max(9 \widehat{\kappa}^2 d B^2, \phi^2/q_\textup{min}^2)}{\mu_2^2 \mu_1^2 \alpha^4}\Big)}{\log \frac{1}{\beta_2}}
\end{flalign}
iterations. The last equality in \cref{eq:31-jan16} follows by plugging in the value of $\beta_2 = 1 - \frac{\alpha^2 \mu_1^2}{8 \max(9 \widehat{\kappa}^2 d B^2, \phi^2/q_\textup{min}^2)}$. Also from \Cref{lem1-sep30}, we have for any $k \in \{1,\ldots,\tilde{K}\}$:
\begin{equation}
    \|\bar{\bm{g}}_{k}\|_2 \leq \exp\Big(-2\big(\beta_2^{-{k}/{2}} - 1\big)\Big)\|\bar{\bm{g}}_0\|_2.
\end{equation}
Now in order to have $\|\bar{\bm{g}}_{K^{*}}\|_2 \leq \varepsilon$, we can choose: 
\begin{equation}
    \label{eq:33-jan16}
    K^{*} = \frac{\log\Big(\frac{1}{4}\log^2\Big(\frac{e^{2}\|\bar{\bm{g}}_0\|_2}{\varepsilon}\Big)\Big)}{\log \frac{1}{\beta_2}}.
\end{equation}
Now in order for this result to be applicable, let us ensure $K^{*} \leq \tilde{K}$; for this, we must have:
\begin{equation}
    \frac{1}{4}\log^2\Big(\frac{e^{2}\|\bar{\bm{g}}_0\|_2}{\varepsilon}\Big) \leq \frac{\min_{j \in [d]} (\widehat{g}_0^{(j)})^2}{\mu_2^2 \alpha^2} + \frac{4 \max(\min_{j \in [d]} (\widehat{g}_0^{(j)})^2, {\phi^2}/q_\textup{max}^2) \max(9 \widehat{\kappa}^2 d B^2, \phi^2/q_\textup{min}^2)}{\mu_2^2 \mu_1^2 \alpha^4}.
\end{equation}
The above can be ensured by simply having:
\begin{equation}
    \frac{1}{4}\log^2\Big(\frac{e^{2}\|\bar{\bm{g}}_0\|_2}{\varepsilon}\Big) \leq \frac{4 \max(\min_{j \in [d]} (\widehat{g}_0^{(j)})^2, {\phi^2}/q_\textup{max}^2) \max(9 \widehat{\kappa}^2 d B^2, \phi^2/q_\textup{min}^2)}{\mu_2^2 \mu_1^2 \alpha^4},
\end{equation}
and this gives us:
\begin{equation}
    \label{eq:15-dec10-old}
    \alpha \leq \mathcal{O}\Bigg(\Bigg(\frac{\mu_1^2 \max(\min_{j \in [d]} (\widehat{g}_0^{(j)})^2, {\phi^2}/q_\textup{max}^2)}{\mu_2^2 \max( \widehat{\kappa}^2 d B^2, \phi^2/q_\textup{min}^2)} \Bigg)^{1/4}  \frac{\max\big(\widehat{\kappa} \sqrt{d} B, \phi/q_\textup{min}\big)}{\mu_1}\log^{-1/2}\Big(\frac{\|\bar{\bm{g}}_0\|_2}{\varepsilon}\Big)\Bigg).
\end{equation}
Earlier, we had imposed $\alpha \leq \frac{\max(3 \widehat{\kappa} \sqrt{d} B, \phi/q_\textup{min})}{\sqrt{2} \mu_1}$; note that the above constraint satisfies this requirement. This is because $\mu_1 \leq \mu_2$ and $\max(\min_{j \in [d]} (\widehat{g}_0^{(j)})^2, {\phi^2}/q_\textup{max}^2) \leq \max(\widehat{\kappa}^2 d \|\widehat{\bm{g}}_0\|_\infty^2, \phi^2/q_\textup{min}^2) \leq \mathcal{O}\big(\max(\widehat{\kappa}^2 d B^2, \phi^2/q_\textup{min}^2)\big)$, where the last step follows from the fact that $\|\widehat{\bm{g}}_0\|_\infty \leq 3 B$ (from \Cref{dist}). From \Cref{lem3-jan7}, we have $\min_{j \in [d]} |\widehat{g}_0^{(j)}|  > \frac{2B}{d} \Big(\frac{\log ({1}/{p})}{\zeta\big(1 + \frac{\log (1/p)}{d}\big)}\Big)^{1/\theta}$ with a probability of at least $p$. For simplicity of notation, let $\widetilde{p} = \Big(\frac{\log ({1}/{p})}{\zeta\big(1 + \frac{\log (1/p)}{d}\big)}\Big)^{1/\theta}$.
Moreover, since  $\widehat{\bm{g}}_0 = \widehat{\bm{Q}} \bar{\bm{g}}_0$ and $\sigma_\text{min}(\widehat{\bm{Q}}) = {\rho_1}$, $\|\bar{\bm{g}}_0\|_2 \leq \frac{\|\widehat{\bm{g}}_0\|_2}{\rho_1} \leq \frac{\sqrt{d} \|\widehat{\bm{g}}_0\|_\infty}{\rho_1} \leq \mathcal{O}(\frac{\sqrt{d} B}{\rho_1})$ (where the last step follows because  $\|\widehat{\bm{g}}_0\|_\infty \leq 3 B$). Thus, with a probability of at least $p$, we can pick:
\begin{equation}
    \label{eq:15-dec10}
    \alpha = \mathcal{O}\Bigg(\frac{1}{\sqrt{\mu_1 \mu_2}} \Bigg({\max\Bigg(\frac{\widetilde{p}^2 B^2}{d^2}, \frac{\phi^2}{q_\textup{max}^2}\Bigg) \max\Bigg(\widehat{\kappa}^2 d B^2, \frac{\phi^2}{q_\textup{min}^2}\Bigg)} \Bigg)^{1/4} \log^{-1/2}\Big(\frac{\sqrt{d} B}{\rho_1 \varepsilon}\Big)\Bigg).
\end{equation}
Next, recalling that $\beta_2 = 1 - \frac{\alpha^2 \mu_1^2}{8 \max(9 \widehat{\kappa}^2 d B^2, \phi^2/q_\textup{min}^2)}$, $\log \frac{1}{\beta_2} \geq \frac{\alpha^2 \mu_1^2}{8 \max(9 \widehat{\kappa}^2 d B^2, \phi^2/q_\textup{min}^2)}$ using \Cref{log-fact}. Plugging this in the expression of $K^{*}$ (\cref{eq:33-jan16}), we get:
\begin{equation}
    K^{*} = \mathcal{O}\Bigg(\frac{\max(\widehat{\kappa}^2 d B^2, \phi^2/q_\textup{min}^2)}{{\alpha}^2 \mu_1^2 } {\log\Big(\log\Big(\frac{\|\bar{\bm{g}}_0\|_2}{\varepsilon}\Big)\Big)}\Bigg).
\end{equation}
Plugging in the value of $\alpha$ from \cref{eq:15-dec10} while recalling that $\overline{\kappa} = \frac{\mu_2}{\mu_1}$, we get:
\begin{equation}
    \label{eq:17-dec10-old}
    K^{*} = \mathcal{O}\Bigg(\overline{\kappa}  \Bigg(\frac{\max\big(\widehat{\kappa} \sqrt{d} B, \frac{\phi}{q_\textup{min}}\big)}{\max\big(\frac{\widetilde{p} B}{d}, \frac{\phi}{q_\textup{max}}\big)}\Bigg) \log\Big(\frac{\sqrt{d} B}{\rho_1 \varepsilon}\Big) {\log\Big(\log\Big(\frac{\|\bar{\bm{g}}_0\|_2}{\varepsilon}\Big)\Big)}\Bigg),
\end{equation}
with a probability of at least $p$ over random initialization. 
Next, recall that $\bar{\bm{g}}_{{K}^{*}} := \bm{\Lambda}^{1/2} {\bm{U}}^\top (\bm{x}_{{K}^{*}} - \bm{x}_{*})$. Also, we have that: $$f(\bm{x}_{{K}^{*}}) = \frac{1}{2}(\bm{x}_{{K}^{*}} - \bm{x}_{*})^\top \bm{Q} (\bm{x}_{{K}^{*}} - \bm{x}_{*}) = \frac{1}{2}\|\bar{\bm{g}}_{{K}^{*}}\|_2^2.$$
So $\|\bar{\bm{g}}_{K^{*}}\|_2 \leq \varepsilon$ implies $f(\bm{x}_{{K}^{*}}) \leq \frac{\varepsilon^2}{2}$. So to summarize, we have $f(\bm{x}_{{K}^{*}}) \leq \frac{\varepsilon^2}{2}$ in ${K}^{*}$ iterations with a probability of at least $p$ over random initialization, where ${K}^{*}$ is given by \cref{eq:17-dec10-old}. Finally, the theorem statement follows by choosing $\phi = \widehat{\kappa} \sqrt{d} B q_\textup{min}$ in the analysis above and recalling that $\kappa_\textup{diag} = \frac{q_\textup{max}}{q_\textup{min}}$.
\end{proof}

\begin{lemma}
\label{lem1-sep30}
Suppose \cref{eq:78-sep30} holds, i.e., $$\|\bar{\bm{g}}_{k+1}\|_2 \leq \Bigg(1 - \frac{\mu_1 \alpha}{\sqrt{\max_{j \in [d]} \big(\widehat{v}_{k}^{(j)} + \max\big((\widehat{g}_0^{(j)})^2, \phi_j^2\big) \delta_k^{2}\big)}}\Bigg) \|\bar{\bm{g}}_k\|_2,$$
for $k = 0$ through to $k = \tilde{K}$, where ${\min_{j \in [d]} \big(\widehat{v}_{k}^{(j)} + \max\big((\widehat{g}_0^{(j)})^2, \phi_j^2\big) \delta_k^{2}\big)} \geq (\mu_2 \alpha)^2$ $\forall$ $k \in \{0,\ldots,\tilde{K}\}$. 
Suppose we set $\alpha \leq \frac{\max(3 \widehat{\kappa} \sqrt{d} B, \phi/q_\textup{min})}{\sqrt{2} \mu_1}$, $\beta_2 = 1 - \frac{\alpha^2 \mu_1^2}{8 \max(9 \widehat{\kappa}^2 d B^2, \phi^2/q_\textup{min}^2)}$ and $\delta_k = \frac{1}{\sqrt{2(1-\beta_2)}} \beta_2^{\frac{k}{2}}$ for $k \geq 0$. Then for any $k \leq \tilde{K}$, we have:
\begin{equation*}
    \|\bar{\bm{g}}_{k+1}\|_2 \leq \exp\Bigg(-2\Big(\beta_2^{-(k+1)/2} - 1\Big)\Bigg)\|\bar{\bm{g}}_0\|_2.
\end{equation*}
\end{lemma}
\begin{proof}
    We employ a two-stage analysis procedure where we first obtain a loose bound for $\|\bar{\bm{g}}_{k+1}\|_2$ in terms of $\|\bar{\bm{g}}_0\|_2$, and then obtain a tighter bound by using this loose bound.
    \\
    \\
    \textbf{First Stage.}
    Let us define $$\widehat{v}_{k} := \sum_{j \in [d]} \widehat{v}_{k}^{(j)} = {\sum_{l=0}^{k} \beta_2^{k-l} \|\widehat{\bm{g}}_l\|_2^2}.$$
    Note that: 
    \begin{flalign}
        \label{eq:20-nov23}
        \max_{j \in [d]} \big(\widehat{v}_{k}^{(j)} + \max\big((\widehat{g}_0^{(j)})^2, \phi_j^2\big) \delta_k^{2}\big) & \leq \sum_{j \in [d]} \widehat{v}_{k}^{(j)} + \max_{j \in [d]} \max\big((\widehat{g}_0^{(j)})^2, \phi_j^2\big) \delta_k^{2} 
        \\
        & \leq \widehat{v}_{k} + \max\Big(\|\widehat{\bm{g}}_0\|_\infty^2, \frac{\phi^2}{q_\text{min}^2}\Big) \delta_k^{2},
    \end{flalign}
    where recall that $q_\text{min} = \min_{i \in [d]} q_i$. Plugging this into \cref{eq:78-sep30}, we get:
    \begin{flalign}
        \label{eq:21-nov23}
        \|\bar{\bm{g}}_{k+1}\|_2 \leq \Bigg(1 - \frac{\mu_1 \alpha}{\sqrt{\widehat{v}_{k} + \max\big(\|\widehat{\bm{g}}_0\|_\infty^2, \phi^2/q_\text{min}^2) {\delta}_k^2}}\Bigg)\|\bar{\bm{g}}_k\|_2.
    \end{flalign}
    We shall be using \cref{eq:21-nov23} henceforth.
    \\
    \\
    From \cref{eq:21-nov23}, we have $\|\bar{\bm{g}}_{0}\|_2 \geq \|\bar{\bm{g}}_{1}\|_2 \geq \ldots \geq \|\bar{\bm{g}}_{\tilde{K}+1}\|_2$. Recall that $\widehat{\bm{g}}_k = \widehat{\bm{Q}} \bar{\bm{g}}_k$, $\sigma_\text{max}(\widehat{\bm{Q}}) = {\rho_2}$, $\sigma_\text{min}(\widehat{\bm{Q}}) = {\rho_1}$ and $\text{cond}(\widehat{\bm{Q}}) = \frac{\rho_2}{\rho_1} = \widehat{\kappa}$. Then, $\|\widehat{\bm{g}}_k\|_2 \leq {\rho_2} \|\bar{\bm{g}}_k\|_2 \leq {\rho_2} \|\bar{\bm{g}}_{0}\|_2 \leq {\rho_2} \big(\frac{\|\widehat{\bm{g}}_{0}\|_2}{\rho_1}\big) = \widehat{\kappa} \|\widehat{\bm{g}}_{0}\|_2$ for all $k \in \{0,\ldots,\tilde{K}+1\}$. So for any $k \leq \tilde{K}$:
    \begin{equation}
    \widehat{v}_k = {\sum_{l=0}^k \beta_2^{k-l} \|\widehat{\bm{g}}_l\|_2^2} \leq {\sum_{l=0}^k \beta_2^{k-l} \widehat{\kappa}^2 \|\widehat{\bm{g}}_0\|_2^2} \leq \frac{\widehat{\kappa}^2 \|\widehat{\bm{g}}_0\|_2^2}{1-\beta_2}.
    \end{equation}
    Also, $\|\widehat{\bm{g}}_0\|_\infty^2 \leq \|\widehat{\bm{g}}_0\|_2^2 \leq \widehat{\kappa}^2 \|\widehat{\bm{g}}_0\|_2^2$. Further, we have chosen the ${\delta}_k$'s so that ${\delta}_k \leq {\delta}_0 \leq \frac{1}{\sqrt{1-\beta_2}}$ for all $k \geq 1$. Then:
    \begin{equation}
        \widehat{v}_k + \max\Big(\|\widehat{\bm{g}}_0\|_\infty^2, \frac{\phi^2}{q_\text{min}^2}\Big) {\delta}_k^2 \leq \frac{2 \max \Big(\widehat{\kappa}^2 \|\widehat{\bm{g}}_0\|_2^2, \frac{\phi^2}{q_\text{min}^2}\Big)}{1-\beta_2} \leq \frac{2 \max \Big(9 \widehat{\kappa}^2 d B^2, \frac{\phi^2}{q_\text{min}^2}\Big)}{1-\beta_2},
    \end{equation}
    {where the last step follows from the fact that $\|\widehat{\bm{g}}_0\|_\infty = \|\bm{D}^{-1} \bm{Q} \bm{x}_0 - \bm{D}^{-1} \bm{Q} \bm{x}_{*}\|_\infty \leq 3 B$ (from \Cref{dist}) and thus, $\|\widehat{\bm{g}}_0\|_2 \leq 3 \sqrt{d} B$}. Plugging this in \cref{eq:21-nov23} gives us:
    \begin{equation}
        \label{eq:43-jan16}
        \|\bar{\bm{g}}_{k+1}\|_2^2 \leq \Bigg(1 - \frac{{\alpha} \mu_1 \sqrt{1-\beta_2}}{\sqrt{2} \max\big(3 \widehat{\kappa} \sqrt{d} B, \frac{\phi}{q_\text{min}}\big)}\Bigg)^2\|\bar{\bm{g}}_{k}\|_2^2 \leq \Bigg(1 - \frac{{\alpha} \mu_1 \sqrt{1-\beta_2}}{\sqrt{2} \max\big(3 \widehat{\kappa} \sqrt{d} B, \frac{\phi}{q_\text{min}}\big)}\Bigg) \|\bar{\bm{g}}_{k}\|_2^2.
    \end{equation}
    Define $\gamma := \Big(1 - \frac{{\alpha} \mu_1 \sqrt{1-\beta_2}}{\sqrt{2} \max(3 \widehat{\kappa} \sqrt{d} B, \phi/q_\text{min})}\Big)$. 
    Unfolding the recursion in \cref{eq:43-jan16} yields:
    \begin{equation}
        \label{eq:86-nov11}
        \|\bar{\bm{g}}_{k+1}\|_2^2 \leq \gamma^{k+1}\|\bar{\bm{g}}_{0}\|_2^2.
    \end{equation}
    This bound is loose but we shall obtain a tighter bound using this bound. 
    \\
    \\
    \noindent \textbf{Second Stage.} With our choice of $\beta_2 = 1 - \frac{\alpha^2 \mu_1^2}{8 \max(9 \widehat{\kappa}^2 d B^2, \phi^2/q_\text{min}^2)}$, we have $\gamma = 1 - \frac{\alpha^2 \mu_1^2}{4 \max(9 \widehat{\kappa}^2 d B^2, \phi^2/q_\text{min}^2)}$ and $\beta_2 - \gamma = {1 - \beta_2}$. Now for any $k \leq \tilde{K}$, we have:
    \begin{flalign}
    \widehat{v}_k = \sum_{l=0}^{k} \beta_2^{k-l} \|\widehat{\bm{g}}_{l}\|_2^2 \leq \sum_{l=0}^{k} \beta_2^{k-l} \rho_2^2 \|\bar{\bm{g}}_{l}\|_2^2 & \leq \sum_{l=0}^{k} \beta_2^{k-l} \rho_2^2  \gamma^l \|\bar{\bm{g}}_{0}\|_2^2 
    \\
    & = \Bigg(\frac{\beta_2^{k+1} - \gamma^{k+1}}{\beta_2 - \gamma}\Bigg) \rho_2^2 \|\bar{\bm{g}}_{0}\|_2^2
    \\
    \label{eq:45-dec29}
    & \leq \Big(\frac{\beta_2^{k+1}}{1-\beta_2}\Big) \widehat{\kappa}^2 \|\widehat{\bm{g}}_{0}\|_2^2,
    \end{flalign}
    where the last step follows from the fact that $\beta_2 - \gamma = {1 - \beta_2}$ and ${\rho_2} \|\bar{\bm{g}}_{0}\|_2 \leq {\rho_2} \big(\frac{\|\widehat{\bm{g}}_{0}\|_2}{\rho_1}\big) = \widehat{\kappa} \|\widehat{\bm{g}}_{0}\|_2$. 
    With our constraint of $\alpha \leq \frac{\max(3 \widehat{\kappa} \sqrt{d} B, \phi/q_\text{min})}{\sqrt{2} \mu_1}$, we have:
    \begin{equation}
        \label{eq:44-jan16-new}
        \beta_2 \geq \frac{1}{2}.
    \end{equation}
    Thus, we have $\delta_k = \frac{1}{\sqrt{2(1-\beta_2)}}\beta_2^{\frac{k}{2}} \leq \frac{1}{\sqrt{1-\beta_2}}\beta_2^{\frac{k+1}{2}}$ for all $k$. Then, recalling that $\|\widehat{\bm{g}}_0\|_\infty^2 \leq \widehat{\kappa}^2 \|\widehat{\bm{g}}_0\|_2^2 \leq 9 \widehat{\kappa}^2 d B^2$, we get using \cref{eq:45-dec29}:
    \begin{equation}
        \widehat{v}_k + \max\Big(\|\widehat{\bm{g}}_0\|_\infty^2, \frac{\phi^2}{q_\text{min}^2}\Big) {\delta}_k^2 \leq \Big(\frac{2 \beta_2^{k+1}}{1-\beta_2}\Big) \max\Big(9 \widehat{\kappa}^2 d B^2, \frac{\phi^2}{q_\text{min}^2}\Big).
    \end{equation}
    Plugging this in \cref{eq:21-nov23}, we get:
    \begin{flalign}
        \label{eq:30-dec23}
        \|\bar{\bm{g}}_{k+1}\|_2 & \leq \Bigg(1 - \frac{\alpha \mu_1 \sqrt{1-\beta_2} \beta_2^{-{(k+1)/2}}}{\sqrt{2} \max\big(3 \widehat{\kappa} \sqrt{d} B, \phi/q_\text{min}\big)} \Bigg) \|\bar{\bm{g}}_k\|_2.
    \end{flalign}
    \Cref{lem1-oct1} ensures that $\frac{\alpha \mu_1 \sqrt{1-\beta_2} \beta_2^{-{(k+1)/2}}}{\sqrt{2} \max(3 \widehat{\kappa} \sqrt{d} B, \phi/q_\text{min})} \leq 1$ for all $k \leq \tilde{K}$; refer to its statement. Now, using the fact that $\exp(-z) \geq 1-z$ for all $z \in \mathbb{R}$ above, we get:
    \begin{equation}
        \label{eq:32-nov28}
        \|\bar{\bm{g}}_{k+1}\|_2 \leq \exp\Bigg(- \frac{\alpha \mu_1 \sqrt{1-\beta_2} \beta_2^{-{(k+1)/2}}}{\sqrt{2} \max\big(3 \widehat{\kappa} \sqrt{d} B, \phi/q_\text{min}\big)} \Bigg)\|\bar{\bm{g}}_k\|_2.
    \end{equation}
    Unfolding the recursion in \cref{eq:32-nov28} gives us:
    \begin{flalign}
    \|\bar{\bm{g}}_{k+1}\|_2 & \leq \exp\Bigg(-\frac{\alpha \mu_1 \sqrt{1-\beta_2} \beta_2^{-{1/2}}}{\sqrt{2} \max\big(3 \widehat{\kappa} \sqrt{d} B, \phi/q_\text{min}\big)} \sum_{l=0}^{k}\big(\beta_2^{-1/2}\big)^l\Bigg)\|\bar{\bm{g}}_0\|_2
    \\
    & = \exp\Bigg(-\frac{\alpha \mu_1 \sqrt{1-\beta_2}}{\sqrt{2} \max\big(3 \widehat{\kappa} \sqrt{d} B, \phi/q_\text{min}\big)} \Bigg(\frac{\beta_2^{-(k+1)/2} - 1}{1 - \sqrt{\beta_2}}\Bigg)\Bigg)\|\bar{\bm{g}}_0\|_2.
    \end{flalign}
    Recall that $\beta_2 = 1 - \frac{\alpha^2 \mu_1^2}{8 \max(9 \widehat{\kappa}^2 d B^2, \phi^2/q_\text{min}^2)}$. Using the fact that $\sqrt{\beta_2} \geq \beta_2$ above, we get:
    \begin{equation}
    \label{eq:34-dec23-new}
    \|\bar{\bm{g}}_{k+1}\|_2 \leq \exp\Bigg(-2\Big(\beta_2^{-(k+1)/2} - 1\Big)\Bigg)\|\bar{\bm{g}}_0\|_2.
    \end{equation}
\end{proof}

\begin{lemma}
    \label{lem1-oct1}
    In the setting of \Cref{lem1-sep30}, $\min_{j \in [d]} \big(\widehat{v}_{k}^{(j)} + \max\big((\widehat{g}_0^{(j)})^2, \phi_j^2\big) \delta_k^{2}\big) \geq (\mu_2 \alpha)^2$ until at least: $$\tilde{K} = \frac{\log\Big(\frac{\min_{j \in [d]} (\widehat{g}_0^{(j)})^2}{(\mu_2 \alpha)^2} + \frac{\max(\min_{j \in [d]} (\widehat{g}_0^{(j)})^2, {\phi^2}/{q_\textup{max}^2})}{{2(1-\beta_2)(\mu_2 \alpha)^2}}\Big)}{\log \frac{1}{\beta_2}}$$
    iterations. Further, this value of $\tilde{K}$ also ensures the condition after \cref{eq:30-dec23} in the proof of \Cref{lem1-sep30}, i.e.,
    \begin{equation}
        \label{eq:55-jan16}
        \frac{\alpha \mu_1 \sqrt{1-\beta_2} \beta_2^{-{(k+1)/2}}}{\sqrt{2} \max(3 \widehat{\kappa} \sqrt{d} B, \phi/q_\textup{min})} \leq 1 \text{ } \forall \text{ } k \leq \tilde{K},
    \end{equation}
    holds.
\end{lemma}
\begin{proof}
Note that for any $j \in [d]$, $\widehat{v}_{k}^{(j)} + \max\big((\widehat{g}_0^{(j)})^2, \phi_j^2\big) \delta_k^{2} \geq \beta_2^{k} \big(\widehat{g}_0^{(j)}\big)^2 + \max\big((\widehat{g}_0^{(j)})^2, \phi_j^2\big) \delta_k^{2}$. Plugging in our choice of $\delta_k$ from \Cref{lem1-sep30} into this, we get:
$$\min_{j \in [d]} \Big(\widehat{v}_{k}^{(j)} + \max\big((\widehat{g}_0^{(j)})^2, \phi_j^2\big) \delta_k^{2}\Big) \geq \min_{j \in [d]} \Bigg((\widehat{g}_0^{(j)})^2 + \frac{\max\big((\widehat{g}_0^{(j)})^2, \phi_j^2\big)}{{2(1-\beta_2)}}\Bigg) \beta_2^{{k}}.$$
So as long as $\min_{j \in [d]} \Big((\widehat{g}_0^{(j)})^2 + \frac{\max((\widehat{g}_0^{(j)})^2, \phi_j^2)}{{2(1-\beta_2)}}\Big) \beta_2^{{k}} \geq (\mu_2 \alpha)^2$, i.e.,
\begin{equation}
    \label{eq:34-dec22}
    \min_{j \in [d]} \Bigg(\frac{(\widehat{g}_0^{(j)})^2}{\mu_2^2} + \frac{\max\big((\widehat{g}_0^{(j)})^2, \phi_j^2\big)}{{2(1-\beta_2)\mu_2^2}}\Bigg) \beta_2^{{k}} \geq \alpha^2,
\end{equation}
we are good. Let us ensure that \cref{eq:55-jan16} is satisfied by the above condition. Note that \cref{eq:55-jan16} is equivalent to:
\begin{equation*}
    \frac{2 \max(9 \widehat{\kappa}^2 d B^2, \phi^2/q_\textup{min}^2) \beta_2^{k+1}}{(1-\beta_2) \mu_1^2} \geq \alpha^2.
\end{equation*}
But per \cref{eq:44-jan16-new}, we have $\beta_2 \geq \frac{1}{2}$ and therefore, the above is ensured by:
\begin{equation}
    \label{eq:36-dec23}
    \frac{\max(9 \widehat{\kappa}^2 d B^2, \phi^2/q_\textup{min}^2) \beta_2^{k}}{(1-\beta_2) \mu_1^2} \geq \alpha^2.
\end{equation}
Now note that:
\begin{flalign}
    \label{eq:60-feb3}
    \min_{j \in [d]} \Bigg(\frac{(\widehat{g}_0^{(j)})^2}{\mu_2^2} + \frac{\max\big((\widehat{g}_0^{(j)})^2, \phi_j^2\big)}{{2(1-\beta_2)\mu_2^2}}\Bigg) & \leq \min_{j \in [d]} \Bigg(\frac{(\widehat{g}_0^{(j)})^2}{\mu_1^2} + \frac{\max\big((\widehat{g}_0^{(j)})^2, \phi_j^2\big)}{{2(1-\beta_2)\mu_1^2}}\Bigg)
    \\
    \label{eq:61-feb3}
    & \leq \min_{j \in [d]} \Bigg(\frac{(\widehat{g}_0^{(j)})^2}{2(1-\beta_2)\mu_1^2} + \frac{\max\big((\widehat{g}_0^{(j)})^2, \phi_j^2\big)}{{2(1-\beta_2)\mu_1^2}}\Bigg)
    \\
    \label{eq:62-feb3}
    & \leq \min_{j \in [d]} \Bigg(\frac{\max\big((\widehat{g}_0^{(j)})^2, \phi_j^2\big)}{{(1-\beta_2)\mu_1^2}}\Bigg)
    \\
    \label{eq:63-feb3}
    & \leq \frac{\max(9 \widehat{\kappa}^2 d B^2, \phi^2/q_\textup{min}^2)}{{(1-\beta_2)\mu_1^2}}.
\end{flalign}
\Cref{eq:60-feb3} follows because $\mu_1 \leq \mu_2$. \Cref{eq:61-feb3} follows because $2(1-\beta_2) \leq 1$ (since $\beta_2 \geq \frac{1}{2}$). \Cref{eq:62-feb3} holds because $(\widehat{g}_0^{(j)})^2 \leq \max\big((\widehat{g}_0^{(j)})^2, \phi_j^2\big)$. \Cref{eq:63-feb3} is obtained by using the fact that $\min_{j \in [d]} \max\big((\widehat{g}_0^{(j)})^2, \phi_j^2\big) \leq \max\big(\|\widehat{\bm{g}}_0\|_2^2, \phi^2/q_\textup{min}^2\big) \leq \max \big(\widehat{\kappa}^2  \|\bar{\bm{g}}_0\|_2^2, \phi^2/q_\textup{min}^2) \leq \max(9 \widehat{\kappa}^2 d B^2, \phi^2/q_\textup{min}^2)$ (where the last step follows from the fact that $\|\widehat{\bm{g}}_0\|_2 \leq 3 \sqrt{d} B$). Thus, \cref{eq:34-dec22} ensures \cref{eq:55-jan16}.
Simplifying \cref{eq:34-dec22}, we get that till at least
\begin{equation}
    \tilde{K}_1 = \frac{\log\Big(\min_{j \in [d]} \Big(\frac{(\widehat{g}_0^{(j)})^2}{(\mu_2 \alpha)^2} + \frac{\max((\widehat{g}_0^{(j)})^2, \phi_j^2)}{{2(1-\beta_2)(\mu_2 \alpha)^2}}\Big)\Big)}{\log \frac{1}{\beta_2}}
\end{equation}
iterations, $\min_{j \in [d]} \big(\widehat{v}_{k}^{(j)} + \max\big((\widehat{g}_0^{(j)})^2, \phi_j^2\big) \delta_k^{2}\big) \geq (\mu_2 \alpha)^2$. 
Using the fact that $\phi_j \geq \frac{\phi}{q_\textup{max}}$ (recall that $q_\textup{max} = \max_{i \in [d]} q_i$), we have:
\begin{flalign}
    \tilde{K}_1 & \geq \frac{\log\Big(\min_{j \in [d]} \Big(\frac{(\widehat{g}_0^{(j)})^2}{(\mu_2 \alpha)^2} + \frac{\max((\widehat{g}_0^{(j)})^2, {\phi^2}/{q_\textup{max}^2})}{{2(1-\beta_2)(\mu_2 \alpha)^2}}\Big)\Big)}{\log \frac{1}{\beta_2}} 
    \\
    & = \frac{\log\Big(\frac{\min_{j \in [d]} (\widehat{g}_0^{(j)})^2}{(\mu_2 \alpha)^2} + \frac{\max(\min_{j \in [d]} (\widehat{g}_0^{(j)})^2, {\phi^2}/{q_\textup{max}^2})}{{2(1-\beta_2)(\mu_2 \alpha)^2}}\Big)}{\log \frac{1}{\beta_2}}.
\end{flalign}
The above is our final bound for $\tilde{K}$.
\end{proof}

\begin{lemma}
    \label{lem3-jan7}
    Fix some $p \in [\frac{1}{e},1)$. Then in the setting of \Cref{thm3-det}, $\min_{j \in [d]} |\widehat{g}_0^{(j)}|  > \frac{2B}{d} \Big(\frac{\log ({1}/{p})}{\zeta\big(1 + \frac{\log (1/p)}{d}\big)}\Big)^{1/\theta}$ with a probability of at least $p$.
\end{lemma}
\begin{proof}
    From \Cref{dist}, we have:
    \begin{equation}
        \mathbb{P}\Big(\min_{j \in [d]} |\widehat{g}_0^{(j)}| > \frac{2 t B}{d}\Big) \geq \exp\Bigg(-\frac{\zeta t^\theta}{1 - \frac{\zeta t^\theta}{d}}\Bigg).
    \end{equation}
    Choosing $t = \Big(\frac{\log ({1}/{p})}{\zeta\big(1 + \frac{\log (1/p)}{d}\big)}\Big)^{1/\theta}$ yields the desired result. Note that if we have $p \geq \frac{1}{e}$, then $t \leq (\frac{1}{\zeta})^{1/\theta} \leq 1$ and so this falls within the range of \Cref{dist}.
\end{proof}

\begin{fact}
    \label{log-fact}
    For $a \in (0, 1)$ we have:
    $\log \Big(\frac{1}{1-a}\Big) \geq a.$
\end{fact}
\begin{proof}
    This follows by using the Taylor series expansion of $\log(1-z)$ for $z \in (0,1)$. Specifically: 
    $$\log \Big(\frac{1}{1-a}\Big) = -\log (1-a) = a + \frac{a^2}{2} + \frac{a^3}{3} + \ldots \geq a.$$
\end{proof}

\section{Detailed Version and Proof of Theorem~\ref{thm3-diag}}
\label{pf-2}
\begin{theorem}[\textbf{Detailed Version of \Cref{thm3-diag}}]
    \label{thm3-diag-det}
    Let $\bm{Q} = \bm{\Lambda}$ be diagonal. Suppose $\|\bm{x}_{*}\|_\infty \leq B$ and our initialization $\bm{x}_0$ is sampled randomly such that ${x}_0^{(i)} \underset{\textup{iid}}{\sim} \textup{Unif}[-2B, 2B]$ $\forall$ $i \in [d]$. Fix some $p \in (0,1)$. Choose 
    $$\alpha = \mathcal{O}\Bigg(B \times {\max\Bigg(\sqrt{{\frac{\log (1/p)}{d + {\log(1/p)}}}}, \frac{1}{\sqrt{\kappa}}\Bigg)} \log^{-1/2}\Big(\frac{\sqrt{d \kappa} B}{\varepsilon}\Big)\Bigg),$$
    $$\beta_2 = 1 - \frac{\alpha^2}{72 B^2}, \delta_k = \frac{\beta_2^{{k}/{2}}}{\sqrt{2(1-\beta_2)}} \text{ and } \phi = B.$$ 
    Then with a probability of at least $p$ over the randomness of $\bm{x}_0$, $f(\bm{x}_{{K}^{*}}) < \frac{\varepsilon^2}{2}$ in 
    \begin{equation*}
        K^{*} = \mathcal{O}\Bigg(\kappa_\textup{Adam} {\log\Big(\frac{\sqrt{d \kappa} B}{\varepsilon}\Big) {\log\Big(\log\Big(\frac{\sqrt{d \kappa} \|\bm{x}_{0} - \bm{x}_{*}\|_\infty}{\varepsilon}\Big)\Big)}}\Bigg)
    \end{equation*}
    iterations of {Adam}, where $\kappa_\textup{Adam} := \min\Big(\frac{d}{\log (1/p)} + 1, \kappa\Big)$.
\end{theorem}
The proof of \Cref{thm3-diag} is similar to that of \Cref{thm3} in \Cref{pf-1}.
\begin{proof}
In this case, $\bm{Q} = \bm{\Lambda}$ is diagonal. Recall Adam's update rule with $\beta_1 = 0$ from \cref{eq:58-oct15} (in the proof of \Cref{thm3}):
\begin{equation}
    x^{(i)}_{k+1} = x^{(i)}_{k} - \frac{\alpha {g}_k^{(i)}}{\sqrt{{\sum_{l=0}^{k} \beta_2^{k-l} \big({g}_l^{(i)}\big)^2}+ \max\Big(\big({g}_0^{(i)}\big)^2, \phi^2\Big)\delta_{k}^{2}}}.
\end{equation}
Here, ${g}_l^{(i)} = \lambda_i ({x}_l^{(i)} - x_{*}^{(i)})$; note that $\lambda_i = \bm{\Lambda}[i,i]$. Subtracting $x_{*}^{(i)}$ from both sides of the above equation, we get:
\begin{equation}
    \label{eq:36-dec13-1-1}
    (x^{(i)}_{k+1} - x_{*}^{(i)}) = \Bigg(1 - \frac{\alpha}{\sqrt{{\sum_{l=0}^{k} \beta_2^{k-l} \big({x}_l^{(i)} - x_{*}^{(i)}\big)^2}+ \max\Big(\big({x}_0^{(i)} - x_{*}^{(i)}\big)^2, \frac{\phi^2}{\lambda_i^2}\Big)\delta_{k}^{2}}}\Bigg) (x^{(i)}_{k} - x_{*}^{(i)}).
\end{equation}
Note that each coordinate's update rule is independent of the other coordinates. For ease of notation, let us define
\begin{equation}
    \bm{z}_{l} := \bm{x}_{l} - \bm{x}_{*}, \widehat{v}^{(i)}_k := \sum_{l=0}^{k} \beta_2^{k-l} \big({z}_l^{(i)}\big)^2 \text{ and } \phi_i := \frac{\phi}{\lambda_i}.
\end{equation}
With this notation, \cref{eq:36-dec13-1-1} can be rewritten as:
\begin{equation}
    \label{eq:36-dec13-1}
    z^{(i)}_{k+1} = \Bigg(1 - \frac{\alpha}{\sqrt{\widehat{v}^{(i)}_k + \max\Big(\big({z}_0^{(i)}\big)^2, \phi_i^2\Big)\delta_{k}^{2}}}\Bigg) z^{(i)}_{k}.
\end{equation}
Now as long as $\widehat{v}^{(i)}_k + \max\Big(\big({z}_0^{(i)}\big)^2, \phi_i^2\Big)\delta_{k}^{2} > {\alpha}^2$, \cref{eq:36-dec13-1} yields:
\begin{equation}
    \label{eq:36-dec13}
    |z^{(i)}_{k+1}| = \Bigg(1 - \frac{\alpha}{\sqrt{\widehat{v}^{(i)}_k + \max\Big(\big({z}_0^{(i)}\big)^2, \phi_i^2\Big)\delta_{k}^{2}}}\Bigg) |z^{(i)}_{k}|,
\end{equation}
i.e., descent is ensured in each iteration or we have continuous descent (similar to \Cref{thm3}). Choose $\alpha \leq \frac{\max(3 B, \phi)}{\sqrt{2}}$, $\beta_2 = 1 - \frac{\alpha^2}{8 \max(9 B^2, \phi^2)}$ and $\delta_k = \frac{1}{\sqrt{2(1-\beta_2)}} \beta_2^{\frac{k}{2}}$. From \Cref{lem1-oct1-2}, we have  continuous descent for coordinate $i \in [d]$ until at least:
\begin{flalign}
    \nonumber
    \tilde{K}_i & = \frac{\log\Big(\frac{(z_0^{(i)})^2}{\alpha^2} + \frac{\max(({z}_0^{(i)})^2, \phi_i^2)}{2(1-\beta_2) \alpha^2}\Big)}{\log \frac{1}{\beta_2}} 
    \\
    \label{eq:62-jan16}
    & = \frac{\log\Big(\frac{(z_0^{(i)})^2}{\alpha^2} + \frac{4 \max(({z}_0^{(i)})^2, \phi_i^2) \max(9 B^2, \phi^2)}{\alpha^4}\Big)}{\log \frac{1}{\gamma}}
\end{flalign}
iterations. The last equality in \cref{eq:62-jan16} follows by plugging in our choice of $\beta_2 = 1 - \frac{\alpha^2}{8 \max(9 B^2, \phi^2)}$. Also from \Cref{lem1-sep30-2}, we have for any $i \in [d]$ and $k \in \{1,\ldots,\tilde{K}_i\}$:
\begin{equation}
    |z^{(i)}_{k}| \leq \exp\Bigg(-2\Big(\beta_2^{-k/2} - 1\Big)\Bigg) |z^{(i)}_{0}|.
\end{equation}
Now in order to have $|z^{(i)}_{K_i^{*}}| \leq \varepsilon$, we can choose: 
\begin{equation}
    K_i^{*} = \frac{\log\Big(\frac{1}{4}\log^2\Big(\frac{e^{2}|z^{(i)}_{0}|}{\varepsilon}\Big)\Big)}{\log \frac{1}{\beta_2}}.
\end{equation}
Thus, we have $\|\bm{z}_{K^{*}}\|_\infty \leq \varepsilon$ in:
\begin{equation}
    \label{eq:65-jan16}
    K^{*} = \frac{\log\Big(\frac{1}{4}\log^2\Big(\frac{e^{2}\|\bm{z}_{0}\|_\infty}{\varepsilon}\Big)\Big)}{\log \frac{1}{\beta_2}}
\end{equation}
iterations. Now in order for this result to be applicable, let us ensure $K^{*} \leq \min_{i \in [d]}\tilde{K}_i$; for that, we must have:
\begin{equation}
    \frac{1}{4}\log^2\Big(\frac{e^{2}\|\bm{z}_{0}\|_\infty}{\varepsilon}\Big) \leq \frac{(z_0^{(i)})^2}{\alpha^2} + \frac{4 \max(({z}_0^{(i)})^2, \phi_i^2) \max(9 B^2, \phi^2)}{\alpha^4}.
\end{equation}
The above can be ensured by having:
\begin{equation}
    \frac{1}{4}\log^2\Big(\frac{e^{2}\|\bm{z}_{0}\|_\infty}{\varepsilon}\Big) \leq \frac{4 \max(({z}_0^{(i)})^2, \phi_i^2) \max(9 B^2, \phi^2)}{\alpha^4},
\end{equation}
and this gives us:
\begin{equation}
    \alpha \leq \mathcal{O}\Bigg(\Big({\min_{i \in [d]} \big(\max(({z}_0^{(i)})^2, \phi_i^2)\big) \max(B^2, \phi^2)} \Big)^{1/4} \log^{-1/2}\Big(\frac{\|\bm{z}_{0}\|_\infty}{\varepsilon}\Big)\Bigg).
\end{equation}
Note that this also satisfies our earlier condition on $\alpha$, namely $\alpha \leq \frac{\max(3 B, \phi)}{\sqrt{2}}$. From \Cref{lem5-jan3}, we have $\min_{i \in [d]} |{z}_0^{(i)}| > \Big(\frac{2 \log ({1}/{p})}{1 + \frac{\log ({1}/{p})}{d}}\Big)\frac{B}{d}$ with a probability of at least $p$. For simplicity of notation, let $\widetilde{p} = \frac{\log (1/p)}{1 + \frac{\log(1/p)}{d}}$.
Moreover, $\min_{i \in [d]} \phi_i = \frac{\phi}{\kappa}$ and $\|\bm{z}_0\|_\infty = \|\bm{x}_0 - \bm{x}_{*}\|_\infty \leq 3 B$. 
Thus, with a probability of at least $p$, we can pick:
\begin{equation}
    \label{eq:15-dec10-2}
    \alpha = \mathcal{O}\Bigg(\Bigg({\max\Big(\frac{\widetilde{p}^2 B^2}{d^2}, \frac{\phi^2}{\kappa^2}\Big) \max(B^2, \phi^2)} \Bigg)^{1/4} \log^{-1/2}\Big(\frac{B}{\varepsilon}\Big)\Bigg).
\end{equation}
Next, recalling that $\beta_2 = 1 - \frac{\alpha^2}{8 \max(9 B^2, \phi^2)}$, $\log \frac{1}{\beta_2} \geq \frac{\alpha^2}{8\max(9 B^2, \phi^2)}$ using \Cref{log-fact}. Plugging this in the expression of $K^{*}$ (\cref{eq:65-jan16}), we get:
\begin{equation}
    K^{*} = \mathcal{O}\Bigg(\frac{\max(B^2, \phi^2)}{\alpha^2}{\log\Big(\log\Big(\frac{\|\bm{z}_0\|_\infty}{\varepsilon}\Big)\Big)}\Bigg).
\end{equation}
Plugging in the value of $\alpha$ from \cref{eq:15-dec10-2} above, we get:
\begin{equation}
    \label{eq:17-dec10-2}
    K^{*} = \mathcal{O}\Bigg(\Bigg(\frac{\max\big(B, \phi\big)}{\max\big(\frac{\widetilde{p}B}{d}, \frac{\phi}{\kappa}\big)}\Bigg) \log\Big(\frac{B}{\varepsilon}\Big) {\log\Big(\log\Big(\frac{\|\bm{z}_{0}\|_\infty}{\varepsilon}\Big)\Big)}\Bigg),
\end{equation}
with a probability of at least $p$ over random initialization. 
Next, note that: $$f(\bm{x}_{{K}^{*}}) = \frac{1}{2}(\bm{x}_{{K}^{*}} - \bm{x}_{*})^\top \bm{\Lambda} (\bm{x}_{{K}^{*}} - \bm{x}_{*}) =\frac{1}{2}\sum_{i \in [d]} \lambda_i ({z}_{{K}^{*}}^{(i)})^2 \leq \frac{d \kappa}{2} \varepsilon^2,$$
where the last step follows because $\|\bm{z}_{K^{*}}\|_\infty \leq \varepsilon$ and $\lambda_i \leq \kappa$ $\forall$ $i \in [d]$. So if we wish to have $f(\bm{x}_{{K}^{*}}) \leq \frac{\varepsilon^2}{2}$, we can replace $\varepsilon$ by $\frac{\varepsilon}{\sqrt{d \kappa}}$ above, i.e., $f(\bm{x}_{{K}^{*}}) \leq \frac{\varepsilon^2}{2}$ in:
\begin{equation}
    K^{*} = \mathcal{O}\Bigg(\Bigg(\frac{\max\big(B, \phi\big)}{\max\big(\frac{\widetilde{p} B}{d}, \frac{\phi}{\kappa}\big)}\Bigg) \log\Big(\frac{\sqrt{d \kappa} B}{\varepsilon}\Big) {\log\Big(\log\Big(\frac{\sqrt{d \kappa} \|\bm{z}_{0}\|_\infty}{\varepsilon}\Big)\Big)}\Bigg)
\end{equation}
iterations, where recall that $\widetilde{p} = \frac{\log (1/p)}{1 + \frac{\log(1/p)}{d}}$. Finally, the theorem statement follows by setting $\phi = B$ in the analysis above.
\end{proof}

\begin{lemma}[Similar to \Cref{lem1-sep30}]
\label{lem1-sep30-2}
Suppose \cref{eq:36-dec13} holds, i.e., for each $i \in [d]$: $$|z^{(i)}_{k+1}| = \Bigg(1 - \frac{\alpha}{\sqrt{\widehat{v}^{(i)}_k + \max\Big(\big({z}_0^{(i)}\big)^2, \phi_i^2\Big)\delta_{k}^{2}}}\Bigg) |z^{(i)}_{k}|,$$
for $k = 0$ through to $k = \tilde{K}_i$, where $\widehat{v}^{(i)}_k + \max\Big(\big({z}_0^{(i)}\big)^2, \phi_i^2\Big) \delta_{k}^{2} > {\alpha}^2$ $\forall$ $k \in \{0,\ldots,\tilde{K}_i\}$. 
Suppose we set $\alpha \leq \frac{\max(3 B, \phi)}{\sqrt{2}}$, $\beta_2 = 1 - \frac{\alpha^2}{8 \max(9 B^2, \phi^2)}$ and $\delta_k = \frac{1}{\sqrt{2(1-\beta_2)}}\beta_2^{\frac{k}{2}}$ for $k \geq 0$. Then for any $k \leq \tilde{K}_i$, we have:
\begin{equation*}
    |z^{(i)}_{k+1}| \leq \exp\Bigg(-2\Big(\beta_2^{-(k+1)/2} - 1\Big)\Bigg) |z^{(i)}_{0}|.
\end{equation*}
\end{lemma}
\begin{proof}
    Similar to \Cref{lem1-sep30}, we employ a two-stage analysis procedure where we first obtain a loose bound for $|z^{(i)}_{k+1}|$ in terms of $|z^{(i)}_{0}|$, and then obtain a tighter bound by using this loose bound.
    \\
    \\
    \textbf{First Stage.} From \cref{eq:36-dec13}, we have $|z_0^{(i)}| \geq |z_1^{(i)}| \geq \ldots \geq |z_{\tilde{K}_i+1}^{(i)}|$. So for any $k \leq \tilde{K}_i$:
    \begin{equation}
    \widehat{v}^{(i)}_k = {\sum_{l=0}^k \beta_2^{k-l} \big(z_l^{(i)}\big)^2} \leq {\sum_{l=0}^k \beta_2^{k-l} \big(z_0^{(i)}\big)^2} \leq \frac{\big(z_0^{(i)}\big)^2}{1-\beta_2}.
    \end{equation}
    Further, we have chosen the ${\delta}_k$'s so that ${\delta}_k \leq {\delta}_0 \leq \frac{1}{\sqrt{1-\beta_2}}$ for all $k \geq 1$. Then:
    \begin{equation}
        \widehat{v}^{(i)}_k + \max\Big(\big({z}_0^{(i)}\big)^2, \phi_i^2\Big) \delta_{k}^{2} \leq \frac{2 \max\big(\big({z}_0^{(i)}\big)^2, \phi_i^2\big)}{1-\beta_2} \leq \frac{2 \max\big(9 B^2, \phi^2\big)}{1 - \beta_2},
    \end{equation}
    where the last step follows from the fact that $\max_{i \in [d]} \phi_i = \max_{i \in [d]} \frac{\phi}{\lambda_i} = \phi$ since $\min_{i \in [d]} \lambda_i = 1$, and because $\|\bm{z}_0\|_\infty = \|\bm{x}_0 - \bm{x}_{*}\|_\infty \leq 3 B$. 
    Plugging this into \cref{eq:36-dec13} gives us:
    \begin{equation}
        \label{eq:73-jan16}
        \big(z^{(i)}_{k+1}\big)^2 \leq \Bigg(1 - \frac{{\alpha} \sqrt{1-\beta_2}}{\sqrt{2} \max(3 B, \phi)}\Bigg)^2\big(z^{(i)}_{k}\big)^2 \leq \Bigg(1 - \frac{{\alpha} \sqrt{1-\beta_2}}{\sqrt{2} \max(3 B, \phi)}\Bigg) \big(z^{(i)}_{k}\big)^2.
    \end{equation}
    Define $\gamma := \Big(1 - \frac{{\alpha} \sqrt{1-\beta_2}}{\sqrt{2} \max(3 B, \phi)}\Big)$. Unfolding the recursion in \cref{eq:73-jan16} yields:
    \begin{equation}
        \label{eq:86-nov11-2}
        \big(z^{(i)}_{k+1}\big)^2 \leq \gamma^{k+1} \big(z^{(i)}_{0}\big)^2.
    \end{equation}
    This bound is loose but we shall obtain a tighter bound using this bound. 
    \\
    \\
    \noindent \textbf{Second Stage.} With our choice of $\beta_2 = 1 - \frac{\alpha^2}{8 \max(9 B^2, \phi^2)}$, we have $\gamma = 1 - \frac{\alpha^2}{4 \max(9 B^2, \phi^2)}$ and $\beta_2 - \gamma = {1 - \beta_2}$. 
    Now for any $k \leq \tilde{K}_i$, we have:
    \begin{flalign}
    \widehat{v}_k^{(i)} = \sum_{l=0}^{k} \beta_2^{k-l} \big(z_l^{(i)}\big)^2 & \leq \sum_{l=0}^{k} \beta_2^{k-l} \gamma^l \big(z_0^{(i)}\big)^2
    \\
    & = \Bigg(\frac{\beta_2^{k+1} - \gamma^{k+1}}{\beta_2 - \gamma}\Bigg) \big(z_0^{(i)}\big)^2
    \\
    & \leq \Big(\frac{\beta_2^{k+1}}{1 - \beta_2}\Big) \big(z_0^{(i)}\big)^2,
    \end{flalign}
    where the last step follows from the fact that $\beta_2 - \gamma = {1 - \beta_2}$. Note that with our constraint of $\alpha \leq \frac{\max(3 B, \phi)}{\sqrt{2}}$, we have: 
    \begin{equation}
        \label{eq:74-jan16-new}
        \beta_2 \geq \frac{1}{2}.
    \end{equation}
    Thus, we have $\delta_k = \frac{1}{\sqrt{2(1-\beta_2)}}\beta_2^{\frac{k}{2}} \leq \frac{1}{\sqrt{1-\beta_2}}\beta_2^{\frac{k+1}{2}}$ for all $k$. Hence,
    \begin{equation}
        \widehat{v}_k^{(i)} + \max\Big(\big({z}_0^{(i)}\big)^2, \phi_i^2\Big)  {\delta}_k^2 \leq \Big(\frac{2 \beta_2^{k+1}}{1-\beta_2}\Big) \max\Big(\big({z}_0^{(i)}\big)^2, \phi_i^2\Big) \leq \Big(\frac{2 \beta_2^{k+1}}{1-\beta_2}\Big)  \max\big(9 B^2, \phi^2\big),
    \end{equation}
    where the last step follows from the fact that $\|\bm{z}_0\|_\infty \leq 3 B$.
    Plugging this in \cref{eq:36-dec13}, we get:
    \begin{flalign}
        \label{eq:59-dec23}
        |z^{(i)}_{k+1}| & \leq \Bigg(1 - \frac{\alpha \sqrt{1-\beta_2} \beta_2^{-{(k+1)/2}}}{\sqrt{2} \max\big(3 B, \phi\big)} \Bigg) |z^{(i)}_{k}|.
    \end{flalign}
    \Cref{lem1-oct1-2} ensures that $\frac{\alpha \sqrt{1-\beta_2} \beta_2^{-{(k+1)/2}}}{\sqrt{2} \max(3 B, \phi)} \leq 1$ for all $k \leq \tilde{K}_i$; refer to its statement.
    Next, using the fact that $\exp(-z) \geq 1-z$ for $z \in \mathbb{R}$ above, we get:
    \begin{equation}
        \label{eq:32-nov28-2}
        |z^{(i)}_{k+1}| \leq \exp\Bigg(- \frac{\alpha \sqrt{1-\beta_2} \beta_2^{-{(k+1)/2}}}{\sqrt{2} \max\big(3 B, \phi\big)} \Bigg) |z^{(i)}_{k}|.
    \end{equation}    
    Unfolding the recursion in \cref{eq:32-nov28-2} gives us:
    \begin{flalign}
    |z^{(i)}_{k+1}|  & \leq \exp\Bigg(-\frac{\alpha \sqrt{1-\beta_2} \beta_2^{-{1/2}}}{\sqrt{2} \max\big(3 B, \phi\big)}
    \sum_{l=0}^{k}\big(\beta_2^{-1/2}\big)^l\Bigg)|z^{(i)}_{0}| 
    \\
    & = \exp\Bigg(-\frac{\alpha \sqrt{1-\beta_2}}{\sqrt{2} \max\big(3 B, \phi\big)} \Bigg(\frac{\beta_2^{-(k+1)/2} - 1}{1 - \sqrt{\beta_2}}\Bigg)\Bigg)|z^{(i)}_{0}|.
    \end{flalign}
    Recall that $\beta_2 = 1 - \frac{\alpha^2}{8 \max(9 B^2, \phi^2)}$. Using the fact that $\sqrt{\beta_2} \geq \beta_2$ above, we get:
    \begin{equation}
    \label{eq:63-dec23}
    |z^{(i)}_{k+1}| \leq \exp\Bigg(-2\Big(\beta_2^{-(k+1)/2} - 1\Big)\Bigg)|z^{(i)}_{0}|.
    \end{equation}
\end{proof}

\begin{lemma}[{Similar to \Cref{lem1-oct1}}]
    \label{lem1-oct1-2}
    In the setting of \Cref{lem1-sep30-2}, $\widehat{v}_{k}^{(i)} + \max\Big(\big({z}_0^{(i)}\big)^2, \phi_i^2\Big) \delta_k^{2} \geq \alpha^2$ until at least: $$\tilde{K}_i = \frac{\log\Big(\frac{(z_0^{(i)})^2}{\alpha^2} + \frac{\max(({z}_0^{(i)})^2, \phi_i^2)}{2(1-\beta_2) \alpha^2}\Big)}{\log \frac{1}{\beta_2}}$$
    iterations. Further, this value of $\tilde{K}_i$ also ensures the condition after \cref{eq:59-dec23} in the proof of \Cref{lem1-sep30-2}, i.e.,
    \begin{equation}
        \label{eq:84-jan16}
        \frac{\alpha \sqrt{1-\beta_2} \beta_2^{-{(k+1)/2}}}{\sqrt{2} \max(3 B, \phi)} \leq 1 \text{ } \forall \text{ } k \leq \tilde{K}_i,
    \end{equation}
    holds.
\end{lemma}
\begin{proof}
Note that for any $i \in [d]$, $\widehat{v}_{k}^{(i)} + \max\Big(\big({z}_0^{(i)}\big)^2, \phi_i^2\Big) \delta_k^{2} \geq \beta_2^{k} \big(z_0^{(i)}\big)^2 + \max\Big(\big({z}_0^{(i)}\big)^2, \phi_i^2\Big) \delta_k^{2}$. Plugging in our choice of $\delta_k$ from \Cref{lem1-sep30-2}, we get:
$$\widehat{v}_{k}^{(i)} + \max\Big(\big({z}_0^{(i)}\big)^2, \phi_i^2\Big) \delta_k^{2} \geq \Bigg(\big(z_0^{(i)}\big)^2 + \frac{\max\big(\big({z}_0^{(i)}\big)^2, \phi_i^2\big)}{{2(1-\beta_2)}}\Bigg)\beta_2^{{k}}.$$
So as long as 
\begin{equation}
    \label{eq:64-dec23}
    \Bigg(\big(z_0^{(i)}\big)^2 + \frac{\max\big(\big({z}_0^{(i)}\big)^2, \phi_i^2\big)}{{2(1-\beta_2)}}\Bigg)\beta_2^{{k}} \geq \alpha^2,
\end{equation}
we are good. Let us ensure that \cref{eq:84-jan16} is ensured by the above condition. Note that \cref{eq:84-jan16} is equivalent to:
\begin{equation*}
    \frac{2 \max(9 B^2, \phi^2) \beta_2^{k+1}}{(1-\beta_2)} \geq \alpha^2.
\end{equation*}
But per \cref{eq:74-jan16-new}, we have $\beta_2 \geq \frac{1}{2}$ and therefore, the above is ensured by:
\begin{equation}
    \label{eq:65-dec23}
    \frac{\max(9 B^2, \phi^2) \beta_2^{k}}{(1-\beta_2)} \geq \alpha^2.
\end{equation}
Now note that:
\begin{flalign}
    \label{eq:101-feb3}
    \Bigg(\big(z_0^{(i)}\big)^2 + \frac{\max\big(\big({z}_0^{(i)}\big)^2, \phi_i^2\big)}{{2(1-\beta_2)}}\Bigg) & \leq \Bigg(\frac{\big(z_0^{(i)}\big)^2}{2(1-\beta_2)} + \frac{\max\big(\big({z}_0^{(i)}\big)^2, \phi_i^2\big)}{{2(1-\beta_2)}}\Bigg)
    \\
    & \leq \frac{\max\big(\big({z}_0^{(i)}\big)^2, \phi_i^2\big)}{{(1-\beta_2)}}
    \\
    \label{eq:103-feb3}
    & \leq \frac{\max(9 B^2, \phi^2) \beta_2^{k}}{(1-\beta_2)}.
\end{flalign}
\Cref{eq:101-feb3} follows because $2(1-\beta_2) \leq 1$ (since $\beta_2 \geq \frac{1}{2}$). \Cref{eq:103-feb3} follows from the fact that $\max\Big(\big({z}_0^{(i)}\big)^2, \phi_i^2\Big) \leq \max\big(\|\bm{z}_0\|_\infty^2, \phi^2\big) \leq \max\big(9 B^2, \phi^2\big)$ for all $i \in [d]$. So \cref{eq:64-dec23} ensures \cref{eq:84-jan16}. Simplifying \cref{eq:64-dec23}, we get that till at least
\begin{equation}
    \tilde{K}_i = \frac{\log\Big(\frac{(z_0^{(i)})^2}{\alpha^2} + \frac{\max(({z}_0^{(i)})^2, \phi_i^2)}{2(1-\beta_2) \alpha^2}\Big)}{\log \frac{1}{\beta_2}}
\end{equation}
iterations, $\widehat{v}_{k}^{(i)} + \max\Big(\big({z}_0^{(i)}\big)^2, \phi_i^2\Big) \delta_k^{2} \geq \alpha^2$.
\end{proof}

\begin{lemma}
    \label{lem5-jan3}
    In the setting of \Cref{thm3-diag-det}, $\min_{i \in [d]} |{z}_0^{(i)}| > \Big(\frac{2 \log ({1}/{p})}{1 + \frac{\log ({1}/{p})}{d}}\Big)\frac{B}{d}$ with a probability of at least $p$.
\end{lemma}
\begin{proof}
    For any $i \in [d]$ and $t < d$, note that $\mathbb{P}\Big(|{z}_0^{(i)}| > \frac{2 t B}{d}\Big) = 1 - \frac{t}{d}$; this is because $|{z}_0^{(i)}| > \frac{2 t B}{d}$ is equivalent to ${x}_0^{(i)} \notin \big[x_{*}^{(i)} - \frac{2 t B}{d}, x_{*}^{(i)} + \frac{2 t B}{d}\big]$ which happens with a probability of $\frac{4B - \frac{4 t B}{d}}{4 B} = 1 - \frac{t}{d}$. Thus, $$\mathbb{P}\Big(\min_{i \in [d]} |{z}_0^{(i)}| > \frac{2 t B}{d}\Big) = \Big(1 - \frac{t}{d}\Big)^d = \exp\Big(d \log\Big(1 - \frac{t}{d}\Big)\Big).$$
    Now:
    \begin{equation*}
        d \log\Big(1 - \frac{t}{d}\Big) = - d \sum_{j=1}^\infty \frac{1}{j} \Big(\frac{t}{d}\Big)^j \geq - d \sum_{j=1}^\infty \Big(\frac{t}{d}\Big)^j = -\frac{t}{1 - \frac{t}{d}}.
    \end{equation*}
    Using this above, we get:
    \begin{equation}
        \label{eq:86-jan8}
        \mathbb{P}\Big(\min_{i \in [d]} |{z}_0^{(i)}| > \frac{2 t B}{d}\Big) \geq \exp\Bigg(-\frac{t}{1 - \frac{t}{d}}\Bigg).
    \end{equation}
    Choosing $t = \frac{\log ({1}/{p})}{1 + \frac{\log (1/p)}{d}}$ yields the desired result.
\end{proof}

\section{Proof of Theorem~\ref{lb-diag}}
\label{lb-diag-pf}
\begin{proof}
    In the proof of \Cref{thm3-diag}, recall that we defined $\bm{z}_{k} := \bm{x}_{k} - \bm{x}_{*}$. From \cref{eq:36-dec13-1} in the proof of \Cref{thm3-diag}, we have with $\delta_k = 0$:
    \begin{equation}
    z^{(i)}_{k+1} = z^{(i)}_{k} - \frac{\alpha z^{(i)}_{k}}{\sqrt{\widehat{v}^{(i)}_k}},
    \end{equation}
    where $\widehat{v}^{(i)}_k := \sum_{l=0}^{k} \beta_2^{k-l} \big({z}_l^{(i)}\big)^2$. Note that:
    $$\widehat{v}^{(i)}_k \geq (z_k^{(i)})^2.$$
    Thus:
    \begin{equation}
        |z^{(i)}_{k+1}| \geq |z^{(i)}_{k}| - \alpha.
    \end{equation}
    Solving the recursion above, we have for $k \geq 1$:
    \begin{equation}
        \label{eq:160-dec24}
        |z^{(i)}_{k}| \geq |z^{(i)}_{0}| - k \alpha.
    \end{equation}
    Let us consider the {best case}, when \cref{eq:160-dec24} holds with equality for all $i \in [d]$. Now, in order to achieve $|z^{(i)}_{K_i}| \leq \min_{j \in [d]} |z^{(j)}_{0}|$ for some $K_i \geq 1$ $\forall$ $i \in [d]$, we should set $\alpha = \min_{j \in [d]} |z^{(j)}_{0}|$ and in that case:
    \begin{equation}
       \max_{i \in [d]} K_i = \Theta\Bigg(\frac{\max_{j \in [d]} |z^{(j)}_{0}|}{\min_{j \in [d]} |z^{(j)}_{0}|}\Bigg) = \Theta\Bigg(\frac{\|\bm{z}_{0}\|_\infty}{\min_{j \in [d]} |z^{(j)}_{0}|}\Bigg).
    \end{equation}
    So in other words, $\|\bm{z}_{K'}\|_\infty \leq \min_{j \in [d]} |z^{(j)}_{0}|$ requires at least $K' = \Omega\Big(\frac{\|\bm{z}_{0}\|_\infty}{\min_{j \in [d]} |z^{(j)}_{0}|}\Big)$ iterations. Further, in the best case, $\|\bm{z}_{K'}\|_\infty \leq \min_{j \in [d]} |z^{(j)}_{0}|$ can imply $f(\bm{x}_{K'}) \leq \frac{1}{2}\big(\min_{j \in [d]} |z^{(j)}_{0}|\big)^2 = \frac{1}{2} \big(\min_{j \in [d]} |x^{(j)}_{0} - x^{(j)}_{*}|\big)^2$.
    \\
    Let us now bound $K' = \Omega\Big(\frac{\|\bm{z}_{0}\|_\infty}{\min_{j \in [d]} |z^{(j)}_{0}|}\Big)$ with constant probability. To that end, note that $\frac{\|\bm{z}_{0}\|_\infty}{\min_{j \in [d]} |z^{(j)}_{0}|} \geq  \frac{|z^{(d)}_{0}|}{\min_{j \in [d]} |z^{(j)}_{0}|}$. Next, 
    \begin{flalign}
        \nonumber
        \mathbb{P}\Big(|z^{(d)}_{0}| \geq 2B, \min_{j \in [d]} |z^{(j)}_{0}| \leq \frac{B}{d}\Big) & = \mathbb{P}\Big(|z^{(d)}_{0}| \geq 2B, \min_{j \in [d-1]} |z^{(j)}_{0}| \leq \frac{B}{d}\Big) 
        \\
        \label{eq:93-jan5}
        & = \mathbb{P}\Big(|z^{(d)}_{0}| \geq 2B\Big) \mathbb{P}\Big(\min_{j \in [d-1]} |z^{(j)}_{0}| \leq \frac{B}{d}\Big).
    \end{flalign}
    Similar to \Cref{lem5-jan3}, we have $\mathbb{P}\big(|z^{(d)}_{0}| \geq 2B\big) = \frac{1}{2}$ and $\mathbb{P}\big(\min_{j \in [d-1]} |z^{(j)}_{0}| \leq \frac{B}{d}\big) = 1 - \big(1 - \frac{1}{2 d}\big)^{d-1} \geq 1 - \exp\big(-\frac{1}{2}\big(1 - \frac{1}{d}\big)\big)$. Plugging this into \cref{eq:93-jan5} gives us:
    \begin{equation}
        \mathbb{P}\Big(|z^{(d)}_{0}| \geq 2B, \min_{j \in [d]} |z^{(j)}_{0}| \leq \frac{B}{d}\Big) \geq \frac{1}{2}\Big(1 - \exp\Big(-\frac{1}{2}\Big(1 - \frac{1}{d}\Big)\Big)\Big).
    \end{equation}
    Thus, $\frac{\|\bm{z}_{0}\|_\infty}{\min_{j \in [d]} |z^{(j)}_{0}|} \geq  \frac{|z^{(d)}_{0}|}{\min_{j \in [d]} |z^{(j)}_{0}|} \geq 2d$ with a probability of at least $\frac{1}{2}\Big(1 - \exp\big(-\frac{1}{2}\big(1 - \frac{1}{d}\big)\big)\Big)$. Hence, $K' = \Omega\Big(\frac{\|\bm{z}_{0}\|_\infty}{\min_{j \in [d]} |z^{(j)}_{0}|}\Big) = \Omega(d)$ with constant probability. Finally, note that this analysis holds for any diagonal $\bm{Q}$.
\end{proof}

\section{Proof of Theorem~\ref{prop:bar-kappa}}
\label{prop:bar-kappa-pf}
The proof relies on Gershgorin's theorem which we state next.
\begin{theorem}[\textbf{Gershgorin's Theorem}]
    \label{thm-gershgorin}
    Suppose $\bm{M} \in \mathbb{C}^{n \times n}$. For each $i \in [n]$, define $D_i := \Big\{z \in \mathbb{C}: \big|z - M[i,i]\big| \leq \sum_{j \neq i} \big|M[i,j]\big|\Big\}$. Then each eigenvalue of $\bm{M}$ lies within one of the $D_i$'s.
\end{theorem}
\noindent Now we are ready to prove \Cref{prop:bar-kappa}.
\begin{proof}
    Let $\bm{R} = \bm{D}^{-1} \bm{Q}$. Note that $\bm{R} = \bm{D}^{-1/2} \big(\bm{D}^{-1/2} \bm{Q} \bm{D}^{-1/2}\big) \bm{D}^{1/2}$. Thus, $\bm{R}$ is similar to $\bm{D}^{-1/2} \bm{Q} \bm{D}^{-1/2}$ and so they have the same eigenvalues. Hence, $\mu_1 = \lambda_\text{min}(\bm{R})$ and $\mu_2 = \lambda_\text{max}(\bm{R})$. Note that $R[i,j] = \frac{Q[i,j]}{q_i}$ and in particular, $R[i,i]=1$. Now since $\bm{Q}$ is $\nu$-diagonally-dominant, we have:
    \begin{equation}
        \label{eq:jan1-101}
        \sum_{j \neq i} |R[i,j]| = \sum_{j \neq i} \frac{|Q[i,j]|}{q_i} \leq \nu,
    \end{equation}
    for all $i \in [d]$. Applying Gershgorin's theorem (\Cref{thm-gershgorin}) on $\bm{R}$ and making use of \cref{eq:jan1-101}, we conclude that:
    \begin{equation}
        \mu_1 = \lambda_\text{max}(\bm{R}) \leq (1 + \nu) \text{ and } \mu_2 = \lambda_\text{min}(\bm{R}) \geq (1 - \nu).
    \end{equation}
    Therefore, $$\overline{\kappa} = \frac{\mu_2}{\mu_1} \leq \frac{1 + \nu}{1 - \nu}.$$
    Next, we will apply Gershgorin's theorem on $\bm{Q}$. Specifically, we get:
    \begin{equation}
        \lambda_\text{max}(\bm{Q}) \leq (1 + \nu) q_\text{max} \text{ and } \lambda_\text{min}(\bm{Q}) \geq (1 - \nu) q_\text{min}.
    \end{equation}
    Thus,
    \begin{equation}
        \kappa = \frac{\lambda_\text{max}(\bm{Q})}{\lambda_\text{min}(\bm{Q})} \leq \Big(\frac{1+\nu}{1-\nu}\Big)\Big(\frac{q_\text{max}}{q_\text{min}}\Big) \leq \Big(\frac{1+\nu}{1-\nu}\Big) \kappa_\text{diag}.
    \end{equation}
\end{proof}

\section{Proof of Proposition~\ref{bad-kappa-bar}}
\label{bad-kappa-bar-pf}
\begin{proof}
For 
\begin{equation*}
{\bm{Q}} = \begin{bmatrix}
            b & b-1 \\
            b-1 & b 
            \end{bmatrix},
\end{equation*}
we have $\lambda_\text{max}(\bm{Q}) = 2b-1$ and $\lambda_\text{min}(\bm{Q}) = 1$. So for $b \to \infty$, $\kappa = \mathcal{O}(b)$. Also, it can be verified that:
\begin{equation}
    \bm{D}^{-1/2} \bm{Q} \bm{D}^{-1/2} = \frac{1}{b} \bm{Q}.
\end{equation}
Thus:
\begin{equation}
    \overline{\kappa} = \frac{\lambda_\text{max}(\bm{D}^{-1/2} \bm{Q} \bm{D}^{-1/2})}{\lambda_\text{min}(\bm{D}^{-1/2} \bm{Q} \bm{D}^{-1/2})} = \frac{\lambda_\text{max}(\bm{Q})/b}{\lambda_\text{min}(\bm{Q})/b} = \kappa.
\end{equation}
\end{proof}

\section{Proof of Theorem~\ref{thm-fixed-pt}}
\label{pf-thm-fixed-pt}
\begin{proof}
From \cref{eq:10-nov13}, we have:
\begin{equation}
    \label{eq:47-dec20}
    \bm{x}_{k+1} = \bm{x}_k - \alpha {\bm{V}}_k {\bm{g}}_k,
\end{equation}
where ${\bm{g}}_k = \bm{Q} (\bm{x}_k - \bm{x}_{*})$ is the gradient of $f$ at $\bm{x}_k$, $\bm{V}_k$ is a diagonal matrix whose $i^{\text{th}}$ diagonal element is $\frac{1}{\sqrt{{v}_{k}^{(i)} + \max(({g}_0^{(i)})^2, \phi^2) \delta_k^{2}}}$  with: 
\begin{equation}
    \label{eq:52-nov18}
    {v}_k^{(i)} =  \beta_2 {v}_{k-1}^{(i)} + \big({{g}_k^{(i)}}\big)^2.
\end{equation}
Moreover, let us define $r_i := \max(|{g}_0^{(i)}|, \phi)$ for ease of notation. Then the $i^{\text{th}}$ diagonal element of $\bm{V}_k$ is $\frac{1}{\sqrt{{v}_{k}^{(i)} + r_i^2 \delta_k^{2}}}$.
\\
\\
Subtracting $\bm{x}_{*}$ from both sides of \cref{eq:47-dec20} and pre-multiplying by $\bm{Q}$ on both sides, we get:
\begin{equation}
    \bm{g}_{k+1} = \bm{g}_k - \alpha \bm{Q} {\bm{V}}_k {\bm{g}}_k = \Big(\bm{\text{I}} - \alpha \bm{Q} {\bm{V}}_k \Big) {\bm{g}}_k.
\end{equation}
Define $h({\bm{g}}_k) := f({\bm{Q}}^{-1} {\bm{g}}_k + \bm{x}_{*}) = f(\bm{x}_k)$.  So if $\overline{\bm{g}}$ is a fixed point for $h(.)$, then $\overline{\bm{x}} = {\bm{Q}}^{-1} \overline{\bm{g}} + \bm{x}_{*}$ is a corresponding fixed point for $f(.)$. Note that: $$h({\bm{g}}_k) = \frac{1}{2} {\bm{g}}_k^{\top} ({\bm{Q}}^{-1})^{\top} \bm{Q} ({\bm{Q}}^{-1}) {\bm{g}}_k = \frac{1}{2} {\bm{g}}_k^{\top} {\bm{Q}}^{-1} {\bm{g}}_k.$$ Now if a fixed point exists, it must be true that:
\small
\begin{equation}
    \label{eq:75-oct6}
    \lim_{k \to \infty} (h({\bm{g}}_k) - h({\bm{g}}_{k+1})) = 0 \implies \lim_{k \to \infty} \frac{1}{2} {\bm{g}}_k^{\top} \underbrace{\Big({\bm{Q}}^{-1} - (\bm{\text{I}} - \alpha {\bm{Q}} {\bm{V}}_k)^\top {\bm{Q}}^{-1} (\bm{\text{I}} - \alpha {\bm{Q}} {\bm{V}}_k)\Big)}_{:=\bm{Z}_k} {\bm{g}}_k = 0.
\end{equation}
\normalsize
So we must have $\lim_{k \to \infty} {\bm{g}}_k^{\top} \bm{Z}_k {\bm{g}}_k = 0$. A trivial solution of the above equation is $\lim_{k \to \infty} {\bm{g}}_k = \vec{\bm{0}}$. Thus, $\overline{\bm{g}} = \vec{\bm{0}}$ is a fixed point for $h(.)$; this implies $\overline{\bm{x}} = \bm{x}_{*}$ is a fixed point for $f(.)$. Let us now investigate the non-trivial case of $\lim_{k \to \infty} {\bm{g}}_k \neq \vec{\bm{0}}$. To that end, let us first simplify $\bm{Z}_k$. We have:
\begin{flalign}
    \label{eq:76-oct6}
    \bm{Z}_k &= \alpha(2 {\bm{V}}_k - \alpha {\bm{V}}_k{\bm{Q}}{\bm{V}}_k).
\end{flalign}
Using this in \cref{eq:75-oct6}, we get:
\begin{flalign}
    \label{eq:78-oct6}
    \lim_{k \to \infty} {\bm{g}}_k^{\top} {\bm{V}}_k (2 {\bm{V}}_k^{-1} - \alpha \bm{Q}) {\bm{V}}_k {\bm{g}}_k = 0.
\end{flalign}
Let $\bm{u}_k = {\bm{V}}_k {\bm{g}}_k$. Then, \cref{eq:78-oct6} is equivalent to:
\begin{equation}
    \label{eq:79-oct6}
    2 \lim_{k \to \infty} \bm{u}_k^\top {\bm{V}}_k^{-1} \bm{u}_k = \alpha \lim_{k \to \infty} \bm{u}_k^\top \bm{Q} \bm{u}_k.
\end{equation}
Here, we have used the fact that $\bm{u}_k^\top = {\bm{g}}_k^\top {\bm{V}}_k$ because ${\bm{V}}_k$ is diagonal and hence symmetric.

{Suppose $\lim_{k \to \infty} |{{g}}_k^{(i)}| = \overline{g}^{(i)} \geq 0$. Also, let $\overline{\bm{g}} = [\overline{g}^{(1)}, \ldots, \overline{g}^{(d)}]$.} From \cref{eq:52-nov18}, recall that ${v}_k^{(i)} =  \beta_2 {v}_{k-1}^{(i)} + \big({{g}_k^{(i)}}\big)^2$. Thus, $\lim_{k \to \infty} {v}_k^{(i)} = \frac{(\overline{g}^{(i)})^2}{1 - \beta_2}$. Further, since $\lim_{k \to \infty} \delta_k = \delta$, we have:
\begin{equation}
    \lim_{k \to \infty} {\bm{V}}_k = \begin{bmatrix}
                \frac{1}{\sqrt{\frac{(\overline{g}^{(1)})^2}{1 - \beta_2} + r_1^2 \delta^2}} & 0 & \ldots & 0\\
                0 & \frac{1}{\sqrt{\frac{(\overline{g}^{(2)})^2}{1-\beta_2} + r_2^2 \delta^2}} & \ldots & 0\\
                \vdots & \vdots & \ddots & \vdots\\
                0 & 0 & \ldots & \frac{1}{\sqrt{\frac{(\overline{g}^{(d)})^2}{1-\beta_2} + r_d^2 \delta^2}}
                \end{bmatrix}.
\end{equation}
Thus:
\begin{equation}
    \lim_{k \to \infty} |{u}_k^{(i)}| = \frac{\overline{g}^{(i)}}{\sqrt{\frac{(\overline{g}^{(i)})^2}{1-\beta_2} + r_i^2 \delta^2}}.
\end{equation}
{
Also note that:
\begin{equation*}
    \lim_{k \to \infty} {\bm{V}}_k^{-1} = \begin{bmatrix}
            {\sqrt{\frac{(\overline{g}^{(1)})^2}{1-\beta_2} + r_1^2 \delta^2}} & 0 & \ldots & 0\\
                0 & {\sqrt{\frac{(\overline{g}^{(2)})^2}{1-\beta_2} + r_2^2 \delta^2}} & \ldots & 0\\
                \vdots & \vdots & \ddots & \vdots\\
                0 & 0 & \ldots & {\sqrt{\frac{(\overline{g}^{(d)})^2}{1-\beta_2} + r_d^2 \delta^2}}
                \end{bmatrix}.
\end{equation*}
}
Thus the LHS of \cref{eq:79-oct6} is:
\begin{flalign}
    2 \lim_{k \to \infty} \bm{u}_k^\top {\bm{V}}_k^{-1} \bm{u}_k & = 2 \lim_{k \to \infty} \sum_{i=1}^d |{u}_k^{(i)}|^2 {\sqrt{\frac{(\overline{g}^{(i)})^2}{1-\beta_2} + r_i^2 \delta^2}}
    \\
    \label{eq:84-oct6}
    & = 2 \sum_{i=1}^d \frac{(\overline{g}^{(i)})^2}{\sqrt{\frac{(\overline{g}^{(i)})^2}{1-\beta_2} + r_i^2 \delta^2}}.
\end{flalign}
But the RHS of \cref{eq:79-oct6} satisfies the lower bound:
\begin{equation}
    \label{eq:85-oct6}
    \alpha \lim_{k \to \infty} \bm{u}_k^\top \bm{Q} \bm{u}_k \geq \alpha \lambda_{\min}(\bm{Q}) \lim_{k \to \infty} \|\bm{u}_k\|_2^2 = \alpha \sum_{i=1}^d \frac{(\overline{g}^{(i)})^2}{{\frac{(\overline{g}^{(i)})^2}{1-\beta_2} + r_i^2 \delta^2}},
\end{equation}
where we have used the fact that $\lambda_{\min}(\bm{Q}) = 1$. Combining the results of \cref{eq:84-oct6} and \cref{eq:85-oct6}, we get:
\begin{equation}
    \label{eq:87-oct6}
    \sum_{i=1}^d \frac{(\overline{g}^{(i)})^2}{\sqrt{\frac{(\overline{g}^{(i)})^2}{1-\beta_2} + r_i^2 \delta^2}} \geq \frac{\alpha}{2} \sum_{i=1}^d \frac{(\overline{g}^{(i)})^2}{{\frac{(\overline{g}^{(i)})^2}{1-\beta_2} + r_i^2 \delta^2}}. 
\end{equation}
If $\delta = 0$, \cref{eq:87-oct6} reduces to:
\begin{equation}
    \sum_{i=1}^d |{\overline{g}^{(i)}}| \geq \frac{\alpha d \sqrt{1 - \beta_2}}{2} \implies \|\overline{\bm{g}}\|_1 \geq \frac{\alpha d \sqrt{1 - \beta_2}}{2}.
\end{equation}
Let us now consider the case of $\delta > 0$. Note that:
\begin{equation}
    \sum_{i=1}^d \frac{(\overline{g}^{(i)})^2}{\sqrt{\frac{(\overline{g}^{(i)})^2}{1 - \beta_2} + r_i^2 \delta^2}} \leq \Bigg(\max_{j \in [d]} \sqrt{\frac{(\overline{g}^{(j)})^2}{1-\beta_2} + r_j^2 \delta^2} \Bigg) \Bigg(\sum_{i=1}^d \frac{(\overline{g}^{(i)})^2}{{\frac{(\overline{g}^{(i)})^2}{1-\beta_2} + r_i^2 \delta^2}}\Bigg).
\end{equation}
Using this in \cref{eq:87-oct6}, we get:
\begin{equation}
    \max_{j \in [d]} \sqrt{\frac{(\overline{g}^{(j)})^2}{1-\beta_2} + r_j^2 \delta^2} \geq \frac{\alpha}{2}.
\end{equation}
Let $\overline{R} := \max_{j \in [d]} r_j$. Now if $\delta < \frac{\alpha}{2 \overline{R}}$, then:
\begin{equation}
    \|\overline{\bm{g}}\|_\infty \geq \sqrt{(1 - \beta_2)\Big(\frac{\alpha^2}{4} - \overline{R}^2 \delta^2\Big)}. 
\end{equation}
{
Also, the RHS of \cref{eq:79-oct6} satisfies the upper bound:
\begin{equation}
    \label{eq:97-oct14}
    \alpha \lim_{k \to \infty} \bm{u}_k^\top \bm{Q} \bm{u}_k \leq \alpha \lambda_{\max}(\bm{Q}) \lim_{k \to \infty} \|\bm{u}_k\|_2^2 = \kappa \alpha \sum_{i=1}^d \frac{(\overline{g}^{(i)})^2}{{\frac{(\overline{g}^{(i)})^2}{1-\beta_2} + r_i^2 \delta^2}},
\end{equation}
where we have used the fact that $\lambda_{\max}(\bm{Q}) = \kappa$. Combining the results of \cref{eq:84-oct6} and \cref{eq:97-oct14}, we get:
\begin{equation}
    \label{eq:99-oct14}
    \sum_{i=1}^d \frac{(\overline{g}^{(i)})^2}{\sqrt{\frac{(\overline{g}^{(i)})^2}{1-\beta_2} + r_i^2 \delta^2}} \leq \frac{\kappa \alpha}{2} \sum_{i=1}^d \frac{(\overline{g}^{(i)})^2}{{\frac{(\overline{g}^{(i)})^2}{1-\beta_2} + r_i^2 \delta^2}}.
\end{equation}
Further, note that:
\begin{equation}
    \sum_{i=1}^d \frac{(\overline{g}^{(i)})^2}{\sqrt{\frac{(\overline{g}^{(i)})^2}{1-\beta_2} + r_i^2 \delta^2}} \geq \Bigg(\min_{j \in [d]} \sqrt{\frac{(\overline{g}^{(j)})^2}{1-\beta_2} + r_j^2 \delta^2} \Bigg) \Bigg(\sum_{i=1}^d \frac{(\overline{g}^{(i)})^2}{{\frac{(\overline{g}^{(i)})^2}{1-\beta_2} + r_i^2 \delta^2}}\Bigg).
\end{equation}
Using this in \cref{eq:99-oct14}, we get:
\begin{equation}
    \min_{j \in [d]} \sqrt{\frac{(\overline{g}^{(j)})^2}{1-\beta_2} + r_j^2 \delta^2} \leq \frac{\kappa \alpha}{2}.
\end{equation}
Let $\overline{r} := \min_{j \in [d]} r_j$. Now if $\delta > \frac{\kappa \alpha}{2 \overline{r}}$, the above equation is not possible and hence the case of $\lim_{k \to \infty} {\bm{g}}_k \neq \vec{\bm{0}}$ is not possible.
}
\end{proof}

\section{Detailed Version and Proof of Theorem~\ref{thm-adanorm-plc}}
\label{sec-smooth-plc-pf}
\begin{theorem}[\textbf{Adam: Per-Coordinate Smooth and PL}]
    \label{thm-adanorm-plc-full}
    Suppose our initialization is $\bm{x}_0$ and Assumptions \ref{asmp-3-jan19} and \ref{lower-bound} hold. Let $\Delta_0 := f(\bm{x}_0) - f^{*}$. Choose
    $$\alpha = \mathcal{O}\Bigg(\Bigg(\max\Bigg(\frac{1}{\kappa_{\textup{max},0}}, \frac{L_\textup{min}}{L_\textup{max}}\Bigg) \frac{\Delta_0^2}{\tilde{\mu} L_\textup{max}}\Bigg)^{1/4} \log^{-1/2}\Big(\frac{\Delta_0}{\varepsilon}\Big)\Bigg), \text{ } \beta_2 = 1 - \frac{\alpha^2 \tilde{\mu}}{16 \Delta_0},$$
    $$\delta_k = \frac{\beta_2^{{k}/{2}}}{\sqrt{2(1-\beta_2)}} \text{ and } \phi = \sqrt{2 L_\textup{min} \Delta_0}.$$
    Then $f(\bm{x}_{{K}^{*}}) - f^{*} \leq \varepsilon$ in 
    \begin{equation*}
        K^{*} = \mathcal{O}\Bigg(\sqrt{\frac{L_\textup{max}}{\tilde{\mu}}} \min \Bigg(\sqrt{\frac{L_\textup{max}}{L_\textup{min}}},\sqrt{{\kappa_{\textup{max},0}}}\Bigg) \log\Big(\frac{\Delta_0}{\varepsilon}\Big) {\log\Big(\log\Big(\frac{\Delta_0}{\varepsilon}\Big)\Big)}\Bigg)
    \end{equation*}
    iterations of {Adam}.
\end{theorem}

\begin{proof}
    For brevity, let $\bm{g}_k := \nabla f(\bm{x}_k)$. Since $f$ is $\{L_i\}_{i=1}^d$-smooth, we get using \Cref{lem11-jan18}:
    \begin{equation*}
        f(\bm{x}_{k+1}) \leq f(\bm{x}) - \alpha \sum_{i=1}^d \frac{(g_k^{(i)})^2}{\sqrt{v_k^{(i)} + \max\big(\big({g}_0^{(i)}\big)^2, \phi^2\big) \delta_{k}^{2}}} + \frac{\alpha^2}{2} \sum_{i=1}^d L_i \frac{(g_k^{(i)})^2}{{v_k^{(i)} + \max\big(\big({g}_0^{(i)}\big)^2, \phi^2\big) \delta_{k}^{2}}}.
    \end{equation*}
    Subtracting $f^{*}$ from both sides above, we get:
    \begin{multline}
        f(\bm{x}_{k+1}) - f^{*} \leq f(\bm{x}_k) - f^{*} \\ - \alpha \sum_{i=1}^d \frac{(g_k^{(i)})^2}{\sqrt{v_k^{(i)} + \max\big(\big({g}_0^{(i)}\big)^2, \phi^2\big) \delta_{k}^{2}}} \Bigg(1 - \frac{\alpha L_i}{2 \sqrt{v_k^{(i)} + \max\big(\big({g}_0^{(i)}\big)^2, \phi^2\big) \delta_{k}^{2}}}\Bigg).
    \end{multline}
    Now until $1 - \frac{\alpha L_i}{2 \sqrt{v_k^{(i)} + \max\big(\big({g}_0^{(i)}\big)^2, \phi^2\big) \delta_{k}^{2}}} \geq \frac{1}{2}$ for each $i \in [d]$, i.e.,
    \begin{equation}
        \label{eq:83-dec15-2}
        v_k^{(i)} + \max\big(\big({g}_0^{(i)}\big)^2, \phi^2\big) \delta_{k}^{2} \geq \alpha^2 L_i^2 \text{ for each } i \in [d],
    \end{equation}
    we have:
    \begin{flalign}
        \label{eq:84-dec15-old}
        f(\bm{x}_{k+1}) - f^{*} & \leq f(\bm{x}_k) - f^{*} - \frac{\alpha}{2} \sum_{i=1}^d \frac{(g_k^{(i)})^2}{\sqrt{v_k^{(i)} + \max\big(\big({g}_0^{(i)}\big)^2, \phi^2\big) \delta_{k}^{2}}}.
    \end{flalign}
    For ease of notation, let us define $\Delta_k := f(\bm{x}_k) - f^{*}$. With this, the above becomes:
    \begin{flalign}
        \label{eq:84-dec15}
        \Delta_{k+1} & \leq \Delta_k - \frac{\alpha}{2} \sum_{i=1}^d \frac{(g_k^{(i)})^2}{\sqrt{v_k^{(i)} + \max\big(\big({g}_0^{(i)}\big)^2, \phi^2\big) \delta_{k}^{2}}}.
    \end{flalign}
    Clearly, $v_k^{(i)} + \max\big(\big({g}_0^{(i)}\big)^2, \phi^2\big) \delta_{k}^{2} \geq \alpha^2 L_i^2$ holds for all $i \in [d]$ until at least: 
    \begin{equation}
        \label{eq:188-jan18}
        \max\big(\big({g}_0^{(i)}\big)^2, \phi^2\big) \delta_{k}^{2} \geq \alpha^2 L_i^2 \text{ } \forall \text{ } i \in [d].
    \end{equation}
    Next, recall that $v_k^{(i)} =  \sum_{l=0}^{k} \beta_2^{k-l} (g_l^{(i)})^2$. Also from \Cref{lem12-jan18}, we have $(g_l^{(i)})^2 \leq 2 L_i(f(\bm{x}_l) - f^{*}) = 2 L_i \Delta_l$. Using this, we get:
    \begin{equation}
        v_k^{(i)} \leq 2 L_i \sum_{l=0}^{k} \beta_2^{k-l} \Delta_l.
    \end{equation}
    From \Cref{lem12-jan18}, we also have $(g_0^{(i)})^2 \leq 2 L_i(f(\bm{x}_0) - \bar{f}^{*}_i(\bm{x}_0))$ where $\bar{f}^{*}_i(\bm{x}_0) = {\min}_{h \in \mathbb{R}} \text{ } f(\bm{x}_0 + h \bm{e}_i)$. Let $\bar{\Delta}_{i, 0} := f(\bm{x}_0) - \bar{f}^{*}_i(\bm{x}_0)$; note that $\bar{\Delta}_{i, 0} \leq \Delta_0$ for all $i \in [d]$. Then, 
    \begin{equation}
        \label{eq:191-jan19}
        (g_0^{(i)})^2 \leq 2 L_i \bar{\Delta}_{i, 0}.
    \end{equation}
    Plugging all of this into \cref{eq:84-dec15}, we get:
    \begin{equation}
        \Delta_{k+1} \leq \Delta_{k} - \sum_{i=1}^d \Bigg(\frac{\alpha}{2 \sqrt{\sum_{l=0}^{k} \beta_2^{k-l} \Delta_l + \max\big(\bar{\Delta}_{i, 0}, \frac{\phi^2}{2 L_i}\big) \delta_{k}^{2}}}\Bigg) \frac{(g_k^{(i)})^2}{\sqrt{2 L_i}}.
    \end{equation}
    Let $R = \max_{i \in [d]}\frac{\phi^2}{2 L_i} = \frac{\phi^2}{2 L_\text{min}}$ and $\bar{\Delta}^{*}_{0} = \max_{i \in [d]} \bar{\Delta}_{i, 0}$. Then the above can be bounded as:
    \begin{equation}
        \Delta_{k+1} \leq \Delta_{k} - \Bigg(\frac{\alpha}{2 \sqrt{\sum_{l=0}^{k} \beta_2^{k-l} \Delta_{l} + \max\big(\bar{\Delta}^{*}_{0}, R\big) \delta_{k}^{2}}}\Bigg) \sum_{i=1}^d \frac{(g_k^{(i)})^2}{\sqrt{2 L_i}}.
    \end{equation}
    Next, using the fact that $f$ satisfies the smoothness-dependent $\tilde{\mu}$-global-PL condition above, we get:
    \begin{flalign}
        \label{eq:86-dec15}
        \Delta_{k+1} \leq \Bigg(1 - \frac{\alpha \sqrt{({\tilde{\mu}}/{2})}}{\sqrt{{\sum_{l=0}^{k} \beta_2^{k-l} \Delta_l + \max\big(\bar{\Delta}^{*}_{0}, R\big) \delta_{k}^{2}}}}\Bigg)\Delta_k.
    \end{flalign}
    Now note that we will have descent in each iteration or {continuous descent} (similar to Theorems \ref{thm3-diag} and \ref{thm3}) until at least $\sqrt{\sum_{l=0}^{k} \beta_2^{k-l} \Delta_l + \max\big(\bar{\Delta}^{*}_{0}, R\big) \delta_{k}^{2}} \geq \alpha \sqrt{\frac{\tilde{\mu}}{2}}$. This will hold until at least:
    \begin{equation}
        \label{eq:89-dec15}
        \max\big(\bar{\Delta}^{*}_{0}, R\big) \delta_{k}^2 \geq \alpha^2 \Big(\frac{\tilde{\mu}}{2}\Big).
    \end{equation}
    But recall our earlier constraint in \cref{eq:188-jan18}, i.e., $\max\big(\big({g}_0^{(i)}\big)^2, \phi^2\big) \delta_{k}^{2} \geq \alpha^2 L_i^2$ for all $i \in [d]$. We shall show that \cref{eq:188-jan18} is stricter than the condition in \cref{eq:89-dec15}. Note that \cref{eq:188-jan18} can be rewritten as:
    \begin{equation}
        \label{eq:197-jan19}
        \alpha^2 \leq \max\Bigg(\frac{\big({g}_0^{(i)}\big)^2}{L_i^2}, \frac{\phi^2}{L_i^2}\Bigg) \delta_{k}^{2} \text{ } \forall \text{ } i \in [d].
    \end{equation}
    On the other hand, recalling that $R = \frac{\phi^2}{2 L_\text{min}}$, \cref{eq:89-dec15} can be rewritten as:
    \begin{equation}
        \label{eq:198-jan19}
        \alpha^2 \leq \Bigg(\frac{2 \bar{\Delta}^{*}_{0}}{\tilde{\mu}}, \frac{\phi^2}{\tilde{\mu} L_\text{min}}\Bigg) \delta_{k}^{2}.
    \end{equation}
    Using \cref{eq:191-jan19}, we get for any $i \in [d]$:
    \begin{equation}
        \label{eq:199-jan19}
        \frac{\big({g}_0^{(i)}\big)^2}{L_i^2} \leq \frac{2 \bar{\Delta}_{i, 0}}{L_i} \leq \frac{2 \bar{\Delta}^{*}_{0}}{\tilde{\mu}},
    \end{equation}
    where the last inequality follows because $\bar{\Delta}^{*}_{0} = \max_{i \in [d]} \bar{\Delta}_{i, 0}$ and $\tilde{\mu} \leq L_i$ $\forall$ $i \in [d]$. Also, 
    \begin{equation}
        \frac{\phi^2}{L_i^2} \leq \frac{\phi^2}{\tilde{\mu} L_\text{min}},
    \end{equation}
    because $\tilde{\mu} \leq L_\text{min} \leq L_i$ $\forall$ $i \in [d]$. Thus, \cref{eq:197-jan19} is stricter than \cref{eq:198-jan19} implying that \cref{eq:188-jan18} is stricter than \cref{eq:89-dec15}.
    \\
    \\
    For concise notation, let us define $\widehat{v}_k := \sum_{l=0}^{k} \beta_2^{k-l} \Delta_l$. With that, we have (from \cref{eq:86-dec15}):
    \begin{equation}
        \label{eq:90-dec15}
        \Delta_{k+1} \leq \Bigg(1 - \frac{\alpha \sqrt{({\tilde{\mu}}/{2})}}{\sqrt{{\widehat{v}_k + \max\big(\bar{\Delta}^{*}_{0}, R\big) \delta_{k}^{2}}}}\Bigg)\Delta_k,
    \end{equation}
    until at least $\max\big(\big({g}_0^{(i)}\big)^2, \phi^2\big) \delta_{k}^{2} \geq \alpha^2 L_i^2$ for all $i \in [d]$; note that this can be rewritten as:
    \begin{equation}
        \alpha^2 \leq \max\Bigg(\min_{i \in [d]}\frac{\big({g}_0^{(i)}\big)^2}{L_i^2}, \min_{i \in [d]} \frac{\phi^2}{L_i^2}\Bigg) \delta_{k}^{2} = \max\Bigg(\min_{i \in [d]}\frac{\big({g}_0^{(i)}\big)^2}{L_i^2}, \frac{\phi^2}{L_\text{max}^2}\Bigg) \delta_{k}^{2}.
    \end{equation}
    Suppose the above holds for the first $\tilde{K}$ iterations. 
    Let us choose $\phi \leq \sqrt{2 L_\textup{min} \Delta_0}$, $\alpha \leq \sqrt{\frac{\Delta_0}{\tilde{\mu}}}$, $\beta_2 = 1 - \frac{\alpha^2 \tilde{\mu}}{16 \Delta_0}$ and $\delta_k = \frac{1}{\sqrt{2(1-\beta_2)}} \beta_2^{\frac{k}{2}}$. Then from \Cref{lem2-dec15},
    \begin{equation}
        \tilde{K} = \frac{\log\Big(\max\Big(\min_{i \in [d]}\frac{({g}_0^{(i)})^2}{L_i^2}, \frac{\phi^2}{L_\textup{max}^2}\Big) \frac{1}{2 (1-\beta_2) \alpha^2}\Big)}{\log \frac{1}{\beta_2}} = \frac{\log\Big(\max\Big(\min_{i \in [d]}\frac{({g}_0^{(i)})^2}{L_i^2}, \frac{\phi^2}{L_\textup{max}^2}\Big) \frac{8 \Delta_0}{\alpha^4 \tilde{\mu}}\Big)}{\log \frac{1}{\beta_2}},
    \end{equation}
    where the last step follows by plugging in $\beta_2 = 1 - \frac{\alpha^2 \tilde{\mu}}{16 \Delta_0}$.
    Also from \Cref{lem1-dec15}, we have for any $k \leq \tilde{K}$:
    \begin{equation}
        \Delta_{k} \leq \exp\Bigg(-2\Big(\beta_2^{-\frac{k}{2}} - 1\Big)\Bigg) \Delta_{0}.
    \end{equation}
    Now in order to have $\Delta_{K^{*}} \leq \varepsilon$, we can choose: 
    \begin{equation}
        \label{eq:198-jan18}
        K^{*} = \frac{\log\big(\frac{1}{4}\log^2\big(\frac{e^2 \Delta_{0}}{\varepsilon}\big)\big)}{\log \frac{1}{\beta_2}}.
    \end{equation}
    Now in order for our result to be applicable, let us ensure $K^{*} \leq \tilde{K}$; this yields:
    \begin{equation}
        \label{eq:94-dec15-2}
        \alpha \leq \mathcal{O}\Bigg(\Bigg(\max\Bigg(\min_{i \in [d]}\frac{({g}_0^{(i)})^2}{L_i^2}, \frac{\phi^2}{L_\textup{max}^2}\Bigg) \frac{\Delta_0}{\tilde{\mu}}\Bigg)^{1/4} \log^{-1/2}\Big(\frac{\Delta_0}{\varepsilon}\Big)\Bigg).
    \end{equation}
    Per \cref{eq:230-jan19} (in the proof of \Cref{lem2-dec15}), $\max\Big(\min_{i \in [d]}\frac{({g}_0^{(i)})^2}{L_i^2}, \frac{\phi^2}{L_\textup{max}^2}\Big) \leq \frac{2 \Delta_0}{\tilde{\mu}}$. Hence, the constraint in \cref{eq:94-dec15-2} satisfies our earlier condition on $\alpha$, namely $\alpha \leq \sqrt{\frac{\Delta_0}{\tilde{\mu}}}$. From \Cref{lower-bound}, $({g}_0^{(i)})^2 \geq 2 \mu_{i,0} {\Delta}_{0}$. Using this in \cref{eq:94-dec15-2} and choosing $\phi = \sqrt{2 L_\textup{min} \Delta_0}$, the following is a valid range for $\alpha$:
    \begin{equation}
        \alpha \leq \mathcal{O}\Bigg(\Bigg(\max\Bigg(\min_{i \in [d]}\frac{\mu_{i,0}}{L_i^2}, \frac{L_\text{min}}{L_\textup{max}^2}\Bigg) \frac{\Delta_0^2}{\tilde{\mu}}\Bigg)^{1/4} \log^{-1/2}\Big(\frac{\Delta_0}{\varepsilon}\Big)\Bigg).
    \end{equation}
    Further, using \Cref{lower-bound}, we have that $\frac{\mu_{i,0}}{L_i} \geq \frac{1}{{\kappa}_{\textup{max},0}}$ for all $i \in [d]$. 
    Using this above along with the fact that $\min_{i \in [d]} \frac{1}{L_i} = \frac{1}{L_\text{max}}$, we can pick the following value of $\alpha$:
    \begin{equation}
        \label{eq:208-jan20}
        \alpha = \mathcal{O}\Bigg(\Bigg(\max\Bigg(\frac{1}{\kappa_{\text{max},0}}, \frac{L_\text{min}}{L_\textup{max}}\Bigg) \frac{\Delta_0^2}{\tilde{\mu} L_\textup{max}}\Bigg)^{1/4} \log^{-1/2}\Big(\frac{\Delta_0}{\varepsilon}\Big)\Bigg).
    \end{equation}
    Next, recalling that $\beta_2 = 1 - \frac{\alpha^2 \tilde{\mu}}{16 \Delta_0}$, $\log \frac{1}{\beta_2} \geq \frac{\alpha^2 \tilde{\mu}}{16 \Delta_0}$ using \Cref{log-fact}. Plugging this in the expression of $K^{*}$ (\cref{eq:198-jan18}), we get:
    \begin{equation}
        K^{*} = \mathcal{O}\Bigg(\frac{\Delta_0}{\alpha^2 \tilde{\mu}}{\log\Big(\log\Big(\frac{\Delta_0}{\varepsilon}\Big)\Big)}\Bigg).
    \end{equation}
    Plugging in the value of $\alpha$ from \cref{eq:208-jan20}, we get:
    \begin{equation}
        \label{eq:96-dec15-2}
        K^{*} = \mathcal{O}\Bigg(\sqrt{\frac{L_\text{max}}{\tilde{\mu}}} \min \Bigg(\sqrt{{\kappa_{\text{max},0}}}, \sqrt{\frac{L_\text{max}}{L_\text{min}}}\Bigg) \log\Big(\frac{\Delta_0}{\varepsilon}\Big) {\log\Big(\log\Big(\frac{\Delta_0}{\varepsilon}\Big)\Big)}\Bigg).
    \end{equation}
\end{proof}

\begin{lemma}[Similar to \Cref{lem1-sep30}]
\label{lem1-dec15}
Suppose \cref{eq:90-dec15} holds: $$\Delta_{k+1} \leq \Bigg(1 - \frac{\alpha \sqrt{({\tilde{\mu}}/{2})}}{\sqrt{{\widehat{v}_k + \max\big(\bar{\Delta}^{*}_{0}, R\big)  \delta_{k}^{2}}}}\Bigg)\Delta_k,$$
for $k = 0$ through to $k = \tilde{K}$, where $\alpha^2 \leq \max\Big(\min_{i \in [d]}\frac{({g}_0^{(i)})^2}{L_i^2}, \frac{\phi^2}{L_\textup{max}^2}\Big) \delta_{k}^{2}$ for all $k \in \{0,\ldots,\tilde{K}\}$. 
Suppose we set $\phi \leq \sqrt{2 L_\textup{min} \Delta_0}$, $\alpha \leq \sqrt{\frac{\Delta_0}{\tilde{\mu}}}$, $\beta_2 = 1 - \frac{\alpha^2 \tilde{\mu}}{16 \Delta_0}$ and $\delta_k = \frac{1}{\sqrt{2(1-\beta_2)}} \beta_2^{\frac{k}{2}}$ for $k \geq 0$. Then for any $k \leq \tilde{K}$, we have:
\begin{equation*}
    \Delta_{k+1} \leq \exp\Bigg(-2\Big(\beta_2^{-(k+1)/2} - 1\Big)\Bigg) \Delta_{0}.
\end{equation*}
\end{lemma}

\begin{proof}
    Similar to \Cref{lem1-sep30}, we perform the analysis in two stages.
    \\
    \\
    \textbf{First Stage.} From \cref{eq:90-dec15}, $\Delta_0 \geq \Delta_1 \geq \ldots \geq \Delta_{\tilde{K}+1}$. So for any $k \leq \tilde{K}$:
    \begin{equation}
    \widehat{v}_k = {\sum_{l=0}^k \beta_2^{k-l} \Delta_l} \leq {\sum_{l=0}^k \beta_2^{k-l} \Delta_0} \leq \frac{\Delta_0}{1-\beta_2}.
    \end{equation}
    Further, we have chosen the ${\delta}_k$'s so that ${\delta}_k \leq {\delta}_0 \leq \frac{1}{\sqrt{1-\beta_2}}$ for all $k \geq 1$. Also, with our restriction on $\phi$, we have that $R = \frac{\phi^2}{2 L_\text{min}} \leq \Delta_0$. Moreover, $\bar{\Delta}^{*}_{0} \leq \Delta_0$ by definition.
    Thus:
    \begin{equation}
        \max\big(\bar{\Delta}^{*}_{0}, R\big) \leq \Delta_0,
    \end{equation}
    and hence
    \begin{equation}
        \widehat{v}_k + \max\big(\bar{\Delta}^{*}_{0}, R\big) \delta_{k}^{2} \leq \widehat{v}_k + \Delta_0 \delta_{k}^{2} \leq \frac{2 \Delta_0}{1-\beta_2}.
    \end{equation}
     Plugging this into \cref{eq:90-dec15} gives us:
    \begin{equation}
        \label{eq:204-jan18}
        \Delta_{k+1} \leq \Bigg(1 - \frac{{\alpha} \sqrt{\tilde{\mu} (1-\beta_2)}}{2 \sqrt{\Delta_0}}\Bigg) \Delta_{k}.
    \end{equation}
    Define $\gamma := \Big(1 - \frac{{\alpha} \sqrt{\tilde{\mu} (1-\beta_2)}}{2 \sqrt{\Delta_0}}\Big)$. Unfolding the recursion in \cref{eq:204-jan18} yields:
    \begin{equation}
        \label{eq:94-dec15}
        \Delta_{k+1} \leq \gamma^{k+1} \Delta_0.
    \end{equation}
    This bound is loose but we shall obtain a tighter bound using this bound. 
    \\
    \\
    \noindent \textbf{Second Stage.} With our choice of $\beta_2 = 1 - \frac{\alpha^2 \tilde{\mu}}{16 \Delta_0}$, we have $\gamma = 1 - \frac{\alpha^2 \tilde{\mu}}{8 \Delta_0}$ and $\beta_2 - \gamma = {1 - \beta_2}$. Now for any $k \leq \tilde{K}$, we have:
    \begin{flalign}
    \widehat{v}_k = \sum_{l=0}^{k} \beta_2^{k-l} \Delta_l & \leq \sum_{l=0}^{k} \beta_2^{k-l} \gamma^l \Delta_0
    \\
    & = \Bigg(\frac{\beta_2^{k+1} - \gamma^{k+1}}{\beta_2 - \gamma}\Bigg) \Delta_0
    \\
    & \leq \Big(\frac{\beta_2^{k+1}}{1-\beta_2}\Big) \Delta_0,
    \end{flalign}
    where the last step follows from the fact that $\beta_2 - \gamma = {1 - \beta_2}$. Note that with our constraint of $\alpha \leq \sqrt{\frac{\Delta_0}{\tilde{\mu}}}$, we have:
    \begin{equation}
        \label{eq:205-jan18-new}
        \beta_2 \geq \frac{1}{2}.
    \end{equation}
    Thus, we have $\delta_k = \frac{1}{\sqrt{2(1-\beta_2)}} \beta_2^{\frac{k}{2}} \leq \frac{1}{\sqrt{1-\beta_2}}\beta_2^{\frac{k+1}{2}} $ for all $k$. Hence:
    \begin{equation}
        \widehat{v}_k + \max\big(\bar{\Delta}^{*}_{0}, R\big) \delta_{k}^{2} \leq \widehat{v}_k + \Delta_0 \delta_{k}^{2} \leq \Big(\frac{2 \beta_2^{k+1}}{1-\beta_2}\Big) \Delta_0.
    \end{equation}
    Plugging this into \cref{eq:90-dec15}, we get:
    \begin{flalign}
        \label{eq:126-dec23}
        \Delta_{k+1} & \leq \Bigg(1 - \frac{\alpha \sqrt{\tilde{\mu} (1-\beta_2)} \beta_2^{-{(k+1)/2}}}{2 \sqrt{\Delta_0}}\Bigg) \Delta_k.
    \end{flalign}
    Our condition in \Cref{lem2-dec15} will ensure $\frac{\alpha \sqrt{\tilde{\mu} (1-\beta_2)} \beta_2^{-{(k+1)/2}}}{2 \sqrt{\Delta_0}} \leq 1$ for all $k \leq \tilde{K}$; refer to its statement. Next, using the fact that $\exp(-z) \geq 1-z$ for $z \in \mathbb{R}$ above, we get:
    \begin{equation}
        \label{eq:100-dec15}
        \Delta_{k+1} \leq \exp\Bigg(-\frac{\alpha \sqrt{\tilde{\mu} (1-\beta_2)} \beta_2^{-{(k+1)/2}}}{2 \sqrt{\Delta_0}}\Bigg) \Delta_k.
    \end{equation}
    Unfolding the recursion in \cref{eq:100-dec15} gives us:
    \begin{flalign}
    \Delta_{k+1}  & \leq \exp\Bigg(- \frac{\alpha \sqrt{\tilde{\mu}(1-\beta_2)} \beta_2^{-1/2}}{2 \sqrt{\Delta_0}}
    \sum_{l=0}^{k}\big(\beta_2^{-1/2}\big)^l\Bigg)\Delta_{0}
    \\
    & = \exp\Bigg(-\frac{\alpha \sqrt{\tilde{\mu}(1-\beta_2)}}{2 \sqrt{\Delta_0}} \Bigg(\frac{\beta_2^{-(k+1)/2} - 1}{1 - \sqrt{\beta_2}}\Bigg)\Bigg)\Delta_0.
    \end{flalign}
    Recall that $\beta_2 = 1 - \frac{\alpha^2 \tilde{\mu}}{16 \Delta_0}$. Using the fact that $\sqrt{\beta_2} \geq \beta_2$ above, we get:
    \begin{equation}
    \label{eq:130-dec23}
    \Delta_{k+1} \leq \exp\Bigg(-2\Big(\beta_2^{-(k+1)/2} - 1\Big)\Bigg)\Delta_{0}.
    \end{equation}
\end{proof}

\begin{lemma}[{Similar to \Cref{lem1-oct1}}]
    \label{lem2-dec15}
    In the setting of \Cref{lem1-dec15}, $\alpha^2 \leq \max\Big(\min_{i \in [d]}\frac{({g}_0^{(i)})^2}{L_i^2}, \frac{\phi^2}{L_\textup{max}^2}\Big) \delta_{k}^{2}$ until: $$\tilde{K} = \frac{\log\Big(\max\Big(\min_{i \in [d]}\frac{({g}_0^{(i)})^2}{L_i^2}, \frac{\phi^2}{L_\textup{max}^2}\Big) \frac{1}{2 (1-\beta_2) \alpha^2}\Big)}{\log \frac{1}{\beta_2}}$$
    iterations. Further, this value of $\tilde{K}$ also ensures the condition after \cref{eq:126-dec23} in the proof of \Cref{lem1-dec15}, i.e.,
    \begin{equation}
        \label{eq:215-jan18}
        \frac{\alpha \sqrt{\tilde{\mu} (1-\beta_2)} \beta_2^{-{(k+1)/2}}}{2 \sqrt{\Delta_0}} \leq 1 \text{ } \forall \text{ } k \leq \tilde{K},
    \end{equation}
    also holds.
\end{lemma}
\begin{proof}
Plugging in our choice of $\delta_k$ from \Cref{lem1-dec15} into our requirement of $$\alpha^2 \leq \max\Big(\min_{i \in [d]}\frac{({g}_0^{(i)})^2}{L_i^2}, \frac{\phi^2}{L_\textup{max}^2}\Big) \delta_{k}^{2},$$ we conclude that as long as 
\begin{equation}
    \label{eq:130-dec23-new}
    \max\Bigg(\min_{i \in [d]}\frac{({g}_0^{(i)})^2}{L_i^2}, \frac{\phi^2}{L_\textup{max}^2}\Bigg) \Big(\frac{\beta_2^{{k}}}{{2(1-\beta_2)}}\Big) \geq {\alpha^2},
\end{equation}
we are good. 
Let us verify that \cref{eq:130-dec23-new} ensures the condition in \cref{eq:215-jan18}. To that end, note that \cref{eq:215-jan18} is equivalent to:
\begin{equation}
    \label{eq:217-jan18}
    \frac{4 \Delta_0 \beta_2^{k+1}}{\tilde{\mu} (1-\beta_2)} \geq \alpha^2.
\end{equation}
But from \cref{eq:205-jan18-new}, we have $\beta_2 \geq \frac{1}{2}$ and therefore, the above is ensured by:
\begin{equation}
    \label{eq:132-dec23}
    \frac{2 \Delta_0 \beta_2^{k}}{\tilde{\mu} (1-\beta_2)} \geq \alpha^2.
\end{equation}
From \cref{eq:199-jan19}, recall that $\frac{({g}_0^{(i)})^2}{L_i^2} \leq \frac{2 \bar{\Delta}^{*}_{0}}{\tilde{\mu}}$ for any $i \in [d]$. Since $\bar{\Delta}^{*}_{0} \leq \Delta_0$, we have:
\begin{equation}
    \min_{i \in [d]} \frac{\big({g}_0^{(i)}\big)^2}{L_i^2} \leq \frac{2 \Delta_{0}}{\tilde{\mu}}.
\end{equation}
Also, with our restriction of $\phi \leq \sqrt{2 L_\textup{min} \Delta_0}$ from \Cref{lem1-dec15}, we have:
\begin{equation}
    \frac{\phi^2}{L_\textup{max}^2} \leq \frac{2 L_\textup{min} \Delta_0}{L_\textup{max}^2} \leq \frac{2 \Delta_0}{L_\textup{max}} \leq \frac{2 \Delta_0}{\tilde{\mu}}.
\end{equation}
Therefore,
\begin{equation}
    \label{eq:230-jan19}
    \max\Bigg(\min_{i \in [d]}\frac{({g}_0^{(i)})^2}{L_i^2}, \frac{\phi^2}{L_\textup{max}^2}\Bigg) \leq \frac{2 \Delta_0}{\tilde{\mu}}.
\end{equation}
and so:
\begin{equation}
    \max\Bigg(\min_{i \in [d]}\frac{({g}_0^{(i)})^2}{L_i^2}, \frac{\phi^2}{L_\textup{max}^2}\Bigg) \Big(\frac{\beta_2^{{k}}}{{2(1-\beta_2)}}\Big) \leq \frac{\Delta_0 \beta_2^k}{\tilde{\mu} (1-\beta_2)}.
\end{equation}
Hence, the condition in \cref{eq:130-dec23-new} is stricter than the one in \cref{eq:132-dec23}. In other words, the condition in \cref{eq:130-dec23-new} ensures \cref{eq:215-jan18}. 
Finally, note that \cref{eq:130-dec23-new} holds till
\begin{equation}
    \tilde{K} = \frac{\log\Big(\max\Big(\min_{i \in [d]}\frac{({g}_0^{(i)})^2}{L_i^2}, \frac{\phi^2}{L_\textup{max}^2}\Big) \frac{1}{2 (1-\beta_2) \alpha^2}\Big)}{\log \frac{1}{\beta_2}}
\end{equation}
iterations.
\end{proof}

\begin{lemma}[\textbf{\enquote{Descent Lemma} for Per-Coordinate Smoothness}]
    \label{lem11-jan18}
    Suppose $f$ is $\{L_i\}_{i=1}^d$-smooth. Then for any $\bm{x}, \bm{y} \in \mathbb{R}^d$, we have:
    \begin{equation*}
        f(\bm{y}) \leq f(\bm{x}) + \nabla f(\bm{x})^\top (\bm{y} - \bm{x}) + \frac{1}{2}\sum_{i=1}^d {L_i} \big({y}^{(i)} - {x}^{(i)}\big)^2.
    \end{equation*}
\end{lemma}
\begin{proof}
    Let $\bm{x}_\tau = \bm{x} + \tau(\bm{y} - \bm{x})$ for $\tau \in [0,1]$. Using the fundamental theorem of calculus, we have:
    \begin{flalign}
        \big|f(\bm{y}) - \big(f(\bm{x}) + \nabla f(\bm{x})^\top (\bm{y} - \bm{x}) \big)\big| &= \Big|\int_{0}^1 \big(\nabla f(\bm{x}_\tau) - \nabla f(\bm{x})\big)^\top (\bm{y} - \bm{x}) d\tau\Big|
        \\
        & \leq \int_{0}^1 \sum_{i=1}^d \Big|\Big(\nabla f(\bm{x}_\tau)^{(i)} - \nabla f(\bm{x})^{(i)}\Big) ({y}^{(i)} - {x}^{(i)})\Big| d\tau
        \\
        & \leq \int_{0}^1 \sum_{i=1}^d L_i \big|{x}_\tau^{(i)} - x^{(i)}\big| \big|{y}^{(i)} - {x}^{(i)}\big| d\tau
        \\
        & \leq \int_{0}^1 \sum_{i=1}^d L_i \big({y}^{(i)} - {x}^{(i)}\big)^2 \tau d \tau
        \\
        & = \frac{1}{2}\sum_{i=1}^d {L_i} \big({y}^{(i)} - {x}^{(i)}\big)^2 
    \end{flalign}
    Thus,
    \begin{equation}
        f(\bm{y}) \leq f(\bm{x}) + \nabla f(\bm{x})^\top (\bm{y} - \bm{x}) + \frac{1}{2}\sum_{i=1}^d {L_i} \big({y}^{(i)} - {x}^{(i)}\big)^2.
    \end{equation}
\end{proof}

\begin{lemma}
    \label{lem12-jan18}
    Suppose $f: \mathbb{R}^d \xrightarrow{} \mathbb{R}$ is $\{L_i\}_{i=1}^d$-smooth. Then for any $\bm{x} \in \mathbb{R}^d$ and $i \in [d]$, we have:
    \begin{equation*}
        {(\nabla f(\bm{x})^{(i)})^2} \leq 2 L_i \big(f(\bm{x}) - \bar{f}^{*}_i(\bm{x})\big),
    \end{equation*}
    where $\bar{f}^{*}_i(\bm{x}) = {\min}_{h \in \mathbb{R}} \text{ } f(\bm{x} + h \bm{e}_i)$. As a result, we also have:
    \begin{equation*}
        {(\nabla f(\bm{x})^{(i)})^2} \leq 2 L_i \big(f(\bm{x}) - {f}^{*}\big),
    \end{equation*}
    where $f^{*} = \min_{\bm{z} \in \mathbb{R}^d} f(\bm{z})$.
\end{lemma}

\begin{proof}
    From \Cref{lem11-jan18}, we have for any $\bm{x}, \bm{y}$:
    \begin{equation}
        \label{eq:225-jan18}
        f(\bm{y}) \leq g(\bm{y}) := f(\bm{x}) + \sum_{i=1}^d \nabla f(\bm{x})^{(i)} \big({y}^{(i)} - {x}^{(i)}\big) + \frac{1}{2}\sum_{i=1}^d {L_i} \big({y}^{(i)} - {x}^{(i)}\big)^2.
    \end{equation}
    Let $\bar{f}^{*}_i(\bm{x}) := {\min}_{h \in \mathbb{R}} \text{ } f(\bm{x} + h \bm{e}_i)$ and $\bar{g}^{*}_i(\bm{x}) := {\min}_{h \in \mathbb{R}} \text{ } g(\bm{x} + h \bm{e}_i)$; note that $\bar{f}^{*}_i(\bm{x}) \leq \bar{g}^{*}_i(\bm{x})$ $\forall$ $\bm{x} \in \mathbb{R}^d$. Since $g(\bm{x} + h \bm{e}_i) = f(\bm{x}) + \nabla f(\bm{x})^{(i)} h + \frac{{L_i} h^2}{2}$,
    we obtain: 
    \begin{equation}
        \bar{g}^{*}_i(\bm{x}) = f(\bm{x}) - \frac{\big(\nabla f(\bm{x})^{(i)}\big)^2}{2 L_i}.
    \end{equation}
    Thus, 
    \begin{equation}
        \bar{f}^{*}_i(\bm{x}) \leq f(\bm{x}) - \frac{\big(\nabla f(\bm{x})^{(i)}\big)^2}{2 L_i}.
    \end{equation}
    Rearranging this, we get:
    \begin{equation}
        {(\nabla f(\bm{x})^{(i)})^2} \leq 2 L_i \big(f(\bm{x}) - \bar{f}^{*}_i(\bm{x})\big).
    \end{equation}
    The other bound follows from the fact that $f_i^{*} \leq \bar{f}^{*}_i(\bm{x})$ for all $\bm{x} \in \mathbb{R}^d$.
\end{proof}

\section{Proof of Corollary~\ref{cor1-jan20}}
\label{cor1-jan20-pf}
\begin{proof}
For any $\bm{x}, \bm{y} \in \mathbb{R}^d$, we have for each $i \in [d]$:
\begin{equation}
    |\nabla f(\bm{y})^{(i)} - \nabla f(\bm{x})^{(i)}| \leq L_i |y^{(i)} - x^{(i)}|.
\end{equation}  
Thus:
\begin{flalign}
    \|\nabla f(\bm{y}) - \nabla f(\bm{x})\|_2^2 = \sum_{i=1}^d |\nabla f(\bm{y})^{(i)} - \nabla f(\bm{x})^{(i)}|^2 & \leq \sum_{i=1}^d L_i^2 |y^{(i)} - x^{(i)}|^2
    \\
    & \leq L_\text{max}^2 \sum_{i=1}^d |y^{(i)} - x^{(i)}|^2
    \\
    & = L_\text{max}^2 \|\bm{y} - \bm{x}\|_2^2.
\end{flalign}
Hence, $\|\nabla f(\bm{y}) - \nabla f(\bm{x})\|_2 \leq L_\text{max} \|\bm{y} - \bm{x}\|_2$. So $f$ is $L_\text{max}$-smooth.
\\
\\
We also have:
\begin{equation}
    \sqrt{2 \tilde{\mu}} \big(f(\bm{x}) - f^{*}\big) \leq \sum_{i=1}^d \frac{(\nabla f(\bm{x})^{(i)})^2}{\sqrt{2 L_i}} \text{ } \forall \text{ } \bm{x} \in \mathbb{R}^d.
\end{equation}
Note that $\frac{(\nabla f(\bm{x})^{(i)})^2}{\sqrt{2 L_i}} \leq \frac{(\nabla f(\bm{x})^{(i)})^2}{\sqrt{2 L_\text{min}}}$. Using this above, we get:
\begin{equation}
    \sqrt{2 \tilde{\mu}} \big(f(\bm{x}) - f^{*}\big) \leq \sum_{i=1}^d \frac{(\nabla f(\bm{x})^{(i)})^2}{\sqrt{2 L_\text{min}}} \leq \frac{\|\nabla f(\bm{x})\|_2^2}{\sqrt{2 L_\text{min}}} \implies \|\nabla f(\bm{x})\|_2^2 \geq 2 \sqrt{\tilde{\mu} L_\text{min}}.
\end{equation}
Hence, $f$ is $\sqrt{\tilde{\mu} L_\text{min}}$-PL.
\end{proof}

\end{document}